\documentclass{article}
\usepackage{iclr2027_conference,times}

\usepackage{amsmath,amsfonts,bm}

\def\eqref#1{equation~\ref{#1}}

\def\1{\bm{1}}

\DeclareMathAlphabet{\mathsfit}{\encodingdefault}{\sfdefault}{m}{sl}
\SetMathAlphabet{\mathsfit}{bold}{\encodingdefault}{\sfdefault}{bx}{n}

\usepackage[utf8]{inputenc}
\usepackage[T1]{fontenc}
\usepackage{xcolor}
\usepackage{hyperref}
\usepackage{url}
\usepackage{booktabs}
\usepackage{colortbl}
\usepackage{amsmath}
\usepackage{amsfonts}
\usepackage{amssymb}
\usepackage{amsthm}
\usepackage{bm}
\usepackage{graphicx}
\usepackage{caption}
\usepackage{wrapfig}
\usepackage{dblfloatfix}
\usepackage{xspace}
\usepackage{microtype}
\usepackage{enumitem}
\usepackage{listings}
\usepackage{aliascnt}
\usepackage[nameinlink,capitalize,noabbrev]{cleveref}

\makeatletter
\newcommand{\subcolfont}{%
    \fontsize{\dimexpr0.9\dimexpr\f@size pt\relax\relax}%
             {\dimexpr1.1\dimexpr\f@size pt\relax\relax}\selectfont}
\makeatother

\newcommand{\suppref}[1]{Supp.~\cref{#1}}
\newcommand{\Suppref}[1]{Supp.~\Cref{#1}}

\definecolor{pdmdblue}{RGB}{33,91,166}
\definecolor{bytedanceblue}{RGB}{47,112,255}
\definecolor{bytedancebluebg}{RGB}{218,232,255}
\definecolor{codegreen}{RGB}{34,139,94}
\definecolor{codepurple}{RGB}{133,74,160}
\definecolor{dmdred}{HTML}{D55E00}
\definecolor{pdmdgreen}{HTML}{009E73}
\definecolor{resblue}{HTML}{3775BA}
\newcommand{\capbox}[2]{{\setlength{\fboxrule}{0.4pt}\setlength{\fboxsep}{1pt}\fcolorbox{#1}{white}{#2}}}
\definecolor{codegray}{RGB}{100,100,100}
\definecolor{codebackground}{RGB}{247,248,250}
\definecolor{pdmdhighlight}{RGB}{218,245,226}
\definecolor{algorithmcomment}{RGB}{92,139,143}
\hypersetup{
  colorlinks=true,
  linkcolor=pdmdblue,
  citecolor=pdmdblue,
  urlcolor=pdmdblue
}

\newcommand{\methodname}{\textsc{PDMD}\xspace}
\newcommand{\dmd}{\textsc{DMD}\xspace}
\newcommand{\dmdtwo}{\textsc{DMD2}\xspace}
\newcommand{\dmdtwodag}{\textsc{DMD2}$^{\dagger}$\xspace}
\newcommand{\rcm}{rCM\xspace}
\newcommand{\anyflow}{AnyFlow\xspace}

\newcommand{\deemph}[1]{\textcolor{black!45}{#1}}
\newcommand{\deemphcite}[1]{%
  \begingroup
  \hypersetup{citecolor=black!45}%
  \citep{#1}%
  \endgroup
}

\newtheorem{theorem}{Theorem}
\newaliascnt{lemma}{theorem}
\newtheorem{lemma}[lemma]{Lemma}
\aliascntresetthe{lemma}
\newaliascnt{proposition}{theorem}
\newtheorem{proposition}[proposition]{Proposition}
\aliascntresetthe{proposition}
\newaliascnt{corollary}{theorem}
\newtheorem{corollary}[corollary]{Corollary}
\aliascntresetthe{corollary}

\crefname{listing}{Algorithm}{Algorithms}
\Crefname{listing}{Algorithm}{Algorithms}
\crefname{lstlisting}{Algorithm}{Algorithms}
\Crefname{lstlisting}{Algorithm}{Algorithms}
\lstdefinestyle{algorithm}{
  language=Python,
  basicstyle=\ttfamily\scriptsize,
  keywordstyle=\color{pdmdblue}\bfseries,
  commentstyle=\color{algorithmcomment},
  stringstyle=\color{codepurple},
  backgroundcolor=\color{white},
  numbers=none,
  frame=tb,
  framerule=0.8pt,
  framesep=4pt,
  rulecolor=\color{black},
  columns=fullflexible,
  keepspaces=true,
  showstringspaces=false,
  breaklines=true,
  tabsize=4,
  captionpos=t,
  abovecaptionskip=2pt,
  belowcaptionskip=5pt,
  aboveskip=0.75em,
  belowskip=0.75em,
  xleftmargin=0pt,
  xrightmargin=0pt,
  escapeinside={(*@}{@*)}
}
\newcommand{\pdmdline}[1]{%
  \makebox[0pt][l]{%
    \raisebox{-0.28\baselineskip}{%
      \textcolor{pdmdhighlight}{%
        \rule{\linewidth}{1.28\baselineskip}%
      }%
    }%
  }%
  \textcolor{black}{\ttfamily\scriptsize #1}%
}

\newcommand{\linkicon}[1]{\raisebox{-0.18em}{\includegraphics[height=1.05em]{figures/icons/#1}}}
\title{\methodname{}: Projected Distribution Matching Distillation for Video Diffusion Models
\\[0.35em]
{\normalfont\large
\href{https://pdmd2026.github.io/}{\linkicon{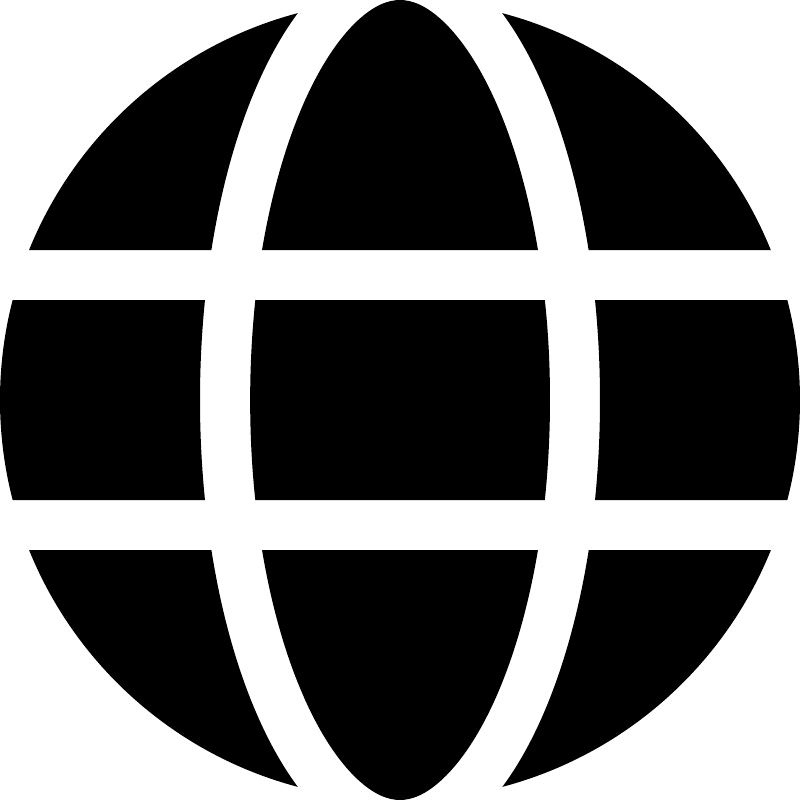}\,Project Page}\hspace{2.5em}
\href{https://huggingface.co/pdmd2026/pdmd_4NFE_full}{\linkicon{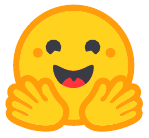}\,Hugging Face}\hspace{2.5em}
\href{https://github.com/ZeamoxWang/pdmd}{\linkicon{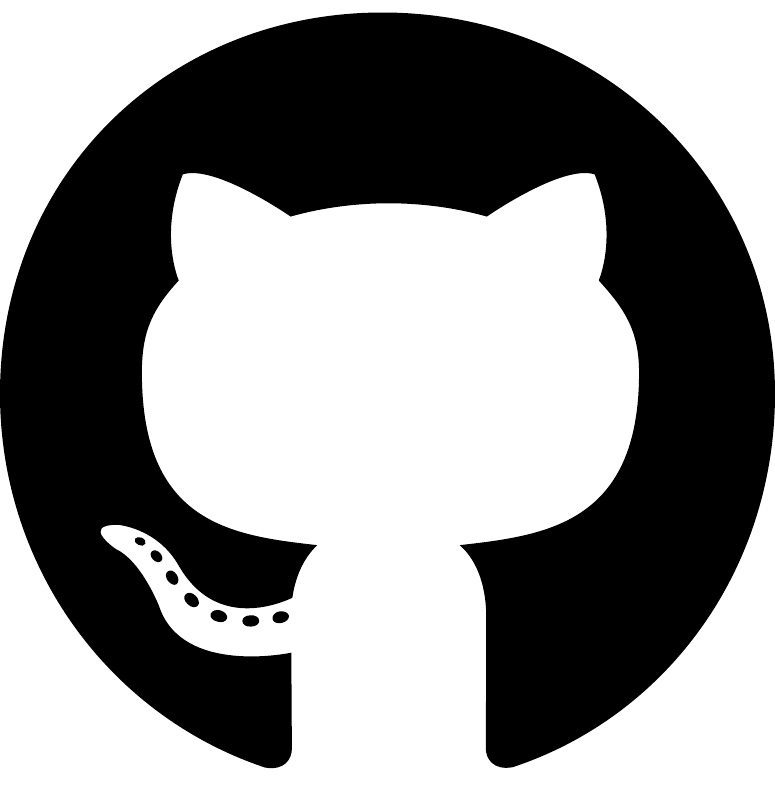}\,GitHub}}}

\author{%
Zimo Wang$^{1*}$\quad
Junkun Yuan$^{2}$\quad
Angtian Wang$^{2}$\quad
Haotian Yang$^{2}$
\\
\bfseries
Canyu Zhang$^{2}$\quad
Siyuan Yuan$^{2}$\quad
Xingchang Huang$^{2}$\quad
Bo Liu$^{2}$
\\
\bfseries
Yizhi Wang$^{2}$\quad
Yiding Yang$^{2}$\quad
Chongyang Ma$^{2}$\quad
Gordon Guocheng Qian$^{2\dagger}$
\\
$^{1}$University of California, San Diego\qquad
$^{2}$ByteDance Inc.
\\
$^{\dagger}$Project lead\qquad
$^{*}$Work done while Zimo Wang was an intern at ByteDance.
}

\iclrfinalcopy

\begin{document}

\maketitle
\addtocontents{toc}{\protect\setcounter{tocdepth}{-1}}
\providecommand{\tsf}{}\renewcommand{\tsf}[2]{\includegraphics[width=#1]{figures/src/teaser_pdmd/#2.jpg}}
\providecommand{\tsgap}{}\renewcommand{\tsgap}{\hspace{1.6pt}}
\providecommand{\tsvgap}{}\renewcommand{\tsvgap}{\\[1.6pt]}
\providecommand{\wtwo}{}\renewcommand{\wtwo}{0.4955\textwidth}
\providecommand{\wthree}{}\renewcommand{\wthree}{0.3293\textwidth}
\providecommand{\wfour}{}\renewcommand{\wfour}{0.2455\textwidth}
\providecommand{\wfive}{}\renewcommand{\wfive}{0.1960\textwidth}
\providecommand{\wsix}{}\renewcommand{\wsix}{0.1625\textwidth}
\providecommand{\weight}{}\renewcommand{\weight}{0.1205\textwidth}
\providecommand{\tslabel}{}\renewcommand{\tslabel}[1]{%
  \makebox[\textwidth][l]{\color{black!62}\sffamily\scriptsize #1}}
\providecommand{\tstag}{}\renewcommand{\tstag}[1]{%
  \makebox[\textwidth][l]{\color{black!55}\sffamily\fontsize{6.4}{7.4}\selectfont #1}}
\providecommand{\tsbadge}{}\renewcommand{\tsbadge}[1]{%
  \llap{\raisebox{4.5pt}{\colorbox{black!72}{\textcolor{white}{\sffamily\fontsize{6.2}{7}\selectfont\,#1\,}}\hspace{4.5pt}}}}
\providecommand{\wthreemini}{}\renewcommand{\wthreemini}{0.1078\textwidth}
\providecommand{\wfourmini}{}\renewcommand{\wfourmini}{0.0798\textwidth}
\providecommand{\wten}{}\renewcommand{\wten}{0.0952\textwidth}
\providecommand{\tsminigap}{}\renewcommand{\tsminigap}{\hspace{1.2pt}}
\providecommand{\tsp}{}\renewcommand{\tsp}[2]{\includegraphics[width=#1]{figures/src/teaser_pdmd/#2.pdf}}
\providecommand{\tshand}{}\renewcommand{\tshand}[2]{%
  \makebox[#1][c]{\color{black!70}\usefont{T1}{pzc}{m}{it}\fontsize{7.5}{8.5}\selectfont ``#2''}}

\begin{center}
    \captionsetup{type=figure}
    \vspace{-0.9em}
    \tsp{\wthree}{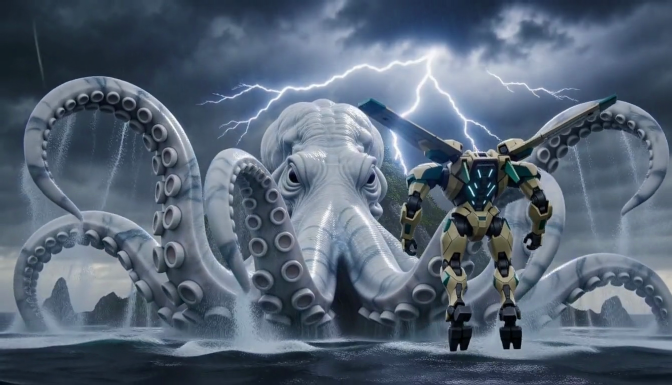}\tsgap\tsp{\wthree}{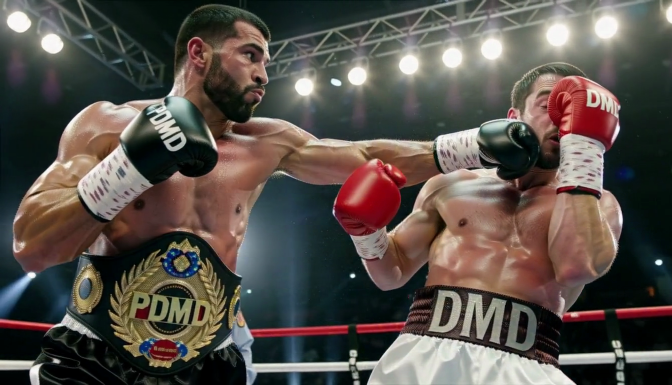}\tsgap\tsp{\wthree}{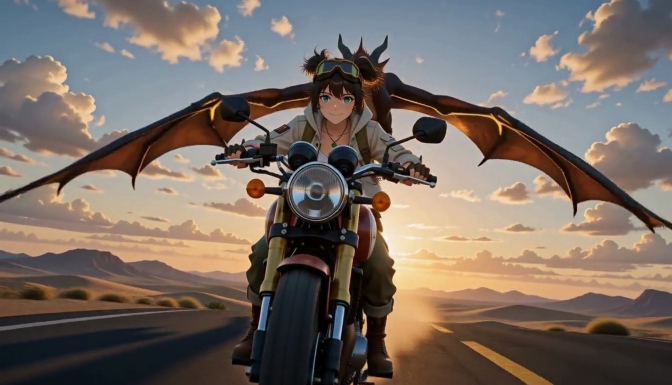}
    \\[1.6pt]
    \tsp{\wthreemini}{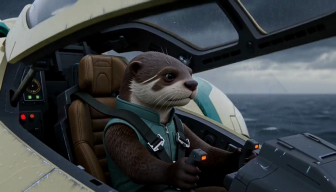}\tsminigap\tsp{\wthreemini}{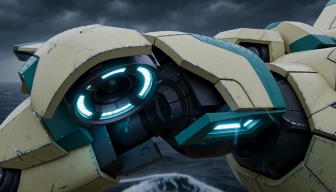}\tsminigap\tsp{\wthreemini}{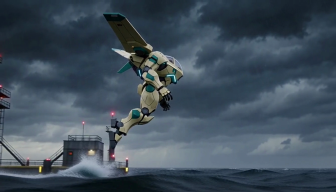}\tsgap%
    \tsp{\wthreemini}{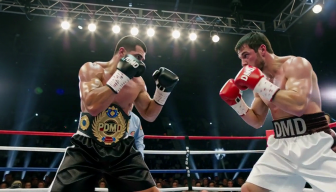}\tsminigap\tsp{\wthreemini}{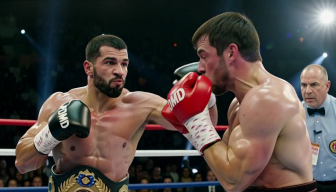}\tsminigap\tsp{\wthreemini}{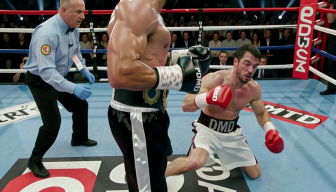}\tsgap%
    \tsp{\wthreemini}{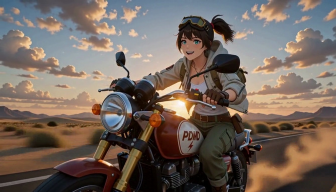}\tsminigap\tsp{\wthreemini}{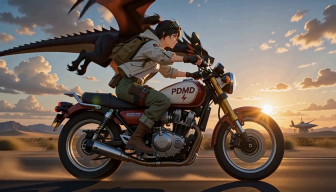}\tsminigap\tsp{\wthreemini}{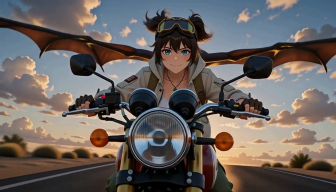}
    \\[2.4pt]
    \tshand{\wthree}{A marble kraken rises in a storm}\tsgap%
    \tshand{\wthree}{\methodname{} lands the knockdown on DMD}\tsgap%
    \tshand{\wthree}{A sunset ride}
    \\[3.4pt]
    \tsp{\wten}{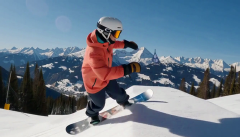}\tsgap\tsp{\wten}{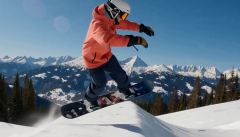}\tsgap\tsp{\wten}{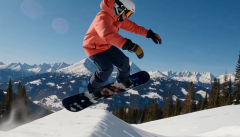}\tsgap\tsp{\wten}{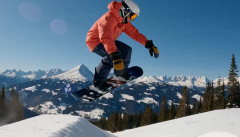}\tsgap%
    \tsp{\wten}{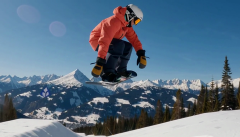}\tsgap\tsp{\wten}{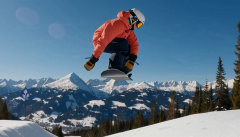}\tsgap\tsp{\wten}{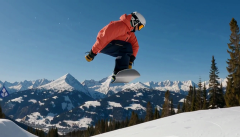}\tsgap\tsp{\wten}{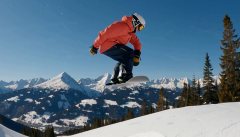}\tsgap%
    \tsp{\wten}{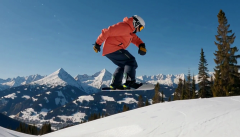}\tsgap\tsp{\wten}{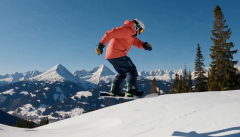}
    \\[2.4pt]
    \tshand{\textwidth}{One snowboard jump: takeoff, grab, landing}
    \par\vspace{0.35em}
    \captionof{figure}{%
        \textbf{4-NFE video generation with \methodname{}.}
        The 14.4-second, $1366{\times}768$ videos generated by
        \methodname{}-distilled MiniMax-H3~\citep{minimaxh3} exhibit natural
        colors and dynamic motion.
    }
    \label{fig:teaser}
\end{center}

\begin{abstract}
Modern video diffusion models require tens of denoising evaluations over long spatiotemporal token sequences. Distribution Matching Distillation (DMD) reduces the number of function evaluations (NFE) to just a few. However, DMD samples can degrade during training, exhibiting progressive oversaturation and artifacts. We trace this instability to critic errors, which enter successive student updates and accumulate over time. We introduce Projected Distribution Matching Distillation (\methodname{}) to filter critic errors. \methodname{} projects out the component of the DMD update parallel to the student--critic endpoint residual. At a fixed noisy query, we prove that this residual is an unbiased estimate of the critic's endpoint error. Under high-dimensional assumptions, this projection removes a constant fraction of critic error while discarding only a vanishing fraction of ideal DMD signal. Empirically, the projection stabilizes training and improves sample quality where DMD degrades and develops unnatural textures. \methodname{} requires only a one-line code change to DMD, with no extra loss, network, data, model pass, or multi-stage training. With Wan2.1, \methodname{} achieves a VBench total score of 83.73 at 4 NFE, surpassing matched DMD by 1.03 points. On MiniMax-H3 joint video--audio generation, \methodname{} achieves a VideoGen-Eval visual total score of 83.17, 0.41 points above the strongest distilled baseline. \methodname{} also achieves the best performance on all six audio metrics among the compared 4-NFE models. Qualitative comparisons and user studies favor \methodname{} over the distilled baselines in visual quality, motion, and audio quality. Code and models are available at \url{https://pdmd2026.github.io/}.
\end{abstract}

\section{Introduction}
\label{sec:introduction}

Video generation has advanced rapidly in visual quality, motion fidelity, and
prompt following, as demonstrated by Sora~\citep{sora2024},
Wan~\citep{wan2025wan}, Seedance~\citep{seedance2025}, and
MiniMax~\citep{minimaxh3}.
This progress comes at substantial inference cost: a modern video diffusion transformer~\citep{rombach2022ldm,peebles2023dit} typically requires a large number of function evaluations (NFE), often dozens, each processing a long spatiotemporal token sequence.
Step distillation~\citep{salimans2022progressive,song2023consistency} reduces
this cost by compressing a many-step teacher into a few-step student that
requires only a few evaluations.
Distribution Matching Distillation (\dmd)~\citep{yin2024onestep} is a widely
used approach to step distillation.
It trains the student to match the distribution produced by the many-step teacher, using the difference between the student score (learned online by a model called the critic) and the frozen teacher score.

\begin{figure}[t]
    \centering
    \setlength{\tabcolsep}{0pt}
    \newcommand{\dgf}[1]{\includegraphics[width=0.105\textwidth]{figures/src/degradation/v4/#1.jpg}}
    \begin{tabular}{@{}c@{\hspace{1pt}}c@{\hspace{1pt}}c@{\hspace{5pt}}c@{\hspace{1pt}}c@{\hspace{1pt}}c@{\hspace{5pt}}c@{\hspace{1pt}}c@{\hspace{1pt}}c@{}}
        \multicolumn{3}{c@{\hspace{5pt}}}{\dmd~\citep{yin2024onestep}} &
        \multicolumn{3}{c@{\hspace{5pt}}}{\dmdtwo~\citep{yin2024improved}} &
        \multicolumn{3}{c@{}}{\textbf{\methodname{} (Ours)}} \\[1pt]
        \scriptsize 100 iter. & \scriptsize 500 iter. & \scriptsize 1,000 iter. &
        \scriptsize 100 iter. & \scriptsize 500 iter. & \scriptsize 1,000 iter. &
        \scriptsize 100 iter. & \scriptsize 500 iter. & \scriptsize 1,000 iter. \\[1pt]
        \dgf{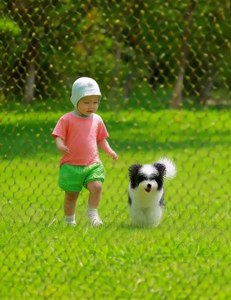} & \dgf{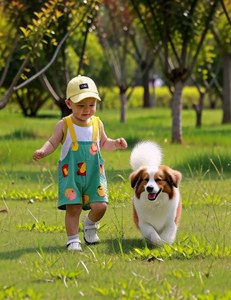} & \dgf{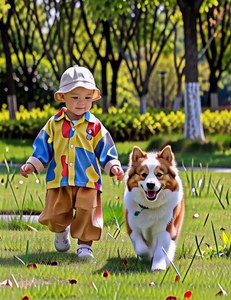} &
        \dgf{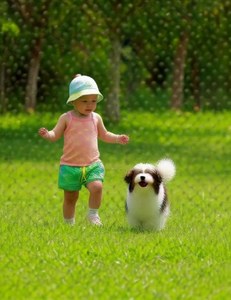} & \dgf{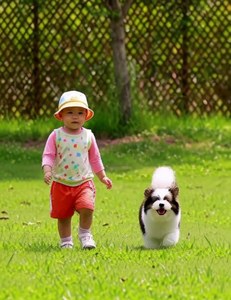} & \dgf{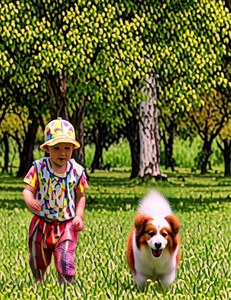} &
        \dgf{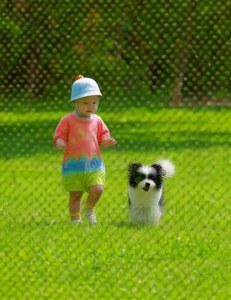} & \dgf{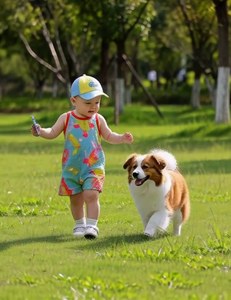} & \dgf{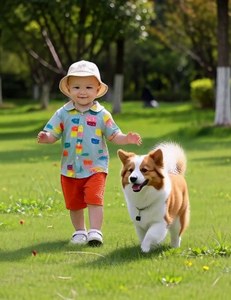}
    \end{tabular}
    \vspace{-0.3em}
    \caption{\textbf{Qualitative comparisons over matched 4-NFE training.} \dmd and \dmdtwo develop unnatural textures after 500 iterations, whereas \methodname{} keeps a natural appearance.}
    \label{fig:degradation}
\end{figure}

\begin{wrapfigure}[18]{r}{0.67\columnwidth}
    \vspace{-1.3em}
    \centering
    \setlength{\fboxsep}{0pt}%
    \setlength{\fboxrule}{3.0pt}%
    \newcommand{\dvf}[2]{\fcolorbox{#1}{white}{\includegraphics[height=0.2547\columnwidth]{figures/src/proj_endpoints/s500/rank4_target_#2_30.jpg}}}%
    \begin{tabular}{@{}c@{\hspace{2pt}}c@{\hspace{1pt}}c@{}}
        \includegraphics[height=0.2547\columnwidth]{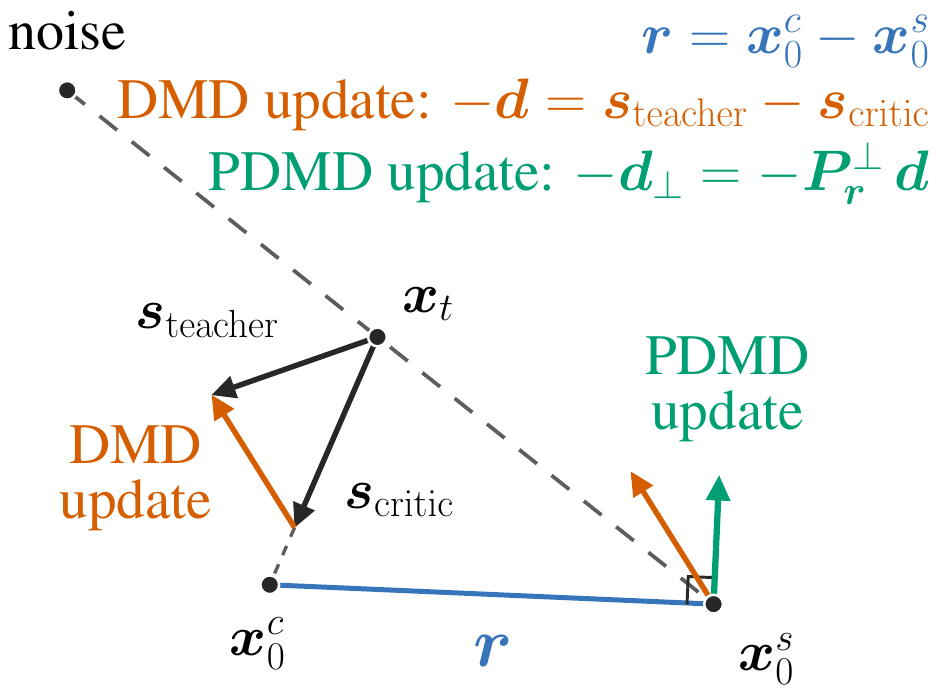} &
        \dvf{dmdred}{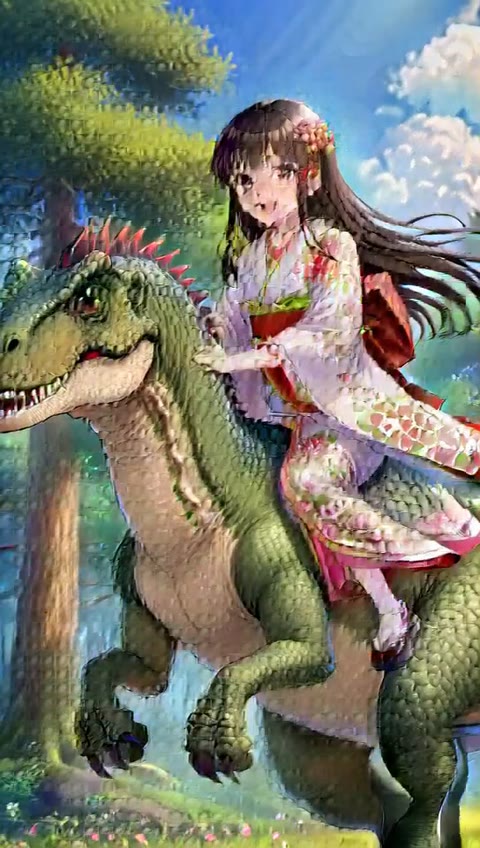} &
        \dvf{pdmdgreen}{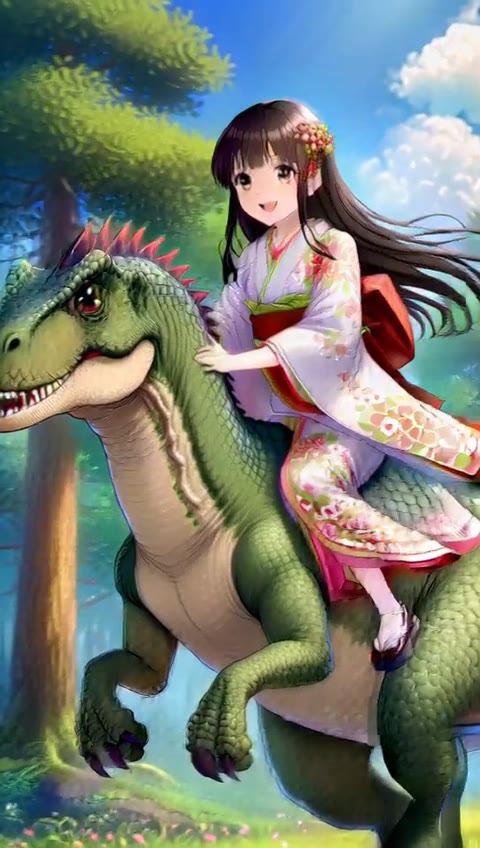}
    \end{tabular}
    \vspace{-0.4em}
    \caption{\textbf{DMD and \methodname{} updates.} DMD moves the student along \textcolor{dmdred}{$-\bm d$}. The critic score $\bm s_{\mathrm{critic}}$, the online-learning part of \textcolor{dmdred}{$\bm d$}, carries hidden error. \methodname{} projects out the component of \textcolor{dmdred}{$\bm d$} parallel to the known residual \textcolor{resblue}{$\bm r$} to suppress critic error, because \textcolor{resblue}{$\bm r$} is the critic error estimate. The two frames decode one real update target: the \protect\capbox{pdmdgreen}{\methodname{} target} contains fewer artifacts than the \protect\capbox{dmdred}{\dmd target}, while keeping 83\% of the update norm. }
    \label{fig:dmd_vectors}
\end{wrapfigure}

The critic is trained alternately with the student, using student-generated samples to track current student's distribution.
However, the online critic has approximation error even for a fixed student and can lag behind a changing one.~\citep{shao2025magicdistillation, yin2024improved}.
Because the critic score enters the \dmd update directly, \textit{critic errors} are repeatedly transferred to the student and can accumulate throughout training.
The critic error can lead to training instability, with samples becoming progressively oversaturated and eventually losing quality~\citep{shao2025magicdistillation, yin2024improved}.
For example, on MiniMax-H3, \dmd becomes visibly oversaturated from 500
iterations onward (\cref{fig:degradation}).

\dmdtwo~\citep{yin2024improved} reduces critic errors by taking more critic updates per student update, and also adds a discriminator to improve quality.
Other methods combine \dmd with additional objectives~\citep{zheng2026rcm,gu2026anyflow}.
We look beyond these additions and ask whether \dmd can reduce critic error using what it already computes.
Our goal is to use an estimate of the critic error direction that costs no extra data or model pass, to cancel part of the critic error with the estimate, and keep most of the \dmd signal.
Our key insight is that the direction from the student sample to the critic's one-step denoised prediction, namely the student--critic endpoint residual, provides an estimate of the critic error direction (\cref{sec:analysis-critic-average}).

To filter and attenuate critic error, we present Projected Distribution Matching Distillation (\methodname{}), which projects out the parallel component to the critic error estimate of the \dmd update.
\methodname{} is a one-line change to \dmd, with no auxiliary objective,
network, model pass, data, or training stage.
Under high-dimensional concentration and weak-alignment assumptions, our
analysis shows \methodname{} removes a constant fraction of critic error
and only a vanishing fraction of useful \dmd signal.

As illustrated by \cref{fig:dmd_vectors}, \methodname{} recasts improving DMD as a filtering problem inside the existing update.
\dmd re-noises a student sample $\bm x_0^s$ to $\bm x_t$ and evaluates the
critic score $\bm s_{\mathrm{critic}}$ and the frozen teacher score
$\bm s_{\mathrm{teacher}}$ at the same query.
The score difference,
$\bm d=\bm s_{\mathrm{critic}}-\bm s_{\mathrm{teacher}}$, defines its student update.
However, the critic error within $\bm s_{\mathrm{critic}}$ is not directly observable, and cannot simply be subtracted from $\bm d$.
The critic score can be linearly converted to a one-step endpoint $\bm x_0^c$ with no extra model pass.
Fortunately, the student--critic endpoint residual $\bm r=\bm x_0^c-\bm x_0^s$ is observable and a critic error estimator.
For a fixed query $\bm x_t$, the ideal critic endpoint is the average student endpoint. Therefore, the observed student--critic endpoint residual estimates critic error, with a sample-variation term that averages to zero.
We prove that, at a fixed query $\bm x_t$, $\bm r$ is an unbiased estimate of the critic endpoint error in \cref{eq:critic-residual-mean}.
Projecting out the component of $\bm d$ parallel to $\bm r$ suppresses critic error. A visual example is provided at right of \cref{fig:dmd_vectors}.

The filtering shows up empirically as more stable training, less
oversaturation, more motion, and better audio.
In 2D analysis, \methodname{} improves energy distance, on-mode fraction, and collapses to a single mode in $3/20$ runs, against $14/20$ without the projection (\cref{fig:2d-collapse};
\suppref{tab:supp-twomode}).
On Wan2.1-T2V-1.3B, \methodname{} reaches \textbf{83.73} total score on VBench, 1.03 points above its matched \dmd baseline (\cref{tab:wan13b}).
On MiniMax-H3-33B joint video--audio generation, it reaches \textbf{83.17}
total score on VideoGen-Eval, 0.41 above the strongest distilled baseline, and obtains the highest scores on all six audio metrics among 4-NFE models
(\cref{tab:h3}). User studies prefer \methodname{} to every 4-NFE baseline on visual, motion,
and audio quality (\cref{tab:h3} and \cref{tab:wan13b}).

Our \textbf{contributions} are:
\begin{itemize}[leftmargin=*,topsep=2pt,itemsep=2pt]
\item We connect \dmd's training instability and visual degradation
to critic error in its score-difference update, and identify the student--critic endpoint residual as an estimate of the critic error.

\item We introduce \methodname{}, a simple and effective algorithm that filters critic error by projecting out the component of the \dmd update parallel to the student--critic endpoint residual. Under high-dimensional assumptions, \methodname{} removes a constant fraction of this error while losing only a vanishing fraction of useful signal.

\item Experiments show improved energy distance and on-mode fraction in 2D, better visual quality on Wan2.1, and better video and audio quality on MiniMax-H3. User studies also favor \methodname{}.
\end{itemize}

\section{Related Work}
\label{sec:related-work}

\paragraph{Distribution matching distillation.}
DMD approximately minimizes the reverse KL divergence between a few-step student and a diffusion teacher~\citep{yin2024onestep}. Its online critic is imperfect, and the resulting error can contribute to progressive oversaturation and reduced motion~\citep{yin2024improved,xu2025fdistill,zheng2026rcm}. Adversarial distillation, such as DMD2, adds discriminator supervision. DMD2 also uses additional critic updates~\citep{yin2024improved,xu2023ufogen,sauer2023add,sauer2024ladd,lin2024lightning,lin2025apt}. These methods can improve fidelity or critic tracking, but they introduce discriminator training and coupled optimization, which can add memory and training costs and be difficult to stabilize~\citep{ge2025senseflow}. SiD, $f$-distill, and SGMD instead replace the reverse-KL formulation with semi-implicit Fisher, discriminator-weighted $f$-divergence, and stop-gradient Fisher objectives, respectively. Each requires additional network evaluations~\citep{zhou2024sid,xu2025fdistill,wu2026sgmd}. Others keep the DMD objective and add teacher-trajectory, real-data, relational, or reward supervision~\citep{wu2026dpdmd,chen2026dataforcing,zhang2026codmd,bai2026adaptive,jiang2025dmdr}. ADV also identifies oversaturation and temporal collapse in few-step video distillation and counters them with an adaptively weighted regression loss and a temporal regularizer on top of DMD~\citep{you2026adaptive}. \methodname{} keeps the DMD score networks and critic training unchanged. It changes only the student update computed from their existing predictions.

\paragraph{Trajectory and hybrid objectives.}
Consistency models enforce agreement along probability-flow ODE trajectories~\citep{song2023consistency,song2024improved,kim2023consistency}. rCM extends this approach to video by adding DMD to continuous-time consistency training~\citep{zheng2026rcm}. MeanFlow learns average velocities~\citep{geng2025meanflow}, while AnyFlow combines MeanFlow-style flow-map training with on-policy DMD through Flow Map Backward Simulation~\citep{gu2026anyflow}. SC-DMD, TMD and FSF-DMD are further hybrids, adding step-size consistency, a MeanFlow-pretrained flow head, or a critic-free update from the flow-map generator's endpoint pseudo-velocity~\citep{ge2026salt,nie2026transition,kim2026fsfdmd}. These methods need another objective, rollout procedure, or training stage. \methodname{} requires no additional model, loss, pass, data, or stage beyond DMD.

\paragraph{Error cancellation and directional decomposition.}
Adaptive noise cancellation and control variates use correlated references to reduce error or variance~\citep{widrow1975adaptive,glynn2002some}. DDS and FlowEdit likewise take differences between paired predictions for image editing~\citep{hertz2023delta,kulikov2025flowedit}, while Perp-Neg and APG project guidance components~\citep{armandpour2023reimagine,sadat2025eliminating}. \methodname{} instead projects the DMD update using the student--critic endpoint residual to suppress critic estimation error.

\section{Methodology}
\label{sec:method}

\subsection{Preliminaries: Distribution Matching Distillation}
\label{sec:method-dmd}

Distribution Matching Distillation (DMD)~\citep{yin2024onestep} trains a few-step student to match the teacher distribution using the difference between an online critic score and a frozen teacher score.

Let $G_{\theta}$ denote the few-step student. Given Gaussian input $\bm z$ and condition $\bm c$, the student produces the clean endpoint $\bm x_0^s(\bm z,\bm c)=G_{\theta}(\bm z,\bm c)$, a sample from the student's final distribution. An online critic field $\bm s_{\mathrm{critic}}(\bm x_t,t,\bm c)$ is trained to estimate the score at the renoised student endpoints. DMD renoises these samples using the same process and same coefficients $\alpha_t$ and $\sigma_t$ used to train the teacher:
\begin{equation}
    \bm x_t
    =
    \alpha_t\bm x_0^s(\bm z,\bm c)+\sigma_t\bm\epsilon,
    \qquad
    \bm\epsilon\sim\mathcal N(\bm 0,\bm I).
    \label{eq:dmd-noising}
\end{equation}

To update the student under the reverse-KL objective, DMD evaluates both the online critic field $\bm s_{\mathrm{critic}}(\bm x_t,t,\bm c)$ and the frozen teacher field $\bm s_{\mathrm{teacher}}(\bm x_t,t,\bm c)$, which estimates the score of the teacher distribution. Omitting scalar timestep weights, the DMD update signal is their difference:
\begin{equation}
    \bm d(\bm x_t,t,\bm c)
    =
    \bm s_{\mathrm{critic}}(\bm x_t,t,\bm c)
    -
    \bm s_{\mathrm{teacher}}(\bm x_t,t,\bm c).
    \label{eq:dmd-difference}
\end{equation}

$\bm d$ is backpropagated as a gradient through $\bm x_0^s=G_\theta(\bm z,\bm c)$ to update the student parameters $\theta$. Optimizing the student and critic in alternation trains the few-step student to match the teacher.

\subsection{Projected Distribution Matching Distillation}
\label{sec:method-pdmd}

DMD uses $-\bm d$ as an endpoint-space descent signal. However, because the critic is imperfect, this teacher--critic score difference contains the critic error. To suppress this contamination, we propose Projected Distribution Matching Distillation (\methodname{}), which projects out the component of the DMD update parallel to the student--critic endpoint residual, an estimate of the critic error.

\begin{wrapfigure}[10]{r}{0.48\columnwidth}
    \vspace{-2em}
    \begin{lstlisting}[
        style=algorithm,
        caption={\methodname{} student step. The highlighted line is the only
        change from DMD.},
        label={alg:pdmd-main}
    ]
x_0_s = G(z, cond)
x_t, t = renoise(x_0_s)           # Eq. (1)
s_t = T.score(x_t, t, cond)
s_c, x_0_c = C.score_end(x_t, t, cond)
d = s_c - s_t                     # Eq. (2)
r = x_0_c - x_0_s
(*@\pdmdline{d = d - dot(d, r) / dot(r, r) * r  \# Eq. (4)}@*)
update(G, dmd_loss(x_0_s, d))
\end{lstlisting}%
    \vspace{-1.0em}
\end{wrapfigure}

Specifically, the critic endpoint is $\bm x_0^c=(\bm x_t+\sigma_t^2\bm s_{\mathrm{critic}})/\alpha_t$. We define $\bm r=\bm x_0^c-\bm x_0^s$ and let
\begin{equation}
    \bm P_{\bm r}
    =
    \frac{\bm r\bm r^{\mathsf T}}{\|\bm r\|^2},
    \qquad
    \bm P_{\bm r}^{\perp}=\bm I-\bm P_{\bm r}.
    \label{eq:perpendicular-projector}
\end{equation}
The projected update is
\begin{equation}
    \bm d_{\perp}
    =
    \bm P_{\bm r}^{\perp}\bm d
    =
    \bm d
    -
    \frac{\langle\bm d,\bm r\rangle}{\|\bm r\|^2}\bm r.
    \label{eq:method-pdmd-projection}
\end{equation}
When $\bm r=\bm 0$, we leave the update unchanged. \methodname{} uses $\bm d_{\perp}$ in place of $\bm d$. This one-line change requires no additional loss, network, model pass, data, or training stage.

\section{Analysis}
\label{sec:analysis}

\subsection{The Projection Filters the Critic Error}
\label{sec:analysis-critic-average}

\paragraph{Residual as unbiased critic-error estimate.}
We show that the residual $\bm r$ that \methodname{} projects along is, at a fixed query, an unbiased estimate of the critic error. Following the re-diffusion process in \cref{eq:dmd-noising}, let $Q$ collect the noisy query $\bm x_t$, its noise level $t$, and the conditioning $\bm c$, and let $(Q,\bm S)$ be a random query--endpoint pair from critic training. Critic training re-diffuses the student's own endpoints, so the endpoint of the pair is the student endpoint, $\bm S=\bm x_0^s$. The critic endpoint is trained with a squared $\ell_2$ denoising loss, whose population-risk minimizer at a fixed query $Q=q$ is the conditional mean (\suppref{lem:supp-l2-conditional-mean}). The population-optimal critic endpoint and the conditional noise are therefore
\begin{equation}
    \bm c^\star(q)=\mathbb E[\bm S\mid Q=q],
    \qquad
    \bm\xi=\bm S-\bm c^\star(q),
    \qquad
    \mathbb E[\bm\xi\mid Q=q]=\bm 0.
    \label{eq:critic-conditional-statistics}
\end{equation}
The learned critic endpoint $\bm x_0^c$ is fixed at $q$, and its error is
\begin{equation}
    \bm e=\bm x_0^c-\bm c^\star(q).
    \label{eq:critic-endpoint-error}
\end{equation}
The residual at that query is $\bm R=\bm x_0^c-\bm S=\bm e-\bm\xi$. $\bm R$ is random through $\bm S$. The residual $\bm r$ of \cref{eq:perpendicular-projector} is one realization of $\bm R$, and $\bm P_{\bm r}$ is the corresponding realization of the random projector $\bm P_{\bm R}$. Taking the conditional mean of $\bm R=\bm e-\bm\xi$ and using $\mathbb E[\bm\xi\mid Q=q]=\bm 0$ gives
\begin{equation}
    \mathbb E[\bm R\mid Q=q]=\bm e,
    \label{eq:critic-residual-mean}
\end{equation}
so the residual is an unbiased estimate of the critic error. Unlike an arbitrary direction, $\bm R$ contains $\bm e$ by construction.

\paragraph{Error removal guarantee.}
We now quantify how much critic error the projection removes. Let $\bm\Sigma(q)=\operatorname{Cov}(\bm S\mid Q=q)$, so that $\mathbb E[\|\bm\xi\|^2\mid Q=q]=\operatorname{tr}\bm\Sigma(q)$ and the second moments of the residual are
\begin{equation}
    \mathbb E[\bm e^{\mathsf T}\bm R\mid Q=q]=\|\bm e\|^2,
    \qquad
    \mathbb E[\|\bm R\|^2\mid Q=q]
    =\|\bm e\|^2+\operatorname{tr}\bm\Sigma(q).
    \label{eq:critic-residual-moments}
\end{equation}
For $\bm e\neq\bm 0$, define the fraction of critic-error energy removed along $\bm R$ as
\begin{equation}
    \gamma_e(\bm R)
    =
    \frac{\|\bm P_{\bm R}\bm e\|^2}{\|\bm e\|^2}
    =
    \frac{(\bm e^{\mathsf T}\bm R)^2}
         {\|\bm e\|^2\|\bm R\|^2}.
    \label{eq:conditional-removal-ratio}
\end{equation}
We set $\gamma_e(\bm 0)=0$ when the residual vanishes. Applying Cauchy--Schwarz to $\mathbb E[(\bm e^{\mathsf T}\bm R)^2/\|\bm R\|^2\mid Q=q]$ with the moments of \cref{eq:critic-residual-moments} yields
\begin{equation}
    \mathbb E\!\left[\gamma_e(\bm R)\mid Q=q\right]
    \geq
    \frac{\|\bm e\|^2}
         {\|\bm e\|^2+\operatorname{tr}\bm\Sigma(q)}.
    \label{eq:conditional-removal-bound}
\end{equation}
Moreover, the Bayes risk $L^\star(q)=\mathbb E[\|\bm S-\bm c^\star(q)\|^2\mid Q=q]$ equals $\operatorname{tr}\bm\Sigma(q)$, so the covariance and conditional-risk views are two forms of the same result (\suppref{sec:proof-l2-optimal-critic}). At a fixed noise level, converting endpoint predictions to scores only rescales the corresponding critic error, so $\gamma_e$ is unchanged in the score space used by DMD.

\paragraph{High-dimensional interpretation.}
When the conditional noise $\bm\xi$ is concentrated in high dimensions and has no dominant alignment with the fixed error direction $\bm e$, the lower bound in \cref{eq:conditional-removal-bound} describes a typical update:
\begin{equation}
    \gamma_e
    \approx
    \frac{\|\bm e\|^2}
         {\|\bm e\|^2+\operatorname{tr}\bm\Sigma(q)}.
    \label{eq:high-dimensional-removal}
\end{equation}
For example, when $\operatorname{tr}\bm\Sigma(q)=\|\bm e\|^2$, the conditional guarantee removes at least $50\%$ of the critic-error energy on average (\suppref{thm:supp-typical-removal,cor:supp-gaussian-removal}). Under high-dimensional assumptions, the projection removes a constant fraction of the critic error but only a vanishing fraction of the ideal signal (\suppref{cor:supp-high-dimensional-separation,cor:supp-net-correction-improvement}). On MiniMax-H3, the projection visibly attenuates speckle artifacts and retains more than 80\% of the update norm for most of training (\suppref{sec:h3-norm-ratio,sec:h3-proj-endpoints}).

\paragraph{Critic lag and collapse.}
\dmdtwo attributes instability to the critic lagging behind the moving student and adds critic updates~\citep{yin2024improved}. In a two-mode model, a strong lag turns the critic's restoring feedback into instability (\suppref{lem:supp-critical-delay}). For a critic whose error arises only from this lag, the error is still the conditional mean of $\bm r$, so the bound of \cref{eq:conditional-removal-bound} applies to it and the projection attenuates the lag error in conditional mean square (\suppref{cor:supp-lag-attenuation}).

\subsection{2D Verification}
\label{sec:analysis-2d}

\paragraph{Error filtering.}
We test the error-filtering performance on a ring of eight Gaussians with an exact teacher score, so the critic is the only learned score. We compare DMD with four variants that each remove one direction from its update (\cref{fig:2d-curves}; \suppref{sec:supp-2d-exp1}). The other projections remove a similar share of the update norm without targeting the critic error. Only \methodname{} improves both distributional metrics over DMD. \Suppref{fig:supp-2d-gamma,tab:supp-2d-net} measure $\gamma_e$, the ideal-signal energy removal fraction $\gamma_s$, and the net error reduction directly along the run: the projection reduces mean squared error to the ideal update by $28\%$ relative to DMD for $\sigma\in[0.15,1.1]$.

\begin{figure}[t]
    \centering
    \begin{minipage}[t]{0.655\linewidth}
        \vspace{0pt}
        \centering
        \includegraphics[width=\linewidth]{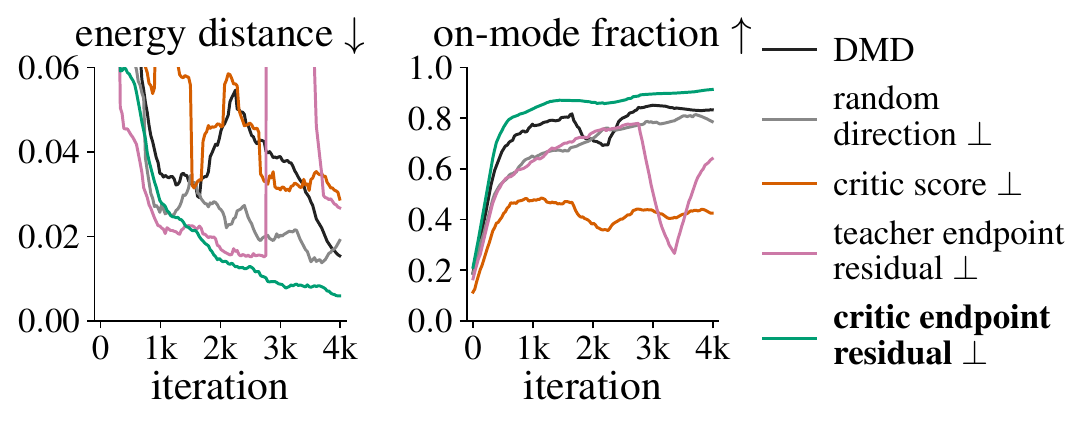}
        \caption{\textbf{2D comparison of projection directions} (20-point moving average).
        Left: energy distance to the target. Right: fraction of samples within $3\sigma$ of a mode center. Among the tested directions, only the \textcolor{pdmdgreen}{\methodname{}} projection stabilizes training and improves both metrics over DMD.}
        \label{fig:2d-curves}
    \end{minipage}
    \hfill
    \begin{minipage}[t]{0.315\linewidth}
        \vspace{0pt}
        \centering
        \includegraphics[width=\linewidth]{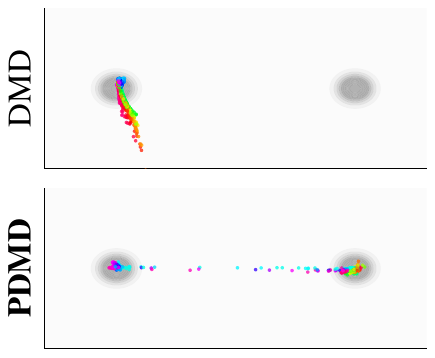}
        \caption{\textbf{Two-mode distillation with lagging critic}. This run illustrates DMD collapse (top) and \methodname{} coverage of both modes (bottom).}
        \label{fig:2d-collapse}
    \end{minipage}
\end{figure}

\paragraph{Collapse under critic lag.}
With the critic learning rate at $1/20$ of the student's on a two-mode target, \dmd collapses in $14/20$ and $19/20$ of the runs across two learning rates, a random projection in $16/20$ and $19/20$, and \methodname{} in $3/20$ and $12/20$ (\cref{fig:2d-collapse}; \suppref{sec:supp-2d-exp2}). \methodname{} typically recovers from an early single-mode state, whereas collapsed \dmd runs remain locked.

\section{Experiments}
\label{sec:experiments}

\subsection{Experimental Setup}
\label{sec:experimental-setup}

\paragraph{Wan2.1-T2V-1.3B.}
\dmdtwodag is our reimplementation of \dmdtwo~\citep{yin2024improved} on Wan2.1-T2V-1.3B at 4 NFE. DMD$^{\dagger}$ is our DMD reimplementation with one critic update per student update and no GAN loss. \methodname differs from DMD$^{\dagger}$ only by the projection in \cref{eq:method-pdmd-projection}. All training hyperparameters are unchanged. We fine-tune all weights. For \rcm~\citep{zheng2026rcm}, ADV~\citep{you2026adaptive}, and \anyflow~\citep{gu2026anyflow}, we use their official checkpoints.

\newsavebox{\wantabbox}
\sbox{\wantabbox}{\scalebox{0.8}{%
    \small
    \setlength{\tabcolsep}{5pt}%
    \begin{tabular}{@{}lccc@{\hspace{4pt}}cc*{3}{>{\centering\arraybackslash}p{11.6mm}}@{}}%
        \toprule
        & & \multicolumn{4}{c}{VBench} & \multicolumn{3}{c}{User study (win $-$ loss, \%)} \\
        \cmidrule(lr){3-6} \cmidrule(lr){7-9}
        Method & NFE & \textbf{Total} & Quality & Dynamic & Semantic & Text & Visual & Motion \\
        \midrule
        \deemph{Wan2.1-T2V-1.3B~\deemphcite{wan2025wan}} &
        \deemph{$50{\times}2$} & \deemph{83.06} & \deemph{85.02} &
        \deemph{68.61} & \deemph{75.23} & \deemph{$-0.4$} & \deemph{$+17.6$} & \deemph{$+24.3$} \\
        \deemph{Wan2.1-T2V-1.3B~\deemphcite{wan2025wan}}& \deemph{$4{\times}2$}& \deemph{68.54}& \deemph{73.60}& \deemph{18.61}& \deemph{48.30} & \deemph{--} & \deemph{--} & \deemph{--} \\
        rCM~\citep{zheng2026rcm} & 4 & 83.13 & 85.14 & 78.06 & 75.10 & $-2.1$ & $+33.8$ & $+17.7$ \\
        ADV~\citep{you2026adaptive} & 4 & 81.47 & 83.61 & 63.89 & 72.91 & $+2.9$ & $+34.6$ & $+19.2$ \\
        AnyFlow~\citep{gu2026anyflow} & 4 & 83.54 & 85.28 & 59.44 & \textbf{76.57} & $+0.6$ & $+24.0$ & $+28.9$ \\
        DMD$^{\dagger}$~\citep{yin2024onestep} & 4 &
        82.70 & 85.06 & 86.39 & 73.24 & $+6.5$ & $+53.6$ & $+35.2$ \\
        \dmdtwodag~\citep{yin2024improved} & 4 &
        83.44 & 85.84 & 76.67 & 73.84 & $+7.5$ & $+39.5$ & $+30.9$ \\
        \textbf{\methodname{} (ours)} & 4 &
        \cellcolor{bytedancebluebg}\textbf{83.73} &
        \cellcolor{bytedancebluebg}\textbf{85.89} &
        \cellcolor{bytedancebluebg}\textbf{89.72} &
        \cellcolor{bytedancebluebg}75.09 & -- & -- & -- \\
        \bottomrule
    \end{tabular}}}
\begin{table}[t]
    \centering
    \caption{\textbf{Quantitative comparisons for Wan2.1-T2V-1.3B at 4 NFE on VBench.}
    $^{\dagger}$ marks our best reimplemented checkpoints; other baselines use official releases. All methods follow the AnyFlow protocol. \methodname{} leads in quality, dynamic degree, and total score. User-study win--loss shares show substantial preferences for \methodname{} over all compared methods in visual and motion quality. }
    \label{tab:wan13b}
    \usebox{\wantabbox}
\end{table}
\paragraph{MiniMax-H3-33B.}
We distill the 33B joint video--audio DiT MiniMax-H3~\citep{minimaxh3} to 4 NFE. Student and critic use rank-128 LoRA with scaling 128 on every block's attention projections and feed-forward layers. We tune learning rates for DMD$^{\dagger}$, \dmdtwodag, \rcm$^{\dagger}$, and \anyflow$^{\dagger}$, reporting each method's best run; \suppref{tab:supp-h3-baseline-recipes} details the two-stage \rcm$^{\dagger}$ and \anyflow$^{\dagger}$ recipes. We also evaluate the community's 4-step H3 Turbo LoRA~\citep{larry2026h3turbo} (\suppref{sec:h3-eval-protocol}).

\paragraph{Evaluation.}
Wan2.1 uses the 944 AnyFlow-augmented VBench prompts~\citep{huang2023vbench,gu2026anyflow}, at 480p with five seeds per prompt. MiniMax-H3 uses 387 VideoGen-Eval prompts~\citep{yang2025videogeneval}, at 544p with mixed aspect ratios and a fixed seed. Each H3 prompt is rewritten once with Qwen3.8-27B~\citep{qwen2026qwen38} following official guidance~\citep{minimaxh3promptguide}, then shared across methods. Video evaluation follows VBench; audio uses six established metrics. In randomized side-by-side comparisons on matched prompts, 20 annotators choose the better clip or a tie for text alignment, visual and motion quality, plus audio and overall quality for H3. \Cref{tab:wan13b,tab:h3} report win-minus-loss shares for \methodname against each baseline (\suppref{sec:experimental-details}).

\subsection{Four-Step Text-to-Video Generation on Wan2.1-T2V-1.3B}
\label{sec:wan-results}

In \cref{tab:wan13b}, \methodname{} scores 83.73 total, above 83.54 for \anyflow~\citep{gu2026anyflow}, 83.06 for the 50-step teacher, and 82.70 for matched \dmd$^{\dagger}$. \dmd$^{\dagger}$ peaks at 1{,}500 iterations, then degrades to 63.04 at 5{,}000 and 56.09 at 7{,}500; \methodname{} scores 83.44 at 5{,}000 and stays above 83.5 through 10{,}000. \methodname{} is also 0.29 points above \dmdtwodag without its additional GAN loss and achieves a higher dynamic-degree score (89.72 versus 76.67). \Cref{fig:wan_grid_main} qualitatively shows \methodname{} produces higher-quality videos with natural textures and less oversaturation. Annotators prefer \methodname{} to every compared method, the 50-step teacher included, on visual and motion quality; text alignment is mostly ties (\cref{tab:wan13b}; \suppref{tab:supp-wan-userstudy}).

\begin{figure}[!t]
    \centering
    \setlength{\tabcolsep}{0pt}
    \renewcommand{\arraystretch}{0}
    \newcommand{\wgl}[1]{\makebox[0.16in][c]{\raisebox{-0.5\height}[0pt][0pt]{\rotatebox{90}{\scriptsize Prompt #1}}}}
    \newcommand{\wgw}{\dimexpr(\textwidth-0.16in-9pt)/8\relax}
    \newcommand{\wgh}[2]{\parbox[b]{\wgw}{\centering\scriptsize #1\\[1pt]\scriptsize #2}}
    \newcommand{\wg}[1]{\includegraphics[width=\wgw]{figures/src/wan_grid/#1.jpg}}
    \newcommand{\wgblock}[2]{%
        \wgl{#1} & \wg{t50_#1_#2} & \wg{t4_#1_#2} & \wg{rcm_#1_#2} & \wg{adv_#1_#2} & \wg{af_#1_#2} & \wg{dmd_#1_#2} & \wg{dmd2_#1_#2} & \wg{pdmd_#1_#2} \\
        & \wg{t50_#1_80} & \wg{t4_#1_80} & \wg{rcm_#1_80} & \wg{adv_#1_80} & \wg{af_#1_80} & \wg{dmd_#1_80} & \wg{dmd2_#1_80} & \wg{pdmd_#1_80} \\[2pt]}
    \begin{tabular}{@{}r@{\hspace{2pt}}c@{\hspace{1pt}}c@{\hspace{1pt}}c@{\hspace{1pt}}c@{\hspace{1pt}}c@{\hspace{1pt}}c@{\hspace{1pt}}c@{\hspace{1pt}}c@{}}
        & \wgh{Wan2.1}{$50{\times}2$ NFE}
        & \wgh{Wan2.1}{$4{\times}2$ NFE}
        & \wgh{rCM}{4 NFE}
        & \wgh{ADV}{4 NFE}
        & \wgh{AnyFlow}{4 NFE}
        & \wgh{DMD$^{\dagger}$}{4 NFE}
        & \wgh{\dmdtwodag}{4 NFE}
        & \wgh{\textbf{\methodname{} (Ours)}}{\textbf{4 NFE}} \\[2pt]
        \wgblock{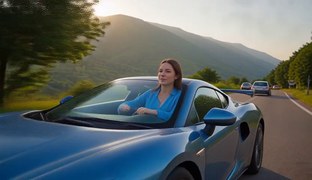}{30}
        \wgblock{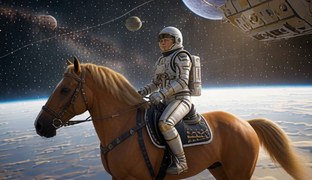}{30}
        \wgblock{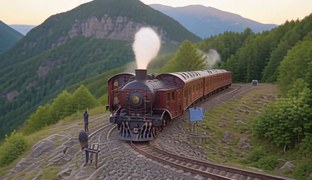}{30}
    \end{tabular}
    \caption{\textbf{Qualitative results on four-step Wan2.1-T2V-1.3B.} Each example shows early and late frames from the same VBench prompt and seed across methods, following the AnyFlow evaluation protocol~\citep{gu2026anyflow}. In these examples, \methodname{} produces natural textures and less oversaturation than the matched DMD and DMD2 models. See \suppref{fig:wan_grid_supp} for more visual results.}
    \label{fig:wan_grid_main}
\end{figure}

\WFclear

\vspace{-5pt}
\subsection{Four-Step Joint Video--Audio Generation on MiniMax-H3-33B}
\label{sec:h3-results}

\begin{table*}[!tb]
    \centering
    \caption{\textbf{Quantitative 4-NFE MiniMax-H3-33B comparisons on VideoGen-Eval.} $^{\dagger}$ marks rows reporting the best reimplemented checkpoint. \methodname{} reaches the highest visual total and performs best across all audio metrics among the distilled models, and the user study prefers it to 4-NFE baselines on visual, motion, and audio quality. Text-alignment judgments are mostly ties (\suppref{tab:supp-h3-userstudy}). }
    \label{tab:h3}
    \tiny
    \setlength{\tabcolsep}{1.1pt}
    \begin{tabular}{@{}lccc@{\hspace{3pt}}ccccccccccccc@{}}%
        \toprule
        & & \multicolumn{4}{c}{Video} & \multicolumn{6}{c}{Audio} &
        \multicolumn{5}{c}{User study (win $-$ loss, \%)} \\
        \cmidrule(lr){3-6} \cmidrule(lr){7-12} \cmidrule(lr){13-17}
        Method & NFE & \textbf{Total} & Quality & Dynamic & Semantic &
        PQ & CE & CU & IS & IB & DeSync $\downarrow$ &
        \textbf{Overall} & Text & Visual & Motion & Audio \\
        \midrule
        \deemph{MiniMax-H3-33B~\deemphcite{minimaxh3}} & \deemph{50} &
        \deemph{82.41} & \deemph{82.22} & \deemph{66.67} &
        \deemph{83.20} & \deemph{6.567} & \deemph{4.188} &
        \deemph{6.213} & \deemph{5.15} & \deemph{0.229} & \deemph{0.797} & \deemph{$-37.0$} & \deemph{$-7.5$} & \deemph{$-35.4$} & \deemph{$-11.9$} & \deemph{$-3.9$} \\
        \deemph{MiniMax-H3-33B~\deemphcite{minimaxh3}} & \deemph{4} &
        \deemph{79.48} & \deemph{79.54} & \deemph{44.44} & \deemph{79.23} & \deemph{6.056} & \deemph{3.167} & \deemph{5.580} & \deemph{3.35} & \deemph{0.129} & \deemph{0.932} & \deemph{$+70.3$} & \deemph{$+16.5$} & \deemph{$+72.6$} & \deemph{$+44.7$} & \deemph{$+47.5$} \\
        H3 Turbo LoRA~\citep{larry2026h3turbo} & 4 &
        81.57 & 81.29 & 54.78 & 82.69 & 6.406 & 3.917 & 6.015 & 4.52 & 0.183 & 0.839 & $+34.6$ & $+2.8$ & $+38.0$ & $+19.4$ & $+9.3$ \\
        \rcm$^{\dagger}$~\citep{zheng2026rcm} & 4 &
        81.18 & 80.98 & 58.40 & 81.95 & 6.063 & 3.711 & 5.446 & 3.69 & 0.194 & 0.916 & $+43.7$ & $+3.1$ & $+51.2$ & $+31.5$ & $+17.1$ \\
        \anyflow$^{\dagger}$~\citep{gu2026anyflow} & 4 &
        81.97 & 81.82 & 64.60 & 82.60 & 6.092 & 3.544 & 5.591 & 4.35 & 0.161 & 0.905 & $+59.2$ & $-1.6$ & $+66.4$ & $+27.9$ & $+25.8$ \\
        \dmd$^{\dagger}$~\citep{yin2024onestep} & 4 &
        82.76 & 82.68 & 61.76 & \textbf{83.05} & 6.381 & 3.801 & 5.931 & 4.79 & 0.176 & 0.905 & $+29.7$ & $+3.6$ & $+25.1$ & $+12.1$ & $+10.9$ \\
        \dmdtwo$^{\dagger}$~\citep{yin2024improved} & 4 &
        82.27 & 82.51 & 59.95 & 81.34 & 6.305 & 3.434 & 5.871 & 4.06 & 0.153 & 0.921 & $+35.1$ & $+9.0$ & $+28.7$ & $+19.4$ & $+22.2$ \\
        \textbf{\methodname{} (ours)} & 4 &
        \cellcolor{bytedancebluebg}\textbf{83.17} & \cellcolor{bytedancebluebg}\textbf{83.25} & \cellcolor{bytedancebluebg}\textbf{71.83} & \cellcolor{bytedancebluebg}82.86 & \cellcolor{bytedancebluebg}\textbf{6.530} & \cellcolor{bytedancebluebg}\textbf{4.062} & \cellcolor{bytedancebluebg}\textbf{6.180} & \cellcolor{bytedancebluebg}\textbf{4.98} & \cellcolor{bytedancebluebg}\textbf{0.195} & \cellcolor{bytedancebluebg}\textbf{0.802} & -- & -- & -- & -- & -- \\
        \bottomrule
    \end{tabular}
\end{table*}

\paragraph{Video.}
\methodname{} scores 83.17 total, above every 4-NFE baseline; \dmd and \anyflow reach 82.76 and 81.97, respectively. \dmd$^{\dagger}$ peaks at 5{,}500 iterations with a tenfold lower learning rate; at \methodname{}'s learning rate, it oversaturates from iteration 500 and subsequently collapses. \Cref{fig:h3_grid_main,fig:h3_motion_main} show more realistic textures, more dynamic motion, and sharper details with \methodname{}.

\begin{figure}[!t]
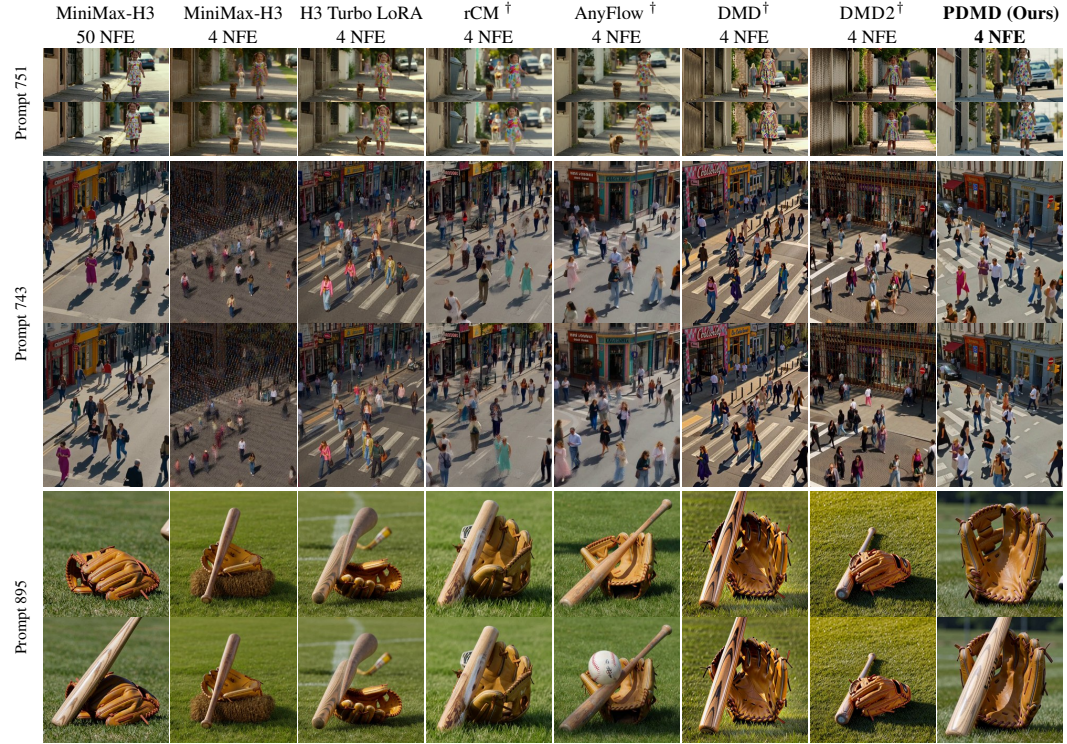

    \centering
    \setlength{\tabcolsep}{0pt}
    \renewcommand{\arraystretch}{0}
    \newcommand{\hgml}[1]{\makebox[0.16in][c]{\raisebox{-0.5\height}[0pt][0pt]{\rotatebox{90}{\tiny Prompt #1}}}}
    \newcommand{\hgmw}{\dimexpr(\textwidth-0.16in-9pt)/8\relax}
    \newcommand{\hgmh}[2]{\parbox[b]{\hgmw}{\centering\scriptsize #1\\[1pt]\scriptsize #2}}
    \newcommand{\hgm}[1]{\includegraphics[width=\hgmw]{figures/src/h3_grid/#1.jpg}}
    \newcommand{\hgmrow}[3]{%
        \hgml{#1} & \hgm{t50_#1_f#2} & \hgm{t4_#1_f#2} & \hgm{lora_#1_f#2} & \hgm{rcm_#1_f#2}
        & \hgm{af_#1_f#2} & \hgm{dmd_#1_f#2} & \hgm{dmd2_#1_f#2} & \hgm{pdmd_#1_f#2} \\
        & \hgm{t50_#1_f#3} & \hgm{t4_#1_f#3} & \hgm{lora_#1_f#3} & \hgm{rcm_#1_f#3}
        & \hgm{af_#1_f#3} & \hgm{dmd_#1_f#3} & \hgm{dmd2_#1_f#3} & \hgm{pdmd_#1_f#3} \\[2pt]}
    \begin{tabular}{@{}r@{\hspace{2pt}}c@{\hspace{1pt}}c@{\hspace{1pt}}c@{\hspace{1pt}}c@{\hspace{1pt}}c@{\hspace{1pt}}c@{\hspace{1pt}}c@{\hspace{1pt}}c@{}}
        & \hgmh{MiniMax-H3}{50 NFE}
        & \hgmh{MiniMax-H3}{4 NFE}
        & \hgmh{H3 Turbo LoRA}{4 NFE}
        & \hgmh{\rcm$^{\dagger}$}{4 NFE}
        & \hgmh{\anyflow$^{\dagger}$}{4 NFE}
        & \hgmh{DMD$^{\dagger}$}{4 NFE}
        & \hgmh{\dmdtwodag}{4 NFE}
        & \hgmh{\textbf{\methodname{} (Ours)}}{\textbf{4 NFE}} \\[2pt]
        \hgmrow{751}{0}{25}
        \hgmrow{743}{68}{113}
        \hgmrow{895}{5}{82}
    \end{tabular}
    \caption{\textbf{Qualitative results for 4-NFE MiniMax-H3 on VideoGen-Eval.} \methodname{} shows fewer artifacts, more realistic textures, and better motion. Each prompt shows an earlier and a later frame of the same clip; \suppref{sec:supp-h3-qualitative} shows more prompts.}
    \label{fig:h3_grid_main}
\vspace{-10pt}
\end{figure}

\begin{figure}[!t]
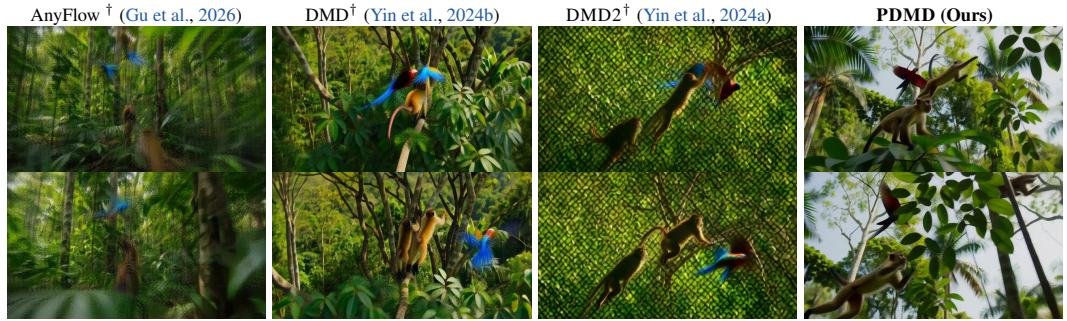

    \centering
    \setlength{\tabcolsep}{0pt}
    \renewcommand{\arraystretch}{0}
    \newcommand{\hmw}{\dimexpr(\textwidth-9pt)/4\relax}
    \newcommand{\hmf}[2]{\includegraphics[width=\hmw]{figures/src/motion969/h3_#1_969_#2.jpg}}
    \newcommand{\hmt}[1]{\parbox[b]{\hmw}{\centering\scriptsize #1}}
    \begin{tabular}{@{}c@{\hspace{3pt}}c@{\hspace{3pt}}c@{\hspace{3pt}}c@{}}
        \hmt{\anyflow$^{\dagger}$~\citep{gu2026anyflow}} &
        \hmt{DMD$^{\dagger}$~\citep{yin2024onestep}} &
        \hmt{\dmdtwodag~\citep{yin2024improved}} &
        \hmt{\textbf{\methodname{} (Ours)}} \\[2pt]
        \hmf{af}{t075} & \hmf{dmd}{t075} & \hmf{dmd2}{t075} & \hmf{pdmd}{t075} \\
        \hmf{af}{t200} & \hmf{dmd}{t200} & \hmf{dmd2}{t200} & \hmf{pdmd}{t200} \\
    \end{tabular}
    \caption{\textbf{Motion on 4-NFE MiniMax-H3} (VideoGen-Eval prompt 760).
    The compared baselines show translucent ghosting artifacts and severely blur the animals, while \methodname{} keeps them sharp.}
    \label{fig:h3_motion_main}
\vspace{-10pt}
\end{figure}

\vspace{-16pt}
\paragraph{Audio.}
\methodname{} leads the compared 4-NFE models on all six audio metrics (\cref{tab:h3}), coming within $0.04$ of the 50-step teacher on production quality (PQ). Metric definitions are in \suppref{sec:supp-h3-audio-crossmodal}.

\vspace{-8pt}
\paragraph{User study.}
Annotators prefer \methodname{} to every compared 4-NFE model in overall, visual, motion, and audio quality (\cref{tab:h3}). Text-alignment judgments are mostly ties (64--80\%; \suppref{tab:supp-h3-userstudy}). Unlike on Wan2.1, the 50-step H3 teacher remains preferred by 37.0 percentage points overall and 35.4 in visual quality. This larger residual gap may reflect the frontier teacher's higher quality ceiling, making four-step distillation more challenging than on Wan2.1.

\vspace{-8pt}
\section{Ablations}
\label{sec:ablations}

\ifdefined\pzbase\else\newlength{\pzbase}\fi
\ifdefined\pzcellw\else\newlength{\pzcellw}\fi
\ifdefined\pzcellh\else\newlength{\pzcellh}\fi
\newcommand{\pzsetup}{%
    \setlength{\tabcolsep}{0pt}%
    \renewcommand{\arraystretch}{0}%
    \global\setlength{\pzbase}{\dimexpr(\textwidth-5pt)/6\relax}%
    \global\setlength{\pzcellw}{\pzbase}%
    \global\setlength{\pzcellh}{\pzbase}%
}
\newcommand{\pzp}[1]{\makebox[0.17in][c]{%
    \raisebox{\dimexpr0.5\pzcellh-0.5\height\relax}[0pt][0pt]{%
    \rotatebox{90}{\scriptsize Prompt #1}}}}
\newcommand{\pzi}[1]{\makebox[0.15in][r]{%
    \raisebox{\dimexpr0.5\pzcellh-0.5\height\relax}[0pt][0pt]{\tiny #1}}}
\newcommand{\pzh}[2]{\makebox[\pzbase][c]{%
    \shortstack{\scriptsize #1\\[1pt]\scriptsize #2}}}
\newcommand{\pz}[1]{\includegraphics[width=\pzcellw]%
    {figures/src/projection_ablation_387/#1}}
\newcommand{\pzrow}[2]{\pz{dmd_#1_#2} & \pz{rand_#1_#2} & \pz{cscore_#1_#2} &
    \pz{tres_#1_#2} & \pz{rpar_#1_#2} & \pz{pdmd_#1_#2}}
\newcommand{\pzhead}{%
    \pzh{DMD$^{\dagger}$}{(no projection)}
    & \pzh{Random}{direction $\perp$}
    & \pzh{Critic}{score $\perp$}
    & \pzh{Teacher endpoint}{residual $\perp$}
    & \pzh{Critic endpoint}{residual $\parallel$}
    & \pzh{\textbf{Critic endpoint}}{\textbf{residual $\perp$ (\methodname{})}}}

\begin{figure}[!t]
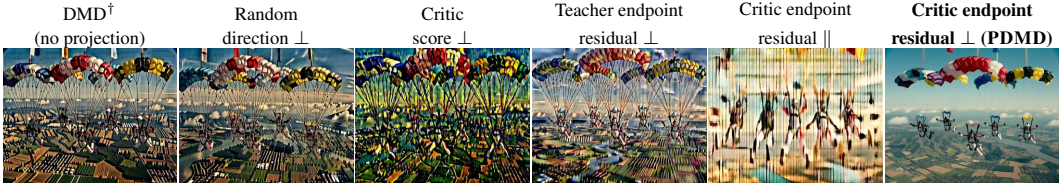

    \centering
    \pzsetup
    \begin{tabular}{@{}c@{\hspace{1pt}}c@{\hspace{1pt}}c@{\hspace{1pt}}c@{\hspace{1pt}}c@{\hspace{1pt}}c@{}}
        \pzhead \\[2pt]
        \global\setlength{\pzcellw}{\pzbase}\global\setlength{\pzcellh}{\dimexpr\pzbase*640/832\relax}%
        \pzrow{882}{1500} \\
    \end{tabular}
    \caption{\textbf{Qualitative ablation of the projection direction on VideoGen-Eval} (prompt 882). Results at 1,500 training iterations. In this example, projecting out the update component parallel to the student--critic endpoint residual yields less saturation than the other variants.}
    \label{fig:projection-ablation-main}
\end{figure}

\vspace{-8pt}
\paragraph{Stability compared with alternative projections.}
\label{sec:ablation-dynamics}
Six MiniMax-H3 variants differ only in the projection applied to the \dmd update. At 1,500 iterations, only \methodname avoids saturation or high-frequency artifacts in \cref{fig:projection-ablation-main}; \suppref{fig:projection-ablation} follows three prompts through 3,000 iterations. \Cref{fig:h3-score-curves} compares training dynamics quantitatively on all 387 VideoGen-Eval prompts: from 500 to 3,000 iterations, \methodname maintains its total score while the other variants decline.
\begin{figure}[!t]
    \centering
    \includegraphics[width=\linewidth]{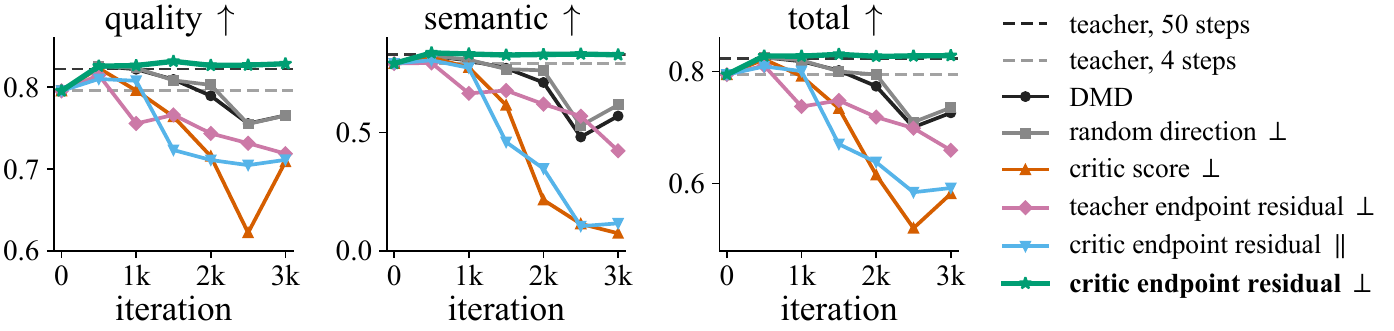}
    \caption{\textbf{Quantitative ablation study: training dynamics of MiniMax-H3 on VideoGen-Eval.} \textcolor{pdmdgreen}{\methodname{}} holds its total score; the others decline.}
    \label{fig:h3-score-curves}
\end{figure}

\paragraph{Direction matters more than magnitude.}
Random projection removes almost no update norm and behaves like \dmd. \methodname removes a median $12\%$ and remains stable over the evaluated interval. The residual-parallel variant keeps only $11\%$ and degrades fastest, consistent with the parallel component carrying destabilizing error (\suppref{sec:h3-norm-ratio}). Restoring half of the removed component also leads to degradation, but later than \dmd (\suppref{sec:h3-partial-retention}). Raising the \methodname student learning rate by $1.25\times$ preserves stability through 5,000 iterations, ending at 83.18 total (\suppref{sec:h3-larger-step}). \dmd at one tenth of the learning rate still oversaturates and peaks at 82.76 total at 5,500 iterations, below \methodname (\suppref{tab:supp-sat}). Together, these controls show that update shrinkage alone does not explain the stability gain. In addition, with one critic update per student update, \methodname scores 82.90 total, versus 82.37 for \dmd at the same update ratio, and exceeds every 4-NFE baseline in \cref{tab:h3} (see \suppref{tab:supp-h3-critic-ratio}).

\section{Limitations}
\label{sec:limitations}

\begin{wrapfigure}[6]{r}{0.50\textwidth}
        \vspace{-\intextsep}
        \centering
        \makebox[0.33\linewidth]{\scriptsize \methodname\ (4 NFE)}%
        \makebox[0.34\linewidth]{\scriptsize \methodname\ (2 NFE)}%
        \makebox[0.33\linewidth]{\scriptsize \methodname\ (1 NFE)}\\[2pt]
        \includegraphics[width=\linewidth]{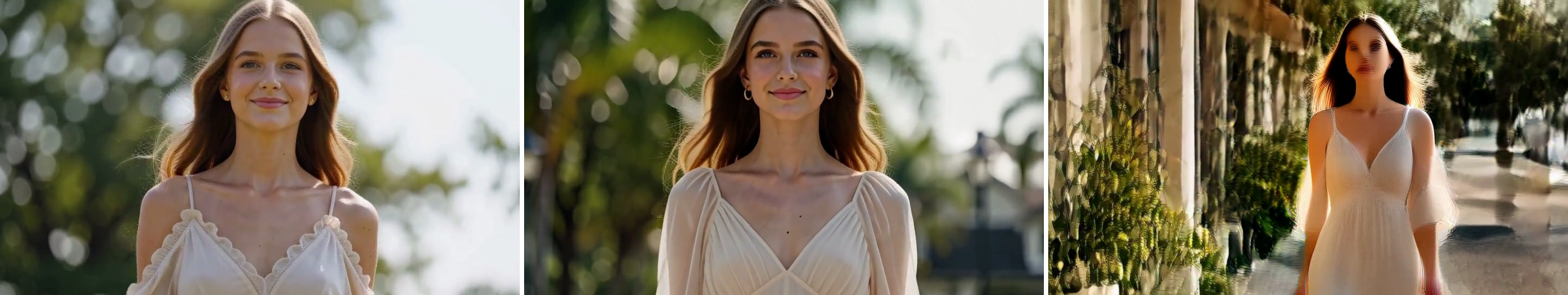}
        \caption{\textbf{Single-step limitation.}}
        \label{fig:limitation}
\end{wrapfigure}

Single-step quality remains limited (\cref{fig:limitation}) and may require an additional objective, such as a discriminator loss. See \cref{sec:h3-two-step,fig:h3_nfe2_samples} for more 2-NFE results and \cref{sec:failure-cases} for further limitations in the supplement.

\section{Conclusion}
\label{sec:conclusion}

We introduced \methodname{}, a one-line DMD modification that projects out the component of the student update parallel to the student--critic endpoint residual, an unbiased estimate of critic error at a fixed query. Under high-dimensional assumptions, the projection removes a constant fraction of critic error while discarding a vanishing fraction of useful signal. \methodname{} requires no additional objective, network, model pass, data, or training stage. Across Wan2.1 and MiniMax-H3, \methodname{} remains stable in settings where DMD degrades, improves video scores, is preferred to the compared 4-NFE baselines in user studies, and performs best on all audio metrics in \cref{tab:h3}. These results highlight \methodname as a promising approach to stable few-step video distillation. \methodname{} further opens up opportunities to combine critic-error filtering with existing distillation methods as a future work.\label{sec:main-text-end}

\subsection*{Ethics Statement}

This work distills two released text-to-video models using existing training prompts and public evaluation benchmarks, without real-video data. The only newly collected data are human preference judgments for the user studies. A distilled student inherits the capabilities and biases of its teacher, and faster generation lowers the cost of producing misleading video. We have not evaluated whether distillation preserves the safety behavior of the base models.

\subsection*{Reproducibility Statement}

We will release the training and evaluation code of the video and 2D experiments, the distilled Wan2.1-T2V-1.3B student weights, and the MiniMax-H3 distillation LoRA. \Cref{alg:pdmd-main} and \cref{eq:method-pdmd-projection} give the complete student update; the projection is the only change from \dmd, so the method adds one line to an existing \dmd implementation. \Suppref{sec:experimental-details} lists every training setting of the video experiments, including the trainable parameters, data, clip size, batch, optimizer, learning rates, critic update ratio and samplers (\suppref{tab:supp-train-config}), and describes the baselines and the evaluation protocols of MiniMax-H3 (\suppref{sec:h3-eval-protocol}) and Wan2.1-T2V-1.3B (\suppref{sec:supp-wan-protocol}), with the prompt sets, seeds, samplers and metric definitions. \Suppref{sec:supp-2d-experiments} specifies the 2D experiments, and \suppref{sec:proofs} gives the full proofs of the theoretical results.

\subsection*{AI use statement}

Generative AI tools assisted with theoretical analysis, implementation, and language polishing. The authors reviewed all AI-assisted content, checked the derivations and proofs, conducted and validated the experiments, and take responsibility for the final content.

\bibliography{main}
\bibliographystyle{iclr2027_conference}

\clearpage
\appendix
\addtocontents{toc}{\protect\setcounter{tocdepth}{2}}
\numberwithin{figure}{section}
\numberwithin{table}{section}
\numberwithin{equation}{section}
\numberwithin{theorem}{section}
\renewcommand{\thetheorem}{\thesection\arabic{theorem}}
\renewcommand{\thelemma}{\thesection\arabic{theorem}}
\renewcommand{\theproposition}{\thesection\arabic{theorem}}
\renewcommand{\thecorollary}{\thesection\arabic{theorem}}
\renewcommand{\thefigure}{\thesection\arabic{figure}}
\renewcommand{\thetable}{\thesection\arabic{table}}
\renewcommand{\theequation}{\thesection\arabic{equation}}

\renewcommand{\contentsname}{Appendix}
\setcounter{tocdepth}{2}
\tableofcontents

\section{Proofs}
\label{sec:proofs}

\subsection{\texorpdfstring{$\ell_2$}{L2}-Optimal Critic Endpoint}
\label{sec:proof-l2-optimal-critic}

Let $(Q,\bm S)$ denote the critic-training query--endpoint pair.  We first establish the conditional-moment identities induced by the critic's $\ell_2$ endpoint-training objective.

\begin{lemma}[$\ell_2$-optimal critic endpoint]
\label{lem:supp-l2-conditional-mean}
Assume $\mathbb E[\|\bm S\|^2\mid Q=q]<\infty$.  The unique minimizer of the conditional squared-error risk
\begin{equation}
    \mathcal L_q(\bm v)
    :=
    \mathbb E[\|\bm v-\bm S\|^2\mid Q=q]
    \label{eq:supp-conditional-risk}
\end{equation}
is
\begin{equation}
    \bm c^\star
    =
    \arg\min_{\bm v}\mathcal L_q(\bm v)
    =
    \mathbb E[\bm S\mid Q=q].
    \label{eq:supp-l2-optimal-endpoint}
\end{equation}
Consequently, for
\begin{equation}
    \bm\xi=\bm S-\bm c^\star,
    \qquad
    \bm\Sigma=\operatorname{Cov}(\bm S\mid Q=q),
    \label{eq:supp-conditional-residual}
\end{equation}
we have
\begin{equation}
    \mathbb E[\bm\xi\mid Q=q]=\bm 0,
    \qquad
    \mathbb E[\bm\xi\bm\xi^{\mathsf T}\mid Q=q]=\bm\Sigma,
    \qquad
    \mathbb E[\|\bm\xi\|^2\mid Q=q]=\operatorname{tr}\bm\Sigma.
    \label{eq:supp-residual-moments}
\end{equation}
\end{lemma}

\begin{proof}
Let $\bm m=\mathbb E[\bm S\mid Q=q]$.  For any deterministic prediction $\bm v$, write
\begin{equation}
    \bm v-\bm S
    =
    (\bm v-\bm m)+(\bm m-\bm S).
\end{equation}
Expanding its squared norm and taking the conditional expectation gives
\begin{align}
    \mathcal L_q(\bm v)
    &=
    \|\bm v-\bm m\|^2
    +\mathbb E[\|\bm S-\bm m\|^2\mid Q=q] \notag\\
    &\quad
    +2(\bm v-\bm m)^{\mathsf T}
      \mathbb E[\bm m-\bm S\mid Q=q].
    \label{eq:supp-risk-expansion}
\end{align}
By the definition of $\bm m$,
\begin{equation}
    \mathbb E[\bm m-\bm S\mid Q=q]
    =
    \bm m-\mathbb E[\bm S\mid Q=q]
    =
    \bm 0.
\end{equation}
Hence
\begin{equation}
    \mathcal L_q(\bm v)
    =
    \mathcal L_q(\bm m)+\|\bm v-\bm m\|^2.
    \label{eq:supp-risk-pythagorean}
\end{equation}
The second term is nonnegative and vanishes if and only if $\bm v=\bm m$, which proves both optimality and uniqueness in \cref{eq:supp-l2-optimal-endpoint}.

Substituting $\bm c^\star=\bm m$ into the definition of $\bm\xi$ gives
\begin{equation}
    \mathbb E[\bm\xi\mid Q=q]
    =
    \mathbb E[\bm S\mid Q=q]-\bm c^\star
    =
    \bm 0.
\end{equation}
Because $\bm c^\star$ is the conditional mean, the definition of conditional covariance becomes
\begin{equation}
    \bm\Sigma
    =
    \mathbb E\!\left[
        (\bm S-\bm c^\star)(\bm S-\bm c^\star)^{\mathsf T}
        \mid Q=q
    \right]
    =
    \mathbb E[\bm\xi\bm\xi^{\mathsf T}\mid Q=q].
\end{equation}
Finally, using $\|\bm\xi\|^2=\operatorname{tr} (\bm\xi\bm\xi^{\mathsf T})$ and linearity of conditional expectation and trace,
\begin{align}
    \mathbb E[\|\bm\xi\|^2\mid Q=q]
    &=
    \mathbb E[\operatorname{tr}(\bm\xi\bm\xi^{\mathsf T})\mid Q=q]\\
    &=
    \operatorname{tr}\mathbb E[\bm\xi\bm\xi^{\mathsf T}\mid Q=q]
    =
    \operatorname{tr}\bm\Sigma.
\end{align}
This proves \cref{eq:supp-residual-moments}.
\end{proof}

\subsection{High-Dimensional Error Removal and Signal Retention}
\label{sec:proof-high-dimensional-filtering}

We make precise the distinction between the structural alignment of critic error with the \methodname{} residual and the non-structural alignment of an ideal update with a rank-one direction.  Throughout this subsection, we condition on a fixed query $Q=q$ and suppress this conditioning from the notation.  We use the critic-training pair $(Q,\bm S)$ and the corresponding quantities $\bm c^\star$, $\bm\xi$, and $\bm\Sigma$ defined in \cref{lem:supp-l2-conditional-mean}.  The endpoint of the pair is the student endpoint, $\bm S=\bm x_0^s$, so the residual below is the $\bm r$ that \methodname{} projects along, read at a fixed query (\cref{sec:analysis-critic-average}).

For the learned critic endpoint $\bm x_0^c$, define
\begin{equation}
    \bm e=\bm x_0^c-\bm c^\star,
    \qquad
    \bm R=\bm x_0^c-\bm S=\bm e-\bm\xi.
    \label{eq:supp-structural-residual}
\end{equation}
Thus, unlike a generic direction, $\bm R$ contains the critic error $\bm e$ explicitly.

For $\bm e\neq\bm 0$ and $\bm R\neq\bm 0$, let $\bm P_{\bm R}$ denote the orthogonal projector onto $\bm R$ and define the critic-error removal ratio
\begin{equation}
    \gamma_e(\bm R)
    :=
    \frac{\|\bm P_{\bm R}\bm e\|^2}{\|\bm e\|^2}
    =
    \frac{(\bm e^{\mathsf T}\bm R)^2}
         {\|\bm e\|^2\|\bm R\|^2}.
    \label{eq:supp-removal-ratio-expanded}
\end{equation}
Set $\gamma_e(\bm 0)=0$ when the residual vanishes.

We next characterize this ratio for a typical realization under explicit high-dimensional assumptions.  Define the covariance effective dimension
\begin{equation}
    d_{\mathrm{eff}}
    :=
    \frac{(\operatorname{tr}\bm\Sigma)^2}
         {\operatorname{tr}(\bm\Sigma^2)}.
    \label{eq:supp-effective-dimension}
\end{equation}

\begin{theorem}[High-dimensional critic-error removal]
\label{thm:supp-typical-removal}
Consider a sequence of conditional problems for which $d_{\mathrm{eff}}\to\infty$.  Write
\begin{equation}
    a=\|\bm e\|^2,
    \qquad
    \tau=\operatorname{tr}\bm\Sigma.
\end{equation}
Assume $a,\tau>0$ and the concentration conditions
\begin{equation}
    \frac{\|\bm\xi\|^2-\tau}{\tau}
    =O_{\mathbb P}(d_{\mathrm{eff}}^{-1/2}),
    \qquad
    \frac{\bm e^{\mathsf T}\bm\xi}{\sqrt{a\tau}}
    =O_{\mathbb P}(d_{\mathrm{eff}}^{-1/2}).
    \label{eq:supp-removal-assumptions}
\end{equation}
Then
\begin{equation}
    \gamma_e(\bm R)
    =
    \frac{a}{a+\tau}
    +O_{\mathbb P}\!\left(d_{\mathrm{eff}}^{-1/2}\right).
    \label{eq:supp-typical-removal}
\end{equation}
No balance assumption on $a/\tau$ is required.  If additionally $a/\tau=\Omega(1)$, then $\gamma_e(\bm R)=\Theta_{\mathbb P}(1)$.
\end{theorem}

\begin{proof}
Let
\begin{equation}
    v=\frac{\|\bm\xi\|^2-\tau}{\tau},
    \qquad
    b=\frac{\bm e^{\mathsf T}\bm\xi}{\sqrt{a\tau}},
    \qquad
    \kappa=\frac{a}{\tau}.
    \label{eq:supp-normalized-fluctuations}
\end{equation}
By assumption, $v,b=O_{\mathbb P}(d_{\mathrm{eff}}^{-1/2})$.  Substituting $\bm R=\bm e-\bm\xi$ into \cref{eq:supp-removal-ratio-expanded} gives
\begin{equation}
    \gamma_e(\bm R)
    =
    \frac{(a-\bm e^{\mathsf T}\bm\xi)^2}
         {a(a+\|\bm\xi\|^2-2\bm e^{\mathsf T}\bm\xi)}
    =
    \frac{(\sqrt{\kappa}-b)^2}
         {\kappa+1+v-2\sqrt{\kappa}b}.
    \label{eq:supp-normalized-removal}
\end{equation}
Subtracting $\kappa/(\kappa+1)$ exactly yields
\begin{equation}
    \gamma_e(\bm R)-\frac{\kappa}{\kappa+1}
    =
    \frac{-2\sqrt{\kappa}b+(\kappa+1)b^2-\kappa v}
         {(\kappa+1)(\kappa+1+v-2\sqrt{\kappa}b)}.
    \label{eq:supp-removal-difference}
\end{equation}
The perturbation of the second denominator factor relative to $\kappa+1$ satisfies
\begin{equation}
    \frac{|v-2\sqrt{\kappa}b|}{\kappa+1}
    \leq |v|+|b|=o_{\mathbb P}(1),
\end{equation}
where $2\sqrt{\kappa}/(\kappa+1)\leq1$.  Hence that factor is comparable to $\kappa+1$ with probability tending to one.  On this event, using $\sqrt{\kappa}/(\kappa+1)\leq1/2$ and $\kappa/(\kappa+1)^2\leq1/4$, the absolute value of \cref{eq:supp-removal-difference} is bounded by a constant multiple of
\begin{equation}
    |b|+b^2+|v|
    =O_{\mathbb P}(d_{\mathrm{eff}}^{-1/2}),
\end{equation}
uniformly over every $\kappa>0$.  Since $\kappa/(\kappa+1)=a/(a+\tau)$, this proves \cref{eq:supp-typical-removal}.  If $a/\tau=\Omega(1)$, the leading term is bounded away from zero, which gives the final statement.
\end{proof}

\begin{corollary}[Gaussian sufficient conditions]
\label{cor:supp-gaussian-removal}
Suppose $\bm\xi\sim\mathcal N(\bm 0,\bm\Sigma)$ and, for a constant $K<\infty$ independent of dimension,
\begin{equation}
    \widehat{\bm e}^{\mathsf T}\bm\Sigma\widehat{\bm e}
    \leq K\frac{\tau}{d_{\mathrm{eff}}},
    \qquad
    \widehat{\bm e}=\frac{\bm e}{\|\bm e\|}.
    \label{eq:supp-gaussian-directional-condition}
\end{equation}
Then the concentration conditions in \cref{eq:supp-removal-assumptions} hold.
\end{corollary}

\begin{proof}
Diagonalizing $\bm\Sigma$ and using the moments of independent standard Gaussian coordinates gives
\begin{equation}
    \operatorname{Var}(\|\bm\xi\|^2)
    =2\operatorname{tr}(\bm\Sigma^2)
    =\frac{2\tau^2}{d_{\mathrm{eff}}}.
    \label{eq:supp-norm-variance}
\end{equation}
Chebyshev's inequality proves the first concentration condition.  Moreover, $\bm e^{\mathsf T}\bm\xi$ is centered Gaussian and
\begin{equation}
    \operatorname{Var}\!\left(
        \frac{\bm e^{\mathsf T}\bm\xi}{\sqrt{a\tau}}
    \right)
    =
    \frac{\widehat{\bm e}^{\mathsf T}\bm\Sigma\widehat{\bm e}}{\tau}
    \leq\frac{K}{d_{\mathrm{eff}}},
\end{equation}
which proves the second condition.
\end{proof}

The preceding result is structural: $\bm R=\bm e-\bm\xi$ contains $\bm e$ explicitly.  Small signal removal requires a different incoherence assumption on the fixed ideal signal direction (the ideal signal is weakly aligned with the residual).

\begin{proposition}[Fixed-signal high-dimensional retention]
\label{prop:supp-signal-retention}
At fixed $q$, let $\bm s_{\mathrm{critic}}^\star(q)$ denote the population-optimal critic score and define the ideal DMD update $\bm d^\star=\bm s_{\mathrm{critic}}^\star(q) -\bm s_{\mathrm{teacher}}(q)$.  Assume $\bm d^\star\neq\bm 0$, and let
\begin{equation}
    \bm s=\frac{\bm d^\star}{\|\bm d^\star\|},
    \qquad
    \gamma_s(\bm R)
    :=
    \frac{\|\bm P_{\bm R}\bm d^\star\|^2}
         {\|\bm d^\star\|^2}
    =
    \frac{(\bm s^{\mathsf T}\bm R)^2}{\|\bm R\|^2},
    \label{eq:supp-signal-directions}
\end{equation}
with $\gamma_s(\bm 0)=0$.  Suppose the concentration conditions in \cref{eq:supp-removal-assumptions} hold and, for a constant $K_s<\infty$ independent of dimension,
\begin{equation}
    (\bm s^{\mathsf T}\bm e)^2
    \leq K_s\frac{a}{d_{\mathrm{eff}}},
    \qquad
    \bm s^{\mathsf T}\bm\Sigma\bm s
    \leq K_s\frac{\tau}{d_{\mathrm{eff}}}.
    \label{eq:supp-signal-incoherence}
\end{equation}
Then
\begin{equation}
    \gamma_s(\bm R)=O_{\mathbb P}(d_{\mathrm{eff}}^{-1}),
    \qquad
    1-\gamma_s(\bm R)=1-O_{\mathbb P}(d_{\mathrm{eff}}^{-1}).
    \label{eq:supp-generic-signal-retention}
\end{equation}
\end{proposition}

\begin{proof}
Because $\mathbb E[\bm\xi]=\bm 0$ and $\mathbb E[\bm\xi\bm\xi^{\mathsf T}]=\bm\Sigma$ under the fixed-query conditioning,
\begin{equation}
    \mathbb E[(\bm s^{\mathsf T}\bm\xi)^2]
    =\bm s^{\mathsf T}\bm\Sigma\bm s
    \leq K_s\frac{\tau}{d_{\mathrm{eff}}}.
\end{equation}
Chebyshev's inequality and \cref{eq:supp-signal-incoherence} therefore give
\begin{equation}
    \bm s^{\mathsf T}\bm e
    =O\!\left(\sqrt{\frac{a}{d_{\mathrm{eff}}}}\right),
    \qquad
    \bm s^{\mathsf T}\bm\xi
    =O_{\mathbb P}\!\left(
        \sqrt{\frac{\tau}{d_{\mathrm{eff}}}}
    \right).
\end{equation}
Since $\bm R=\bm e-\bm\xi$, it follows that
\begin{equation}
    (\bm s^{\mathsf T}\bm R)^2
    =O_{\mathbb P}\!\left(
        \frac{a+\tau}{d_{\mathrm{eff}}}
    \right).
    \label{eq:supp-signal-numerator}
\end{equation}
Writing $v$ and $b$ as in \cref{eq:supp-normalized-fluctuations},
\begin{equation}
    \frac{\|\bm R\|^2}{a+\tau}
    =
    1+\frac{\tau v-2\sqrt{a\tau}b}{a+\tau}
    =1+o_{\mathbb P}(1),
    \label{eq:supp-residual-norm-concentration}
\end{equation}
where we used $\tau/(a+\tau)\leq1$ and $2\sqrt{a\tau}/(a+\tau)\leq1$.  Dividing \cref{eq:supp-signal-numerator} by \cref{eq:supp-residual-norm-concentration} proves the result.
\end{proof}

\begin{corollary}[High-dimensional error--signal separation]
\label{cor:supp-high-dimensional-separation}
Under the conditions of \cref{thm:supp-typical-removal,prop:supp-signal-retention}, if $a/\tau=\Omega(1)$, then
\begin{equation}
    \boxed{
        \gamma_e=\Theta_{\mathbb P}(1),
        \qquad
        \gamma_s=O_{\mathbb P}(d_{\mathrm{eff}}^{-1})
    }.
    \label{eq:supp-order-separation}
\end{equation}
\end{corollary}

\begin{proof}
The conclusion follows immediately from \cref{thm:supp-typical-removal,prop:supp-signal-retention} and $a/\tau=\Omega(1)$.
\end{proof}

We next translate the fractional separation into an error guarantee for the projected update.  At a fixed noise level, let $\bm e'$ denote the score-space critic error, so that $\bm d=\bm d^\star+\bm e'$, and define $\widetilde{\bm d}=\bm P_{\bm R}^{\perp}\bm d$.  The score--endpoint conversion makes $\bm e'$ a scalar multiple of $\bm e$, so its projected energy fraction is also $\gamma_e$.  Orthogonality gives the exact identity
\begin{equation}
    \|\widetilde{\bm d}-\bm d^\star\|^2
    =
    \|\bm e'\|^2(1-\gamma_e)
    +\|\bm d^\star\|^2\gamma_s,
    \label{eq:supp-projected-correction-error}
\end{equation}
and therefore
\begin{equation}
    \|\bm d-\bm d^\star\|^2
    -\|\widetilde{\bm d}-\bm d^\star\|^2
    =
    \|\bm e'\|^2\gamma_e-\|\bm d^\star\|^2\gamma_s.
    \label{eq:supp-net-correction-improvement}
\end{equation}

\begin{corollary}[Net improvement over the DMD update]
\label{cor:supp-net-correction-improvement}
Under the conditions of \cref{cor:supp-high-dimensional-separation}, assume $\bm e'\neq\bm 0$ and
\begin{equation}
    \frac{\|\bm d^\star\|^2}{\|\bm e'\|^2}
    =o(d_{\mathrm{eff}}).
    \label{eq:supp-signal-error-energy-condition}
\end{equation}
Then
\begin{equation}
    \frac{
        \|\bm d-\bm d^\star\|^2
        -\|\widetilde{\bm d}-\bm d^\star\|^2
    }{\|\bm e'\|^2}
    =
    \gamma_e-
    \frac{\|\bm d^\star\|^2}{\|\bm e'\|^2}\gamma_s
    =\Theta_{\mathbb P}(1).
    \label{eq:supp-asymptotic-net-improvement}
\end{equation}
In particular, $\|\widetilde{\bm d}-\bm d^\star\|^2 <\|\bm d-\bm d^\star\|^2$ with probability tending to one.
\end{corollary}

\begin{proof}
By \cref{cor:supp-high-dimensional-separation} and \cref{eq:supp-signal-error-energy-condition},
\begin{equation}
    \frac{\|\bm d^\star\|^2}{\|\bm e'\|^2}\gamma_s
    =o_{\mathbb P}(1),
\end{equation}
whereas $\gamma_e$ is bounded away from zero with probability tending to one. Substitution into \cref{eq:supp-net-correction-improvement} proves both claims.
\end{proof}

Thus, under the stated assumptions, error alignment is structural, whereas signal alignment is asymptotically weak: \methodname{} removes a constant fraction of critic-error energy, while the fraction of ideal-signal energy removed along the same rank-one direction vanishes with effective dimension.

\paragraph{Evidence at scale.}
Although $\bm d^\star$ and $\bm e$ are not observable on MiniMax-H3, the removed component itself can be decoded and seen. The same update target, decoded with and without the projection, loses its speckle and keeps its content (\cref{fig:dmd_vectors}; \cref{fig:supp-proj-endpoints}; \cref{sec:h3-proj-endpoints}). Three further measurements point the same way. The kept component carries more than $80\%$ of the update norm for most of training (\cref{sec:h3-norm-ratio}). Removing any of four other directions instead does not stabilize training (\cref{sec:ablations}, \cref{sec:h3-norm-ratio}). Putting half of the removed component back brings the degradation back (\cref{sec:h3-partial-retention}). The removed direction carries the destabilizing part of the update, and training with the projected update stays stable where \dmd degrades (\cref{fig:degradation}).

\paragraph{Projecting versus subtracting an error estimate.}
Since $\mathbb E[\bm R]=\bm e$, subtracting an estimate of the error is an alternative to deleting a direction. Write $\bm e'=\kappa\bm e$ with $\kappa>0$ and consider the family
\begin{equation}
    \widetilde{\bm d}_{\lambda}
    =
    \bm d-\lambda\kappa\bm R,
    \qquad
    \lambda\geq 0,
    \label{eq:supp-subtraction-family}
\end{equation}
which is the \dmd update at $\lambda=0$ and subtracts the unbiased estimate at $\lambda=1$. Since $\bm d-\lambda\kappa\bm R =\bm d^\star+(1-\lambda)\bm e'+\lambda\kappa\bm\xi$ and $\mathbb E[\bm\xi]=\bm 0$,
\begin{equation}
    \mathbb E\|\widetilde{\bm d}_{\lambda}-\bm d^\star\|^2
    =
    \kappa^2\bigl[(1-\lambda)^2a+\lambda^2\tau\bigr],
    \label{eq:supp-subtraction-risk}
\end{equation}
minimized at $\lambda^\star=a/(a+\tau)$ with value $\kappa^2a\tau/(a+\tau)$.

The unbiased choice $\lambda=1$ has expected squared error $\kappa^2\tau$: it removes the error in mean but injects the full conditional noise. The optimal fixed coefficient $\lambda^\star$ achieves $\kappa^2a\tau/(a+\tau)$ but depends on the unobservable $a$ and $\tau$. The projection needs neither; its realized error is given exactly by \cref{eq:supp-projected-correction-error}. Its coefficient $\langle\bm d,\bm R\rangle/\|\bm R\|^2$ equals $\kappa\lambda^\star(1+o_{\mathbb P}(1))$ under the conditions of \cref{cor:supp-net-correction-improvement}, so it asymptotically estimates the optimal fixed subtraction coefficient from the two available vectors. This convergence in probability does not by itself establish equality of the expected risks.

Three further properties favor the projection. The projection never reverses the \dmd direction: $\langle\bm P_{\bm R}^{\perp}\bm d,\bm d\rangle =\|\bm P_{\bm R}^{\perp}\bm d\|^2\geq 0$ for every realization, so the angle between $\bm P_{\bm R}^{\perp}\bm d$ and $\bm d$ never exceeds $90^\circ$, whereas $\langle\bm d-\lambda\kappa\bm R,\bm d\rangle =\|\bm d\|^2-\lambda\kappa\langle\bm R,\bm d\rangle$ turns negative once the subtracted term is large and aligned with $\bm d$. The projection never lengthens the update, $\|\bm P_{\bm R}^{\perp}\bm d\|\leq\|\bm d\|$, whereas $\bm d-\lambda\kappa\bm R$ can be longer than $\bm d$. The projection is invariant to any nonzero rescaling of $\bm R$, whereas subtraction needs $\kappa$ at every noise level, and an overestimate flips the remaining error. \Cref{sec:h3-partial-retention} evaluates this family at half the projection's own coefficient, and that run degrades as \dmd does.

\paragraph{Biased critic-error estimates with smaller deviation.}
\label{par:supp-biased-critic-error}
In graphics, a biased estimator with smaller variance often replaces an unbiased one to lower the mean squared error~\citep{zwicker2015recent}. A biased critic-error estimate with smaller mean-squared deviation offers a related tradeoff in our analysis. The proofs use $\bm\xi$ only through $\|\bm\xi\|^2$ and $\bm e^{\mathsf T}\bm\xi$. Let the subtracted endpoint deviate from $\bm c^\star$ by $\bm\xi'=\bm m+\bm\zeta$ with a fixed $\bm m$ and a centered $\bm\zeta$, so $\bm R'=\bm e-\bm\xi'$ has bias $-\bm m$, and write $\tau'=\|\bm m\|^2+\operatorname{tr}\operatorname{Cov}(\bm\zeta)$. If $\bm\xi'$ satisfies \cref{eq:supp-removal-assumptions} with $\tau'$ in place of $\tau$, \cref{thm:supp-typical-removal} gives $\gamma_e=a/(a+\tau')+O_{\mathbb P}(d_{\mathrm{eff}}^{-1/2})$: a bias costs the same as an equal amount of variance, and a deviation with smaller mean square removes more error. The fixed $\bm m$ may align with the signal, but it adds at most $\|\bm m\|^2/(a+\tau')$ to $\gamma_s$ up to a constant: the smaller the bias, the less its direction matters.

\subsection{Critic Lag and Allocation Stability}
\label{sec:proof-stability}

Distribution matching leaves the allocation of mass between separated modes weakly constrained, and the online critic sees that allocation only as it was some updates earlier.  This section models the resulting feedback loop and shows that the student--critic endpoint residual observes the delayed error itself. Lag is one component of the critic error $\bm e$ of \cref{eq:critic-endpoint-error}; \cref{eq:conditional-removal-bound} and \cref{sec:proof-high-dimensional-filtering} bound the removal of the total error, and this section models the lag component alone.

We keep the notation of \cref{eq:critic-conditional-statistics}: $(Q,\bm S)$ is the critic-training query--endpoint pair, $\bm c^\star(q)=\mathbb E[\bm S\mid Q=q]$, $\bm\xi=\bm S-\bm c^\star(q)$, and $\bm\Sigma(q)=\operatorname{Cov}(\bm S\mid Q=q)$.  Critic training re-diffuses the student's own endpoints (\cref{eq:dmd-noising}), so the critic-training endpoint is the student endpoint, $\bm S=\bm x_0^s$.  The residual that \methodname{} projects along, $\bm r=\bm x_0^c-\bm x_0^s$ of \cref{eq:perpendicular-projector}, is therefore the conditional endpoint residual $\bm R=\bm e-\bm\xi$ of \cref{sec:analysis-critic-average}, and $\bm P_{\bm r}=\bm P_{\bm R}$.

\paragraph{Allocation slice.}
At a fixed external condition, let the target have two separated mode regions with masses $\pi^\star$ and $1-\pi^\star$ for $\pi^\star\in(0,1)$, and otherwise arbitrary within-mode distributions.  Let $\mathcal M(\bm x_0)\in\{+,-\}$ denote the mode assignment and let $p_+$ and $p_-$ be the student's mode-conditioned endpoint distributions.  We isolate the local slice in which these distributions are fixed while their allocation varies,
\begin{equation}
    p_\theta(\bm x_0)
    =
    \pi(\theta)p_+(\bm x_0)
    +(1-\pi(\theta))p_-(\bm x_0),
    \label{eq:supp-allocation-slice}
\end{equation}
where the student induces the latent basins $A_\pm(\theta)=\{\bm z:\mathcal M(G_\theta(\bm z))=\pm\}$ and hence the allocation $\pi(\theta)=\mathbb P_{\bm z}[\bm z\in A_+(\theta)]$.  Write $u$ for training time, $\pi(u)$ for the allocation of the student at time $u$, and
\begin{equation}
    \Delta(u)=\pi(u)-\pi^\star
    \label{eq:supp-allocation-deviation}
\end{equation}
for its deviation from the matched allocation.

Fix the noise level $t$ and the external condition.  The query is then the noisy point itself, $Q_t=\alpha_t\bm S+\sigma_t\bm\epsilon$ of \cref{eq:dmd-noising}, so conditioning on $Q=q$ and on $Q_t=q$ coincide and we use whichever is clearer.  With $M=\mathcal M(\bm S)$, define the mode posterior, the mode-conditioned posterior means, and their contrast,
\begin{equation}
\begin{gathered}
    w_\pi(q)=\mathbb P(M=+\mid Q_t=q),
    \qquad
    \bm m_\pm(q)=\mathbb E[\bm S\mid Q_t=q,M=\pm],\\
    \Delta\bm m(q)=\bm m_+(q)-\bm m_-(q).
\end{gathered}
    \label{eq:supp-mode-posterior}
\end{equation}
Let $\bm c_\pi^\star(q)=\mathbb E[\bm S\mid Q_t=q]$ denote the population-optimal critic endpoint when the student allocation is $\pi$, so that $\bm c^\star(q)=\bm c^\star_{\pi(u)}(q)$ at training time $u$.

\begin{lemma}[Allocation moves the critic target along the mode contrast]
\label{lem:supp-allocation-barycenter}
For every allocation $\pi$, $\bm c_\pi^\star(q)=\bm m_-(q)+w_\pi(q)\Delta\bm m(q)$, and for any two allocations $\pi$ and $\pi'$,
\begin{equation}
    \bm c_{\pi'}^\star(q)-\bm c_\pi^\star(q)
    =
    \bigl(w_{\pi'}(q)-w_\pi(q)\bigr)\Delta\bm m(q).
    \label{eq:supp-allocation-shift}
\end{equation}
\end{lemma}

\begin{proof}
Conditioning on $M$ gives $\bm c_\pi^\star(q)=w_\pi(q)\bm m_+(q)+(1-w_\pi(q))\bm m_-(q)$, which is the first identity.  The mode-conditioned distributions are held fixed in \cref{eq:supp-allocation-slice}, so $\bm m_\pm(q)$ does not depend on $\pi$ and only $w_\pi(q)$ moves with the allocation; subtracting the first identity at $\pi'$ and at $\pi$ proves \cref{eq:supp-allocation-shift}.
\end{proof}

\paragraph{Delayed allocation feedback.}
Near the matched allocation we treat $\Delta$ as a scalar channel.  Against a critic that is population-optimal for the current student, we model the population update as restoring the match at a rate $\lambda>0$, $\dot\Delta(u)=-\lambda\Delta(u)$.  The online critic is instead trained against the student it saw earlier: writing $\tau>0$ for that delay and keeping the same channel, the critic responds to $\pi(u-\tau)$ rather than $\pi(u)$, and the allocation obeys the delay equation
\begin{equation}
    \dot\Delta(u)=-\lambda\Delta(u-\tau).
    \label{eq:supp-delayed-allocation}
\end{equation}

\begin{lemma}[Critical delay]
\label{lem:supp-critical-delay}
Let $\lambda,\tau>0$.  Every root of the characteristic equation
\begin{equation}
    z+\lambda e^{-z\tau}=0
    \label{eq:supp-characteristic}
\end{equation}
has negative real part if and only if $\lambda\tau<\pi/2$; at $\lambda\tau=\pi/2$ the pair $z=\pm i\lambda$ lies on the imaginary axis.  The zero solution of \cref{eq:supp-delayed-allocation} is therefore asymptotically stable for $\lambda\tau<\pi/2$ and unstable for $\lambda\tau>\pi/2$.
\end{lemma}

\begin{proof}
Substituting $\Delta(u)=e^{zu}$ in \cref{eq:supp-delayed-allocation} gives $ze^{zu}=-\lambda e^{z(u-\tau)}$, which is \cref{eq:supp-characteristic}. Let $z=x+iy$ be a root with $x\geq0$; roots occur in conjugate pairs, so we may take $y\geq0$.  From $z=-\lambda e^{-z\tau}$,
\begin{equation}
    |z|=\lambda e^{-x\tau}\leq\lambda,
    \qquad
    x=-\lambda e^{-x\tau}\cos(y\tau),
    \qquad
    y=\lambda e^{-x\tau}\sin(y\tau).
    \label{eq:supp-root-components}
\end{equation}
Since $\lambda e^{-x\tau}>0$, the middle identity and $x\geq0$ force $\cos(y\tau)\leq0$, and the smallest $y\tau\geq0$ with $\cos(y\tau)\leq0$ is $\pi/2$, so $y\tau\geq\pi/2$.  The first identity gives $y\leq|z|\leq\lambda$ and hence $y\tau\leq\lambda\tau$.  A root with nonnegative real part therefore requires $\lambda\tau\geq\pi/2$, so $\lambda\tau<\pi/2$ leaves every root in the open left half plane.  At $\lambda\tau=\pi/2$, substituting $z=i\lambda$ in \cref{eq:supp-characteristic} gives $i\lambda+\lambda e^{-i\pi/2}=i\lambda-i\lambda=0$, so $\pm i\lambda$ are roots; for $\lambda\tau>\pi/2$ this pair has crossed into the right half plane~\citep{hayes1950roots}.
\end{proof}

\paragraph{The residual observes the lag.}
Model a lagging critic as one whose only error is staleness: at training time $u$ its endpoint is population-optimal for the student of time $u-\tau$, $\bm x_0^c=\bm c^\star_{\pi(u-\tau)}(q)$.  Its endpoint-space error \cref{eq:critic-endpoint-error} is then the lag error, and \cref{lem:supp-allocation-barycenter} places it on the mode contrast,
\begin{equation}
    \bm e
    =
    \bm e_{\mathrm{lag}}^{x_0}(q)
    :=
    \bm c^\star_{\pi(u-\tau)}(q)-\bm c^\star_{\pi(u)}(q)
    =
    \bigl(w_{\pi(u-\tau)}(q)-w_{\pi(u)}(q)\bigr)\Delta\bm m(q).
    \label{eq:supp-lag-endpoint-error}
\end{equation}
At a fixed noise level, the mixture-level Tweedie identity $\bm s_\pi(q)=(\alpha_t\bm c_\pi^\star(q)-q)/\sigma_t^2$, where $\bm s_\pi=\nabla_q\log p_{\pi,t}$ is the score of the density $p_{\pi,t}$ of $Q_t$ at allocation $\pi$, converts this into the score-space lag error
\begin{equation}
    \bm e_{\mathrm{lag}}^{s}(q)
    =
    \bm s_{\pi(u-\tau)}(q)-\bm s_{\pi(u)}(q)
    =
    \frac{\alpha_t}{\sigma_t^2}\bm e_{\mathrm{lag}}^{x_0}(q),
    \label{eq:supp-lag-score-error}
\end{equation}
which is the error this critic contributes to the DMD update \cref{eq:dmd-difference}.  In the notation of \cref{sec:proof-high-dimensional-filtering}, the pure-lag model is the case $\bm e'=\bm e_{\mathrm{lag}}^{s}(q)$ of $\bm d=\bm d^\star+\bm e'$. Because the critic-training endpoint is the student's own endpoint, the residual carries this error explicitly,
\begin{equation}
    \bm r=\bm x_0^c-\bm x_0^s=\bm e_{\mathrm{lag}}^{x_0}(q)-\bm\xi,
    \qquad
    \mathbb E[\bm r\mid Q=q]=\bm e_{\mathrm{lag}}^{x_0}(q),
    \label{eq:supp-residual-lag-mean}
\end{equation}
where the second identity uses $\mathbb E[\bm\xi\mid Q=q]=\bm 0$ from \cref{lem:supp-l2-conditional-mean}.  The conditional mean of the observable residual is the stale-score error itself, up to the scalar $\sigma_t^2/\alpha_t$.

\begin{corollary}[Conditional attenuation of the lag error]
\label{cor:supp-lag-attenuation}
Let $\bm e_{\mathrm{lag}}^{x_0}(q)\neq\bm 0$.  Then
\begin{equation}
    \mathbb E\!\left[
        \frac{\|\bm P_{\bm r}\bm e_{\mathrm{lag}}^{s}\|^2}
             {\|\bm e_{\mathrm{lag}}^{s}\|^2}
        \,\middle|\,Q=q
    \right]
    \geq
    \frac{\|\bm e_{\mathrm{lag}}^{x_0}(q)\|^2}
         {\|\bm e_{\mathrm{lag}}^{x_0}(q)\|^2+\operatorname{tr}\bm\Sigma(q)},
    \label{eq:supp-lag-removal-bound}
\end{equation}
and the part of the lag error that survives the projection \cref{eq:method-pdmd-projection} obeys
\begin{equation}
    \mathbb E\bigl[
        \|\bm P_{\bm r}^{\perp}\bm e_{\mathrm{lag}}^{s}\|^2
        \mid Q=q
    \bigr]
    \leq
    \frac{\operatorname{tr}\bm\Sigma(q)}
         {\|\bm e_{\mathrm{lag}}^{x_0}(q)\|^2+\operatorname{tr}\bm\Sigma(q)}
    \,\|\bm e_{\mathrm{lag}}^{s}(q)\|^2
    <
    \|\bm e_{\mathrm{lag}}^{s}(q)\|^2.
    \label{eq:supp-lag-attenuation}
\end{equation}
\end{corollary}

\begin{proof}
By \cref{eq:supp-lag-score-error} the two lag errors differ by the positive scalar $\alpha_t/\sigma_t^2$, which cancels in the ratio, so the left-hand side of \cref{eq:supp-lag-removal-bound} is $\mathbb E[\gamma_e(\bm r)\mid Q=q]$ of \cref{eq:conditional-removal-ratio} with $\bm e=\bm e_{\mathrm{lag}}^{x_0}(q)$.  By \cref{eq:supp-residual-lag-mean}, $\bm r=\bm e-\bm\xi$ with $\mathbb E[\bm\xi\mid Q=q]=\bm 0$ and $\mathbb E[\|\bm\xi\|^2\mid Q=q]=\operatorname{tr}\bm\Sigma(q)$ (\cref{lem:supp-l2-conditional-mean}), which are exactly the conditional moments \cref{eq:critic-residual-moments} behind \cref{eq:conditional-removal-bound}; that bound is \cref{eq:supp-lag-removal-bound}.  Since $\bm P_{\bm r}$ and $\bm P_{\bm r}^{\perp}$ are complementary orthogonal projectors, $\|\bm P_{\bm r}^{\perp}\bm e_{\mathrm{lag}}^{s}\|^2 =\|\bm e_{\mathrm{lag}}^{s}\|^2-\|\bm P_{\bm r}\bm e_{\mathrm{lag}}^{s}\|^2$, and taking conditional expectations turns \cref{eq:supp-lag-removal-bound} into \cref{eq:supp-lag-attenuation}; the last inequality is strict because $\|\bm e_{\mathrm{lag}}^{x_0}(q)\|>0$.
\end{proof}

The delayed feedback of \cref{eq:supp-delayed-allocation} reaches the student through $\bm e_{\mathrm{lag}}^{s}$. \dmd applies it in full. By \cref{cor:supp-lag-attenuation}, \methodname{} passes on, in conditional mean square, at most the fraction
\begin{equation}
    \rho(q)
    =
    \frac{\operatorname{tr}\bm\Sigma(q)}
         {\|\bm e_{\mathrm{lag}}^{x_0}(q)\|^2+\operatorname{tr}\bm\Sigma(q)}
    <1.
    \label{eq:supp-lag-surviving-fraction}
\end{equation}
Treating this fraction as a constant gain on the scalar channel gives a one-parameter model of the projected feedback, $\dot\Delta(u)=-\rho\lambda\Delta(u-\tau)$, in which the stability condition of \cref{lem:supp-critical-delay} becomes $\rho\lambda\tau<\pi/2$. The model is a heuristic: $\rho$ bounds an energy fraction in expectation, the constant-gain step stands in for a quantity that varies with the query and with the critic, and the projection also acts on the ideal update, so $\rho$ scales the lag error rather than the whole channel. The two-mode results in \cref{tab:supp-twomode} are consistent with this qualitative interpretation. The gain $\rho$ tracks the state of the critic. When the critic is stale, the lag error dominates the conditional spread of the student's endpoints, $\rho\to0$, and the delayed feedback is suppressed. When the critic has caught up, $\bm e_{\mathrm{lag}}=\bm 0$, the residual is pure conditional noise, and the update passes through up to the signal-removal term of \cref{eq:supp-projected-correction-error}. The projection damps the allocation feedback when the critic is stale and releases it when the critic is current. A fixed reduction of the step, such as the random direction in \cref{tab:supp-twomode}, has no such dependence.

\section{2D Experiments}
\label{sec:supp-2d-experiments}

This section tests two theoretical predictions in a minimal, fully observable setting. Experiment~1 measures whether the projection preferentially removes critic error while retaining the distribution-matching signal (\cref{eq:conditional-removal-bound} and \cref{cor:supp-high-dimensional-separation}). Experiment~2 measures mode allocation under critic lag and tests the attenuation of \cref{cor:supp-lag-attenuation}.

\paragraph{Shared protocol.}
Both experiments use a variance-exploding (VE) diffusion toy in $\mathbb R^2$ in which the critic error and ideal update can be estimated from student samples. The teacher score is \emph{analytic} (the exact posterior-mean denoiser of the Gaussian-mixture target), so the online critic is the only learned score in the system. Experiment~1 measures its deviation from the population-optimal student score. The student is a one-step generator $\bm x_0^s=G_\theta(\bm z)$, $\bm z\sim\sigma_{\max}\mathcal N(\bm 0,\bm I)$, and the critic is an EDM-parameterized denoiser trained by denoising score matching on the student's own re-diffused samples (\cref{eq:dmd-noising} with $\alpha_t=1$); both are 3-layer MLPs of width 128.  Training uses Adam ($\beta_1{=}0$, $\beta_2{=}0.999$), a learning rate of $2\times10^{-3}$ for both networks unless stated otherwise, a batch size of 1{,}024, 4{,}000 iterations in Experiment~1 and 6{,}000 iterations in Experiment~2, and noise levels $\sigma\sim\mathrm{LogUniform}[0.02,5]$ with the standard DMD weighting.  \methodname{} differs from the DMD variant \emph{only} by the per-sample projection of \cref{eq:method-pdmd-projection}; seeds, data, and schedules are shared. We report the \emph{energy distance} between 2{,}048 generated and target samples (a rotation-invariant two-sample distance, zero iff the distributions match); \emph{modes covered}, the number of mixture components whose share of on-mode samples exceeds $25\%$ of its target share; the \emph{on-mode fraction}, the share of samples within $3\sigma_{\mathrm{mode}}$ of a component center (a precision proxy); and the \emph{removed fraction}, the relative reduction in update norm, $1-\|\bm d_{\mathrm{kept}}\|/\|\bm d\|$.

\subsection{Experiment 1: Error Filtering on a Multi-Mode Target}
\label{sec:supp-2d-exp1}

\paragraph{Setup.}
The target distribution is a ring of eight Gaussians (radius 2, standard deviation $\sigma_{\mathrm{mode}}=0.12$). We train five variants. Four remove one direction from the DMD update $\bm d$, and the unprojected baseline provides a reference: (a) \textbf{DMD}: no direction is removed; (b) \textbf{random direction $\perp$}: remove a fresh random direction for each sample. In 2D, this projection removes $1-2/\pi\approx36\%$ of the update norm on average, close to the share \methodname{} removes. Variant (b) therefore tests whether reducing the update norm alone explains the gain; (c) \textbf{critic score $\perp$}: remove the direction of $\bm s_{\mathrm{critic}}(\bm x_t)$.  The critic-score direction is the closest alternative to the \methodname{} direction; the two differ only by the query displacement $\sigma_t\bm\epsilon$; (d) \textbf{student--teacher endpoint residual $\perp$}: remove $\bm r_{\mathrm{teacher}}=\bm x_0^{\mathrm{teacher}}-\bm x_0^s$.  The teacher endpoint $\bm x_0^{\mathrm{teacher}}$ comes from $\bm s_{\mathrm{teacher}}$ through the standard time-dependent rescaling and costs no extra forward pass.  The teacher is exact, so $\bm r_{\mathrm{teacher}}$ contains no critic error; (e) \textbf{student--critic endpoint residual $\perp$}: the \methodname{} update, which removes $\bm r=\bm x_0^c-\bm x_0^s$. The error-removal analysis predicts an ordering based on the information carried by each direction. The student--critic endpoint residual has the structural form $\bm R=\bm e-\bm\xi$ targeted by the removal bound.  The critic-score direction also contains critic error, but mixes it with the full marginal-score geometry rather than isolating $\bm R$; variants (b) and (d) do not contain critic error, and variant (c) can also remove useful signal. All variants share the seed, data, and schedule.  The figures show one run per variant with 200 recorded snapshots.

\paragraph{Distribution metrics.}
\Cref{fig:2d-curves} in the main text shows the training curves. \Cref{fig:supp-2d-grid} shows the generated samples of every variant at iterations $0,1{,}000,\ldots,4{,}000$. We average each metric over the last quarter of training (iterations $3{,}000$ to $4{,}000$). \methodname{} reaches energy distance $0.0080$ and on-mode fraction $0.903$.  DMD reaches $0.0273$ and $0.838$.  The random direction reaches $0.0180$ and $0.795$.  The critic-score projection reaches $0.0332$ and $0.423$.  The teacher-residual projection reaches $2.04$ and $0.413$; the large error spike near iteration $3{,}000$, visible in both figures, dominates the average for this variant. Averaged over the shown $4{,}000$ iterations, the random direction removes $36\%$ of the update norm, the critic-score projection $39\%$, the teacher-residual projection $43\%$, and \methodname{} $42\%$; the four deletions shrink the step by a similar amount. The ordering is consistent with the predicted error geometry. Among the tested directions, removing the student--critic endpoint residual is the only deletion that improves both metrics over DMD.  Removing a random direction or the error-free teacher residual does not reproduce the gain.  Since a rank-one deletion in $d=2$ removes one of only two available directions, this experiment is a stringent signal-retention test; \methodname{} still improves both metrics over DMD.

\paragraph{Direct measurement of the removal ratios.}
The metrics above read the projection through its consequences. To look at the mechanism itself, \cref{fig:supp-2d-gamma} measures $\gamma_e$ and $\gamma_s$ of \cref{sec:proof-high-dimensional-filtering} directly along the \methodname{} run. In 2D we estimate the population-optimal critic endpoint by a kernel average over a bank of $2^{17}$ student samples, $\bm c^\star(q)=\sum_i\bm x_i\,\mathcal N(q;\bm x_i,\sigma^2\bm I) /\sum_i\mathcal N(q;\bm x_i,\sigma^2\bm I)$; the same weights give $\operatorname{tr}\bm\Sigma(q)$. With the analytic teacher, this gives a Monte Carlo estimate of $\bm d^\star=(\bm c^\star(q)-\bm x_0^{\mathrm{teacher}}(q))/\sigma^2$; the critic error is $\bm e=\bm x_0^c(q)-\bm c^\star(q)$ and the residual is $\bm r=\bm x_0^c(q)-\bm x_0^s$ for the probe's own endpoint. Every 100 iterations, 512 probes at each of 12 log-spaced noise levels in $[0.02,5]$ give the per-sample $\gamma_e(\bm r)$, $\gamma_s(\bm r)$, and the lower bound $\|\bm e\|^2/(\|\bm e\|^2+\operatorname{tr}\bm\Sigma(q))$ of \cref{eq:conditional-removal-bound}; the diagnostics use a separate random stream, so the measured run reproduces the stored \methodname{} run bit for bit. In $d=2$ a uniformly random rank-one deletion removes half of the energy of any fixed vector on average, and $d_{\mathrm{eff}}\leq2$, so the high-dimensional limit $\gamma_s\to0$ is not available in the plane. What the plane can test is preferential error removal: $\gamma_e$ exceeds $\gamma_s$ on average in the evaluated noise band (\cref{tab:supp-2d-gamma}). \Cref{fig:supp-2d-gamma} shows exactly that split: $\gamma_e$ peaks at $0.75$ where the critic carries error, both ratios return to $1/2$ once the critic is accurate, and the lower bound of \cref{eq:conditional-removal-bound} holds at every snapshot and noise level. \Cref{tab:supp-2d-gamma} repeats the measurement for the other three deleted directions on the same probes: the random direction sits at $1/2$ for both ratios, the student--teacher endpoint residual removes more signal than error, and only the student--critic endpoint residual removes clearly more error than signal, $0.18\pm0.01$ more error energy than the random direction and above the theoretical bound in both blocks.

\Cref{tab:supp-2d-net} reports the net effect of each deletion on the same probes, in the two forms of \cref{eq:supp-net-correction-improvement,cor:supp-net-correction-improvement}. With $\widetilde{\bm d}=\bm P_{\bm b}^{\perp}\bm d$ the update after deleting the direction $\bm b$, the error ratio $\nu=\mathbb E\|\widetilde{\bm d}-\bm d^\star\|^2/ \mathbb E\|\bm d-\bm d^\star\|^2$ is the distance of the projected update to the ideal update relative to that of the \dmd update, and the improved fraction $\varphi$ is the share of probes on which the projected update is closer to $\bm d^\star$, the empirical form of the probability in \cref{cor:supp-net-correction-improvement}.

In this 2D setting, deleting the student--critic endpoint residual reduces the average error in the evaluated noise band: in the band $\sigma\in[0.15,1.1]$ our projection makes the error fall by $28\%$. Deleting a random direction lowers it by $16\%$, the critic score by $20\%$, and the student--teacher endpoint residual by $8\%$; over all twelve noise levels the student--teacher endpoint residual raises the error by $8\%$. Paired over the same snapshot and noise level, the student--critic endpoint residual has a lower error ratio than the random direction in $82\%$ of the cells in the band, than the critic score in $68\%$, and than the teacher endpoint residual in $99\%$. The same holds on a run trained without the projection: probes along the \dmd run, where each deletion is computed at every snapshot but never used to update the student, give the same ordering, $\nu=0.90$ for the student--critic endpoint residual against $0.98$, $0.94$ and $1.21$.

\begin{table}[t]
\centering
\caption{\textbf{Removal ratios of the four deleted directions on the ring of eight.} Mean $\gamma_e$ (critic error removed) and $\gamma_s$ (ideal signal removed) over the same probes, noise draws, and snapshots as \cref{fig:supp-2d-gamma}, along the \methodname{} run. Left: the noise band $\sigma\in[0.15,1.1]$ in which the critic carries error.  Right: all 12 noise levels. In the plane a uniformly random rank-one deletion removes $1/2$ of any fixed vector's energy on average, so $1/2$ is the reference for both ratios. The last column of each block is the mean lower bound $\|\bm e\|^2/(\|\bm e\|^2+\operatorname{tr}\bm\Sigma(q))$ of \cref{eq:conditional-removal-bound} on the same probes; it is a bound on $\gamma_e$ for the student--critic endpoint residual and is listed in that row.}
\label{tab:supp-2d-gamma}
\small
\setlength{\tabcolsep}{4pt}
\begin{tabular}{lcccc@{\hspace{1.2em}}cccc}
\toprule
& \multicolumn{4}{c}{$\sigma\in[0.15,1.1]$} & \multicolumn{4}{c}{all levels} \\
\cmidrule(lr){2-5}\cmidrule(lr){6-9}
Deleted direction & $\gamma_e$ & $\gamma_s$ & $\gamma_e-\gamma_s$ & $\gamma_e$ bound
& $\gamma_e$ & $\gamma_s$ & $\gamma_e-\gamma_s$ & $\gamma_e$ bound \\
\midrule
Random direction & $0.50$ & $0.50$ & $0.00$ & -- & $0.50$ & $0.50$ & $0.00$ & -- \\
Critic score & $0.52$ & $0.49$ & $0.03$ & -- & $0.52$ & $0.49$ & $0.03$ & -- \\
Student--teacher residual & $0.56$ & $0.61$ & $-0.05$ & -- & $0.55$ & $0.59$ & $-0.04$ & -- \\
\textbf{Student--critic residual (\methodname{})} & $\bm{0.68}$ & $0.57$ & $\bm{0.11}$ & $0.25$
& $\bm{0.61}$ & $0.56$ & $\bm{0.04}$ & $0.21$ \\
\bottomrule
\end{tabular}
\end{table}

\begin{table}[t]
\centering
\caption{\textbf{Net effect of the four deletions on the ring of eight.} Same probes, noise draws and snapshots as \cref{tab:supp-2d-gamma}, along the \methodname{} run. $\nu$ ($\downarrow$) is the error ratio $\mathbb E\|\widetilde{\bm d}-\bm d^\star\|^2/ \mathbb E\|\bm d-\bm d^\star\|^2$: below $1$ the deletion moves the update closer to the ideal update, above $1$ farther away. $\varphi$ ($\uparrow$) is the improved fraction, the share of probes on which the projected update is closer to $\bm d^\star$ than the \dmd update. Left: the noise band $\sigma\in[0.15,1.1]$ in which the critic carries error.  Right: all 12 noise levels. Best value per column in bold.}
\label{tab:supp-2d-net}
\small
\setlength{\tabcolsep}{6pt}
\begin{tabular}{lcc@{\hspace{1.6em}}cc}
\toprule
& \multicolumn{2}{c}{$\sigma\in[0.15,1.1]$} & \multicolumn{2}{c}{all levels} \\
\cmidrule(lr){2-3}\cmidrule(lr){4-5}
Deleted direction & $\nu$ $\downarrow$ & $\varphi$ $\uparrow$ & $\nu$ $\downarrow$ & $\varphi$ $\uparrow$ \\
\midrule
Random direction & $0.84$ & $0.68$ & $0.95$ & $0.63$ \\
Critic score & $0.81$ & $0.71$ & $\bm{0.88}$ & $0.64$ \\
Student--teacher residual & $0.92$ & $0.67$ & $1.08$ & $0.62$ \\
\textbf{Student--critic residual (\methodname{})} & $\bm{0.72}$ & $\bm{0.74}$ & $0.89$ & $\bm{0.66}$ \\
\bottomrule
\end{tabular}
\end{table}

\begin{figure}[t]
\centering
\includegraphics[width=0.6\linewidth]{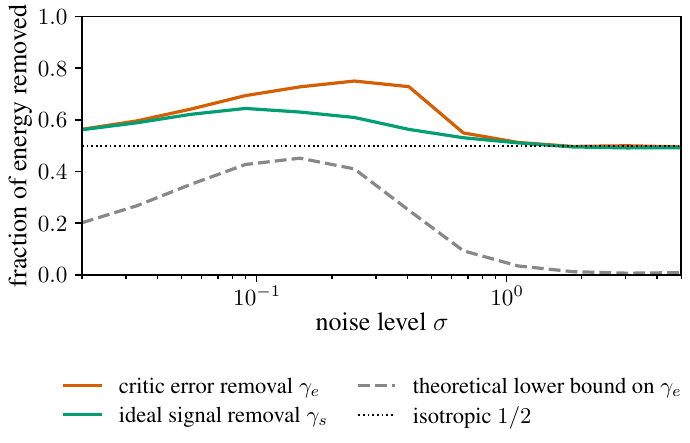}
\caption{\textbf{Direct measurement of the removal ratios on the ring of eight.} Along the \methodname{} run of \cref{sec:supp-2d-exp1}: the fraction of critic-error energy $\gamma_e$ and of ideal-signal energy $\gamma_s$ that the projection removes, against the noise level. The grey dashed line is the theoretical lower bound $\|\bm e\|^2/(\|\bm e\|^2+\operatorname{tr}\bm\Sigma(q))$ of \cref{eq:conditional-removal-bound} on $\gamma_e$. Each curve is a mean over the snapshots from iteration 100 to 4{,}000. The dotted line marks the mean energy fraction $1/2$ removed by a uniformly random rank-one deletion from a fixed vector in the plane.}
\label{fig:supp-2d-gamma}
\end{figure}

\begin{figure}[p]
\centering
\includegraphics[width=0.92\linewidth]{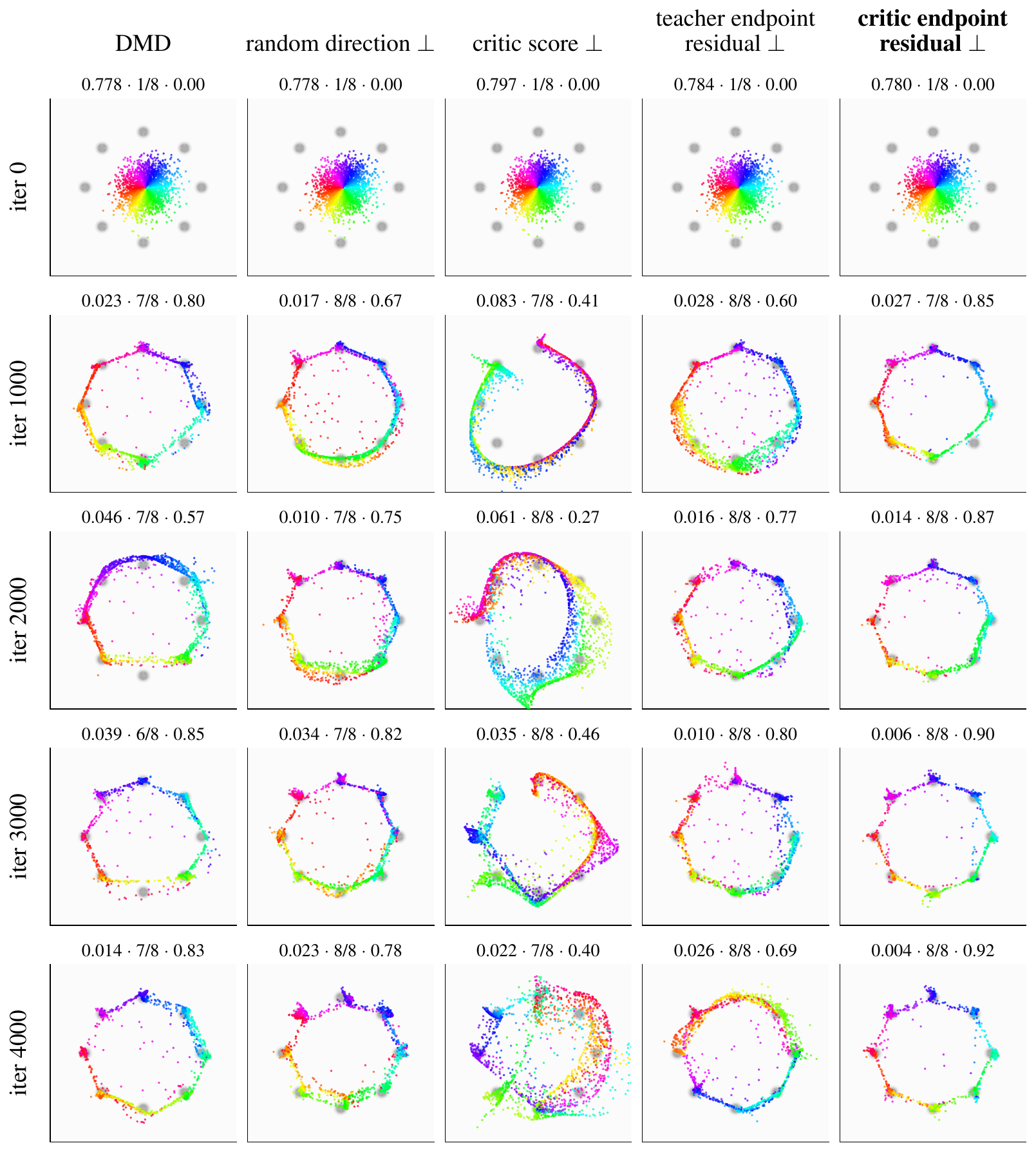}
\caption{\textbf{Generated samples of the five 2D variants over training.} Columns are the five variants of \cref{sec:supp-2d-exp1}.  Rows are iterations $0,1{,}000,\ldots,4{,}000$ of the same runs as \cref{fig:2d-curves}. Point color encodes the angle of the latent $\bm z$, visualizing how latent directions map to target modes. Each panel lists the energy distance ($\downarrow$, lower is better), the number of covered modes ($\uparrow$), and the on-mode fraction ($\uparrow$, higher is better) at the shown iteration.  The student--critic-endpoint-residual column (\methodname{}) covers all eight modes from iteration $2{,}000$ onward and stays tight. The critic-score column covers the modes but does not concentrate.  The teacher-residual column shows the instability near iteration $3{,}000$.}
\label{fig:supp-2d-grid}
\end{figure}

\subsection{Experiment 2: Two-Mode Allocation Under Critic Lag}
\label{sec:supp-2d-exp2}

\paragraph{Setup.}
The target is the symmetric two-mode mixture $\tfrac12\mathcal N(+\bm\mu,\sigma_{\mathrm{mode}}^2\bm I) +\tfrac12\mathcal N(-\bm\mu,\sigma_{\mathrm{mode}}^2\bm I)$ with $\bm\mu=(2,0)$ and $\sigma_{\mathrm{mode}}=0.2$, a Gaussian instance of the allocation slice in \cref{eq:supp-allocation-slice}. The observable of interest is the antisymmetric mode imbalance $a_k=\mathrm{share}(\text{samples in mode }+)-\tfrac12$; the symmetric solution has $a=0$ and total collapse has $|a|=\tfrac12$. A critic that lags behind the student motivates the delayed-feedback model of \cref{eq:supp-delayed-allocation}, and the observable residual has conditional mean equal to the resulting tracking error (\cref{eq:supp-residual-lag-mean}).  We stress this interaction by running the critic at a learning rate $20\times$ smaller than the student's and repeat the entire comparison at two step sizes to test whether the effect is specific to one setting: student $2\times10^{-3}$ with critic $1\times10^{-4}$, and student $1\times10^{-3}$ with critic $5\times10^{-5}$. Three variants use the same 20 consecutive seeds ($0$--$19$) at each step size: plain DMD, the random-direction control~(b) of \cref{sec:supp-2d-exp1}, and \methodname{}; nothing else changes.  This slow-tracking regime mirrors the practical motivation for DMD2, which adds more critic updates (TTUR) to improve tracking of the moving student~\citep{yin2024improved}.

\paragraph{Results.}
    In the lag-stressed regime, \methodname{} yields the predicted reduction in allocation collapse (\cref{tab:supp-twomode}): plain DMD collapses to a single mode in $14/20$ seeds (mean final $|a|=0.355$), while \methodname{} collapses in $3/20$ (mean $0.138$) at equal precision (mean on-mode fraction $0.973$ vs.\ $0.970$) and reaches a mean energy distance of $0.386$ against DMD's $1.274$. The random-direction control remains close to DMD, collapsing in $16/20$ seeds, with a mean final $|a|=0.396$ and mean energy distance $1.449$, while removing a mean $0.363$ of the update norm against \methodname{}'s $0.514$: the reduction follows the direction that is removed rather than the amount. The two remaining directions of \cref{sec:supp-2d-exp1} also reduce collapse in this setting: the critic-score projection yields $0/20$ and the student--teacher-endpoint-residual projection yields $5/20$, but with lower on-mode fractions ($0.730$ and $0.871$ versus $0.973$); Experiment~1 rules out both alternatives on accuracy grounds. The smaller step size reproduces the comparison: DMD collapses in $19/20$ seeds and the random direction in $19/20$, while \methodname{} collapses in $12/20$. The trajectories are more informative than the counts (\cref{fig:supp-twomode-grid}).  DMD and \methodname{} both pass through a strongly imbalanced transient early in training (the one-step student initially covers a single mode, $|a|\approx0.5$).  Plain DMD then \emph{locks}: every collapsed seed remains at $|a|=0.500$ for the rest of training.  \methodname{} \emph{recovers}: a typical seed follows $|a|\colon0.50\to0.40\to0.15\to0.03$ and then fluctuates near balance. \Cref{cor:supp-lag-attenuation} provides a model for the first part of the transition: the projection attenuates the delayed component that drives the imbalance. Empirically, the remaining distribution-matching field transports mass back toward balance. Seventeen of the twenty \methodname{} seeds finish with both modes covered.

\begin{table*}[t]
\centering
\caption{\textbf{Two-mode collapse counts over 20 consecutive seeds at two step sizes.} Collapse is $\text{modes covered}<2$ at the end of training; $|a|$ is the final imbalance; $|a|$, on-mode fraction, and energy distance are means $\pm$ standard deviations over seeds. The critic learning rate is $20\times$ smaller than the student's in both blocks. The random direction is control~(b) from \cref{sec:supp-2d-exp1}: the same rank-one deletion along an unstructured direction.}
\label{tab:supp-twomode}
\footnotesize
\setlength{\tabcolsep}{4pt}
\begin{tabular*}{\textwidth}{@{\extracolsep{\fill}}llcccc@{}}
\toprule
Student lr & Variant & Collapsed $\downarrow$ &
Mean $|a|$ $\downarrow$ &
On-mode fraction & Mean energy $\downarrow$ \\
\midrule
$2\times10^{-3}$ & DMD & $14/20$ & $0.355\pm0.219$ & $0.970\pm0.045$ & $1.274\pm0.856$ \\
& Random direction $\perp$ & $16/20$ & $0.396\pm0.197$ & $0.975\pm0.065$ & $1.449\pm0.756$ \\
& \textbf{\methodname{} (ours)} & $\bm{3/20}$ & $\bm{0.138\pm0.162}$ & $0.973\pm0.039$ &
$\bm{0.386\pm0.689}$ \\
\midrule
$1\times10^{-3}$ & DMD & $19/20$ & $0.477\pm0.104$ & $0.913\pm0.231$ & $1.899\pm0.618$ \\
& Random direction $\perp$ & $19/20$ & $0.476\pm0.108$ & $0.991\pm0.018$ & $1.791\pm0.440$ \\
& \textbf{\methodname{} (ours)} & $\bm{12/20}$ & $\bm{0.331\pm0.215}$ & $0.925\pm0.224$ & $\bm{1.317\pm1.046}$ \\
\bottomrule
\end{tabular*}
\end{table*}

\begin{figure}[t]
\centering
\includegraphics[width=0.95\linewidth]{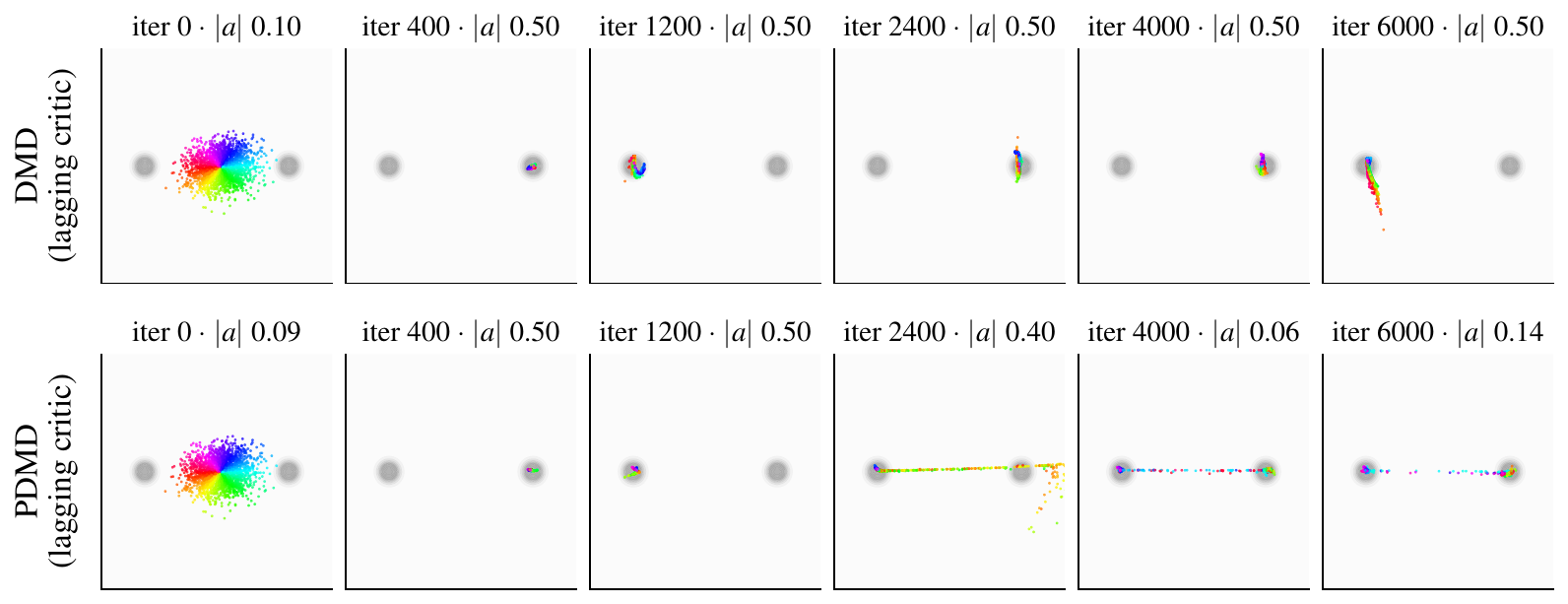}
\caption{\textbf{Two-mode collapse process under a lagging critic (seed 0).} Rows are plain DMD and \methodname{}; columns are training iterations. Each panel lists the iteration and the imbalance $|a|$; $|a|=0$ is the balanced solution and $|a|=0.5$ is full collapse.  Both variants start from a collapsed transient at iteration 400. Plain DMD stays locked on one mode for all 6{,}000 iterations. \methodname{} transports mass back between the modes (iteration 2{,}400 shows samples in transit) and settles near balance.  The main text shows the collapsed DMD endpoint in \cref{fig:2d-collapse}.}
\label{fig:supp-twomode-grid}
\end{figure}

\paragraph{Connection to critic tracking in practice.}
We realize delayed tracking by lowering the critic learning rate.  DMD2 addresses the same critic-tracking challenge in image distillation and increases the critic update frequency so that it follows the moving student more closely~\citep{yin2024improved}. Our two-mode experiment makes the associated allocation response directly observable, while the video results show its large-scale signature as degradation that grows with training (\cref{fig:degradation}).

\section{Experimental Details}
\label{sec:experimental-details}

This section describes the training and evaluation settings for the paper's three experimental domains and their baselines.

\begin{table}[t]
    \centering
    \caption{\textbf{Training configurations of the video experiments.} Our DMD, DMD2, and \methodname{} runs in \cref{tab:wan13b} share the Wan2.1 column and the \dmd$^{\dagger}$, \dmdtwodag and \methodname{} rows of \cref{tab:h3} share the MiniMax-H3 column, except where a baseline paragraph says otherwise.}
    \label{tab:supp-train-config}
    \footnotesize
    \setlength{\tabcolsep}{5pt}
    \begin{tabular}{@{}lll@{}}
        \toprule
        & Wan2.1-T2V-1.3B & MiniMax-H3-33B \\
        \midrule
        Trainable parameters & all & LoRA, rank and scaling 128 \\
        Training data & 42K captions, no real video & 248K prompts, no real video \\
        DMD2 discriminator positives & H3 teacher rollouts & H3 teacher rollouts \\
        Clip & 81 frames, $480{\times}832$ & 5\,s, 544p, mixed aspect ratios \\
        GPUs / global batch & 64 H100 / 64 & 16 H100 / 16 \\
        Optimizer & AdamW $(0, 0.999)$, wd 0.01 & AdamW $(0, 0.9)$, no wd \\
        Student / critic learning rate & $4{\times}10^{-7}$ / $8{\times}10^{-8}$ & $5{\times}10^{-5}$ / $1{\times}10^{-5}$ \\
        Critic updates per student update & 1 & 5 \\
        Teacher guidance in distillation & 5 & none \\
        Student sampler (4 NFE) & $\sigma_{\max}=320$, $t=0.934, 0.76, 0.11$ & released scheduler, shift 12 / 3 \\
        \bottomrule
    \end{tabular}
\end{table}

\begin{table}[t]
    \centering
    \caption{\textbf{Two-stage recipes of the \rcm$^{\dagger}$ and \anyflow$^{\dagger}$ ports on MiniMax-H3.} Both follow their reference: a trajectory stage first, then a stage that adds \dmd on top of it. Every other setting follows \cref{tab:supp-train-config}. Iterations count every update, as in \cref{sec:h3-eval-protocol}, so the student update count of a stage with a critic is a fraction of it. Stage 1 is listed at the checkpoint that initializes stage 2; training it further did not reach the reported model.}
    \label{tab:supp-h3-baseline-recipes}
    \footnotesize
    \setlength{\tabcolsep}{4pt}
    \begin{tabular}{@{}lll@{}}
        \toprule
        & Stage 1 & Stage 2 (the reported model) \\
        \midrule
        \multicolumn{3}{@{}l}{\textit{\rcm$^{\dagger}$: dCM, then dCM${}+{}$\dmd}} \\
        Learning rate & $5\times10^{-6}$ & $5\times10^{-5}$ (critic $2\times10^{-5}$) \\
        Iterations & 3{,}000 & 6{,}000, saved every 500 \\
        Initialization & LoRA from a full-parameter run & stage 1 \\
        Student:critic & no critic & $1{:}4$ \\
        Loss & $L_{\mathrm{dCM}}$ & $100\,L_{\mathrm{dCM}}+L_{\mathrm{DMD}}$ \\
        Reported & -- & 2{,}000 iterations, 400 student updates \\
        \midrule
        \multicolumn{3}{@{}l}{\textit{\anyflow$^{\dagger}$: flow map, then flow map${}+{}$\dmd}} \\
        Learning rate & $5\times10^{-5}$ & $5\times10^{-6}$ (critic $5\times10^{-6}$) \\
        Iterations & 500 & 3{,}000, saved every 500 \\
        Initialization & base model & stage 1 \\
        Student:critic & no critic & $1{:}1$ \\
        Loss & $L_{\mathrm{flowmap}}$ & $L_{\mathrm{flowmap}}+1.0\,L_{\mathrm{DMD}}$ \\
        Reported & -- & 3{,}000 iterations, 1{,}500 student updates \\
        \bottomrule
    \end{tabular}
\end{table}

\subsection{Wan2.1-T2V-1.3B}
\label{sec:supp-wan-protocol}

\paragraph{Training.}
We build training on the \rcm codebase~\citep{zheng2026rcm} with its consistency loss disabled and the \dmd implementation of Hyper-Bagel~\citep{lu2025hyperbagel}, so \dmd is the only objective, and train every parameter of the student from the Wan2.1-T2V-1.3B weights. Training uses 42K captions without real-video latents. The frozen Wan2.1 teacher uses guidance scale 5 throughout distribution matching. \cref{tab:supp-train-config} lists the optimizer, learning rates, batch and sampler. One critic step follows every student step.

\paragraph{Baselines.}
\dmdtwodag is our reimplementation of \dmdtwo on Wan2.1-T2V-1.3B, with the DMD2 discriminator head and five critic steps per student step. Its discriminator positives are clean samples generated by rolling out MiniMax-H3, chosen for their higher visual quality than Wan2.1-T2V-1.3B rollouts. These H3-generated samples provide the adversarial supervision; the distribution-matching teacher remains Wan2.1-T2V-1.3B. \dmd$^{\dagger}$ removes the discriminator head and the extra critic steps, so the critic and the student update once each. \methodname{} adds only the projection to the \dmd$^{\dagger}$ run. \cref{tab:wan13b} reports \dmd$^{\dagger}$ at 1{,}500 iterations, \dmdtwodag at 6{,}500 and \methodname{} at 5{,}500, the best checkpoint of each run under the evaluation below. \rcm, ADV~\citep{you2026adaptive} and \anyflow are the publicly released Wan2.1-T2V-1.3B checkpoints: \rcm samples with $\sigma_{\max}=1600$ and intermediate times $0.783$, $0.609$ and $0.406$, ADV with intermediate times $0.900$, $0.757$ and $0.522$, and \anyflow with its official 4-step bidirectional setting.

\paragraph{Evaluation.}
We reuse \anyflow's evaluation end to end: its 944 augmented VBench prompts with five samples each ($4{,}720$ clips), its per-sample seeds and video writer, and the same 16 VBench dimensions and weights for the quality, semantic and total scores. All distilled rows use 4 NFEs without classifier-free guidance. The two teacher rows use UniPC with 50 and 4 steps, a shift of 8, and guidance scale 5, which doubles their NFE. The dynamic-degree column is the VBench dimension of that name.

\paragraph{User study.}
The preference columns of \cref{tab:wan13b} follow the MiniMax-H3 protocol of \cref{sec:h3-eval-protocol} without the audio or overall-quality questions, with a pool of 20 annotators. The prompts are a uniform subset of VBench, every tenth prompt, 95 in all, and every clip uses seed 0. \methodname{} is paired with each of six baselines of \cref{tab:wan13b}: the 50-step teacher, rCM, ADV, AnyFlow, DMD$^{\dagger}$ and \dmdtwodag, 570 pairs in all, each judged on text alignment, visual quality and motion quality. \cref{tab:supp-wan-userstudy} gives the win, tie and loss shares. The undistilled 4-step model is left out because many of its clips are black frames.

\begin{table}[t]
    \centering
    \caption{\textbf{Vote distribution behind the Wan2.1 user study of \cref{tab:wan13b}.} Each row is \methodname{} against one baseline, on the three questions of that study. Win, tie and loss are shares of the judgments, in percent, and net is the win share minus the loss share. Text alignment draws ties on roughly $82$--$88\%$ of judgments throughout; the visual and motion questions separate them. The grey row is the 50-step teacher.}
    \label{tab:supp-wan-userstudy}
    \footnotesize
    \setlength{\tabcolsep}{2.7pt}
    \begin{tabular}{@{}lrrrrrrrrrrrr@{}}
        \toprule
        & \multicolumn{4}{c}{Text alignment} & \multicolumn{4}{c}{Visual quality}
        & \multicolumn{4}{c}{Motion quality} \\
        \cmidrule(lr){2-5}\cmidrule(lr){6-9}\cmidrule(lr){10-13}
        Baseline & \multicolumn{1}{c}{Win} & \multicolumn{1}{c}{Tie} & \multicolumn{1}{c}{Loss} & \multicolumn{1}{c}{Net} & \multicolumn{1}{c}{Win} & \multicolumn{1}{c}{Tie} & \multicolumn{1}{c}{Loss} & \multicolumn{1}{c}{Net}
        & \multicolumn{1}{c}{Win} & \multicolumn{1}{c}{Tie} & \multicolumn{1}{c}{Loss} & \multicolumn{1}{c}{Net} \\
        \midrule
        \deemph{Wan2.1-T2V-1.3B, $50{\times}2$ NFE} & \deemph{8.5} & \deemph{82.5} & \deemph{8.9} & \deemph{$-0.4$} & \deemph{36.7} & \deemph{44.3} & \deemph{19.1} & \deemph{$+17.6$} & \deemph{39.9} & \deemph{44.5} & \deemph{15.6} & \deemph{$+24.3$} \\
        \dmd$^{\dagger}$ & 9.3 & 87.8 & 2.8 & $+6.5$ & 58.1 & 37.4 & 4.5 & $+53.6$ & 37.7 & 59.8 & 2.5 & $+35.2$ \\
        \dmdtwodag & 10.1 & 87.2 & 2.7 & $+7.5$ & 49.9 & 39.7 & 10.4 & $+39.5$ & 35.9 & 59.0 & 5.1 & $+30.9$ \\
        ADV~\citep{you2026adaptive} & 10.7 & 81.5 & 7.8 & $+2.9$ & 46.6 & 41.4 & 12.0 & $+34.6$ & 33.3 & 52.7 & 14.1 & $+19.2$ \\
        \rcm~\citep{zheng2026rcm} & 5.7 & 86.6 & 7.8 & $-2.1$ & 46.7 & 40.4 & 12.9 & $+33.8$ & 29.6 & 58.4 & 11.9 & $+17.7$ \\
        \anyflow~\citep{gu2026anyflow} & 8.7 & 83.3 & 8.1 & $+0.6$ & 39.5 & 45.0 & 15.5 & $+24.0$ & 39.7 & 49.5 & 10.8 & $+28.9$ \\
        \bottomrule
    \end{tabular}
\end{table}

\subsection{MiniMax-H3-33B}
\label{sec:h3-eval-protocol}

\paragraph{Training.}
The student and critic use LoRA adapters with rank and scaling parameter 128 on the attention projections and both feed-forward layers of every DiT block; the teacher is frozen and, being guidance-distilled, is sampled without classifier-free guidance. Training uses the 248K prompt-only VidProM prompts~\citep{wang2024vidprom} released with Self-Forcing~\citep{huang2025selfforcing}, without rewriting and without real-video data. Each GPU holds one 5-second 544p clip; training batches contain mixed aspect ratios. \Cref{tab:supp-train-config} lists the optimizer, learning rates, and batch size. The projection of \cref{eq:method-pdmd-projection} is applied per sample and separately to the video and audio latents, each with its own residual $\bm r$. The remaining training choices follow Hyper-Bagel~\citep{lu2025hyperbagel}, a multi-step \dmd distillation framework. We also adopt the two-timescale update rule (TTUR, several critic updates per student update) of \dmdtwo~\citep{yin2024improved}: five critic updates follow every student update, and we count every update as one iteration.

\paragraph{Baselines.}
All \dmd-family rows of \cref{tab:h3} share the setup above. \methodname{} is reported at 3{,}500 iterations. \dmd$^{\dagger}$ is reported at 5{,}500 iterations of the run with both learning rates divided by ten, the best \dmd run of our learning-rate search; at the \methodname{} learning rate \dmd peaks at 500 iterations and degrades afterwards (\cref{fig:h3-score-curves}). \dmdtwodag adds the DMD2 discriminator head and is reported at 500 iterations at the \methodname{} learning rate. The discriminator's positive examples are clean samples generated by rolling out the frozen MiniMax-H3 teacher; no real-video data are used. The same search was run for \dmdtwodag, with both learning rates divided by ten and by one hundred over 0--6{,}000 iterations. No checkpoint of either run exceeded the DMD2 500-iteration score at the original learning rate. \rcm$^{\dagger}$ and \anyflow$^{\dagger}$ each port a two-stage reference recipe to MiniMax-H3; \cref{tab:supp-h3-baseline-recipes} gives both stages and \cref{sec:h3-baseline-notes} the two places where our port departs from the reference. \rcm$^{\dagger}$ is reported at 2{,}000 iterations of its second stage and \anyflow$^{\dagger}$ at 3{,}000, which at their update ratios are 400 and 1{,}500 student updates against \methodname{}'s 3{,}500 iterations. For both \rcm$^{\dagger}$ and \anyflow$^{\dagger}$ we trained three learning rates, the one in \cref{tab:supp-h3-baseline-recipes} and that rate multiplied and divided by five, and report the best of the three. We also varied the weight of the \dmd term of \rcm$^{\dagger}$ and saw no clear difference in the final scores. H3 Turbo LoRA~\citep{larry2026h3turbo} is the released v4-600 EMA adapter.

\subsubsection*{Where the ports depart from the reference}
\label{sec:h3-baseline-notes}
Two settings could not be carried over unchanged. First, both references obtain the time derivative of the student with a Jacobian--vector product, which MiniMax-H3 does not support, so we use the finite-difference tangent that the reference implementation also provides (its \texttt{fd\_type=2}) and approximate $\mathrm{d}u/\mathrm{d}t$ by finite differences in both ports. Second, the \anyflow$^{\dagger}$ flow-map stage uses $\epsilon=0.005$, a fresh coupling per step, flow-map ratios $0.5$ and $0.25$, and fp32 embedders; its second stage applies \dmd to the rollout endpoint with $N\in\{2,4,8,16,50\}$, and keeps the reference's rule that the critic and student learning rates are equal.

\paragraph{Evaluation.}
Every MiniMax-H3 score in the paper, including the per-checkpoint tracking of \cref{sec:ablation-dynamics}, is computed on all 387 prompts of VideoGen-Eval~\citep{yang2025videogeneval} with one harness; we do not reimplement any scoring rule. Each prompt is rewritten once with Qwen3.8-27B~\citep{qwen2026qwen38}, using the official MiniMax-H3 video-prompt writing guide~\citep{minimaxh3promptguide} as the system prompt, and the same rewritten prompt is given to every method. One clip per prompt is rendered at 4 NFEs with the released scheduler (shift 12 for video and 3 for audio), 124 frames at 24\,fps, seed 42, and 544p with a prompt-dependent aspect ratio: across the 387 prompts we cycle through $21{:}9$, $9{:}21$, $16{:}9$, $9{:}16$, $1{:}1$, $4{:}3$ and $3{:}4$. Quality is computed from the seven prompt-independent VBench~\citep{huang2023vbench} dimensions (subject consistency, background consistency, temporal flickering, motion smoothness, dynamic degree, aesthetic quality, imaging quality), normalized and weighted as in VBench. The nine semantic dimensions (object class, multiple objects, human action, color, spatial relationship, scene, appearance style, temporal style, overall consistency) follow each VBench dimension's own rule, with Qwen3.8-27B as the vision-language judge; the semantic score is their VBench-weighted average, and the total score is $(4\cdot\mathrm{quality}+\mathrm{semantic})/5$. Semantic dimensions differ in how many of the 387 prompts they apply to (spatial relationship 25, appearance style 30, overall consistency all 387), so we read them at the aggregate level.
\label{sec:vgeneval387}

\paragraph{Audio metrics.}
\label{sec:supp-h3-audio-crossmodal}
The first four audio columns of \cref{tab:h3} evaluate the generated audio alone: PQ, CE and CU are the production quality, content enjoyment and content usefulness of Audiobox Aesthetics~\citep{tjandra2025audiobox}, and IS is the Inception score~\citep{salimans2016improved} over PaSST~\citep{koutini2022passt} audio tags. Two additional metrics compare the audio with the video: IB is the ImageBind~\citep{girdhar2023imagebind} cosine between the audio and the video of the same clip, and DeSync is the audio--video offset in seconds predicted by Synchformer~\citep{iashin2024synchformer}; both come from the av-benchmark suite of MMAudio~\citep{cheng2025mmaudio}. \methodname{} leads both among the 4-NFE rows.

\paragraph{User study.}
The preference columns of \cref{tab:h3} come from a pairwise study on the same 387 VideoGen-Eval clips as the table, judged by 20 annotators. Each pair shows the \methodname{} clip and the clip of one baseline for the same prompt and seed, side by side in random left--right order, and the annotator picks the better clip or a tie on five questions: an overall judgment, text alignment, visual quality, motion quality, and audio quality. The seven baselines are the rows of \cref{tab:h3}, including the 50-step teacher and the undistilled 4-step model. For each baseline and question the table reports the share of votes for \methodname{} minus the share of votes for the baseline, in percent, so a positive number favors \methodname{} and ties count for neither side. \cref{tab:supp-h3-userstudy} gives the underlying win, tie and loss shares. All $387$ pairs were judged for each baseline.

\begin{table}[t]
    \centering
    \caption{\textbf{Vote distribution behind the MiniMax-H3 user study of \cref{tab:h3}.} Each row is \methodname{} against one baseline, on the five questions of that study. Win, tie and loss are shares of the judgments, in percent, and net is the win share minus the loss share, computed from the displayed percentages. Text alignment draws ties on roughly $64$--$80\%$ of judgments throughout. Grey rows are the undistilled base model at 50 and 4 sampling steps.}
    \label{tab:supp-h3-userstudy}
    \scriptsize
    \setlength{\tabcolsep}{1.8pt}
    \resizebox{\textwidth}{!}{%
    \begin{tabular}{@{}l*{20}{r}@{}}
        \toprule
        & \multicolumn{4}{c}{Overall quality} & \multicolumn{4}{c}{Text alignment} & \multicolumn{4}{c}{Visual quality} & \multicolumn{4}{c}{Motion quality} & \multicolumn{4}{c}{Audio quality} \\
        \cmidrule(lr){2-5}\cmidrule(lr){6-9}\cmidrule(lr){10-13}\cmidrule(lr){14-17}\cmidrule(lr){18-21}
        Baseline & \multicolumn{1}{c}{Win} & \multicolumn{1}{c}{Tie} & \multicolumn{1}{c}{Loss} & \multicolumn{1}{c}{Net} & \multicolumn{1}{c}{Win} & \multicolumn{1}{c}{Tie} & \multicolumn{1}{c}{Loss} & \multicolumn{1}{c}{Net} & \multicolumn{1}{c}{Win} & \multicolumn{1}{c}{Tie} & \multicolumn{1}{c}{Loss} & \multicolumn{1}{c}{Net} & \multicolumn{1}{c}{Win} & \multicolumn{1}{c}{Tie} & \multicolumn{1}{c}{Loss} & \multicolumn{1}{c}{Net} & \multicolumn{1}{c}{Win} & \multicolumn{1}{c}{Tie} & \multicolumn{1}{c}{Loss} & \multicolumn{1}{c}{Net} \\
        \midrule
        \deemph{MiniMax-H3-33B, 50 NFE} & \deemph{14.2} & \deemph{34.6} & \deemph{51.2} & \deemph{$-37.0$} & \deemph{10.6} & \deemph{71.3} & \deemph{18.1} & \deemph{$-7.5$} & \deemph{8.0} & \deemph{48.6} & \deemph{43.4} & \deemph{$-35.4$} & \deemph{8.0} & \deemph{72.1} & \deemph{19.9} & \deemph{$-11.9$} & \deemph{7.5} & \deemph{81.1} & \deemph{11.4} & \deemph{$-3.9$} \\
        \deemph{MiniMax-H3-33B, 4 NFE} & \deemph{79.3} & \deemph{11.6} & \deemph{9.0} & \deemph{$+70.3$} & \deemph{26.4} & \deemph{63.8} & \deemph{9.8} & \deemph{$+16.6$} & \deemph{78.3} & \deemph{16.0} & \deemph{5.7} & \deemph{$+72.6$} & \deemph{50.1} & \deemph{44.4} & \deemph{5.4} & \deemph{$+44.7$} & \deemph{49.6} & \deemph{48.3} & \deemph{2.1} & \deemph{$+47.5$} \\
        H3 Turbo LoRA & 51.2 & 32.3 & 16.5 & $+34.7$ & 13.4 & 76.0 & 10.6 & $+2.8$ & 49.9 & 38.2 & 11.9 & $+38.0$ & 26.9 & 65.6 & 7.5 & $+19.4$ & 13.7 & 81.9 & 4.4 & $+9.3$ \\
        \dmd$^{\dagger}$ & 45.2 & 39.3 & 15.5 & $+29.7$ & 15.2 & 73.1 & 11.6 & $+3.6$ & 34.9 & 55.3 & 9.8 & $+25.1$ & 19.9 & 72.4 & 7.8 & $+12.1$ & 15.5 & 79.8 & 4.7 & $+10.8$ \\
        \dmdtwodag & 51.7 & 31.8 & 16.5 & $+35.2$ & 19.1 & 70.8 & 10.1 & $+9.0$ & 39.0 & 50.6 & 10.3 & $+28.7$ & 27.6 & 64.1 & 8.3 & $+19.3$ & 26.6 & 69.0 & 4.4 & $+22.2$ \\
        \rcm$^{\dagger}$ & 60.2 & 23.3 & 16.5 & $+43.7$ & 16.3 & 70.5 & 13.2 & $+3.1$ & 62.8 & 25.6 & 11.6 & $+51.2$ & 39.8 & 51.9 & 8.3 & $+31.5$ & 27.9 & 61.2 & 10.9 & $+17.0$ \\
        \anyflow$^{\dagger}$ & 69.5 & 20.2 & 10.3 & $+59.2$ & 9.0 & 80.4 & 10.6 & $-1.6$ & 72.4 & 21.7 & 5.9 & $+66.5$ & 33.9 & 60.2 & 5.9 & $+28.0$ & 30.0 & 65.9 & 4.1 & $+25.9$ \\
        \bottomrule
    \end{tabular}%
    }
\end{table}

\section{Additional Controlled Experiments on MiniMax-H3}
\label{sec:additional-controls}

The ablations of \cref{sec:ablations} change the projection direction only and keep every other setting fixed. This section provides additional measurements and controls for those runs. \cref{sec:h3-proj-directions-training} follows the six projection variants on three prompts to 3{,}000 iterations. \cref{sec:h3-norm-ratio} compares retained update norms, and \cref{sec:h3-proj-endpoints} decodes one training batch to show what the projection removes. \cref{sec:h3-larger-step,sec:h3-partial-retention} address two questions left open by the ablations: whether \methodname{} stays stable only because the projection shrinks the update, and whether a partial projection still stabilizes training. \cref{sec:h3-two-step} repeats the comparison at 2 NFE, and \cref{sec:h3-critic-ratio} trains \methodname{} with one critic update per student update. Every run uses the 4-NFE MiniMax-H3 recipe of \cref{sec:ablations}: the same initialization, data, LoRA rank, TTUR schedule, and sampler, except where the subsection says otherwise. We save a checkpoint every 500 iterations and score each one on the 387 VideoGen-Eval prompts with the protocol of \cref{sec:vgeneval387}. Scores are in percent.

\subsection{Projection Directions Over Training}
\label{sec:h3-proj-directions-training}

\cref{fig:projection-ablation-main} shows one prompt at 1{,}500 iterations. \cref{fig:projection-ablation} follows the same six variants on three SoRA prompts to 3{,}000 iterations, with the frames at 500, 1{,}500 and 3{,}000 iterations stacked under each prompt. In these examples, all five control variants degrade in different ways, while \methodname{} continues to follow its prompt at 3{,}000 iterations.

\ifdefined\pacellw\else\newlength{\pacellw}\fi
\ifdefined\pacellh\else\newlength{\pacellh}\fi
\newcommand{\pasetup}{%
    \setlength{\tabcolsep}{0pt}%
    \renewcommand{\arraystretch}{0}%
    \setlength{\pacellw}{\dimexpr(\textwidth-0.32in-9pt)/6\relax}%
    \setlength{\pacellh}{\dimexpr\pacellw*4/7\relax}%
}
\newcommand{\pap}[1]{\makebox[0.17in][c]{%
    \raisebox{\dimexpr0.5\pacellh-0.5\height\relax}[0pt][0pt]{%
    \rotatebox{90}{\scriptsize Prompt #1}}}}
\newcommand{\pai}[1]{\makebox[0.15in][r]{%
    \raisebox{\dimexpr0.5\pacellh-0.5\height\relax}[0pt][0pt]{\tiny #1}}}
\newcommand{\pah}[2]{\makebox[\pacellw][c]{%
    \shortstack{\scriptsize #1\\[1pt]\scriptsize #2}}}
\newcommand{\pa}[1]{\includegraphics[width=\pacellw]%
    {figures/src/projection_ablation/#1}}
\newcommand{\parow}[2]{\pa{dmd_#1_#2} & \pa{rand_#1_#2} & \pa{cscore_#1_#2} &
    \pa{tres_#1_#2} & \pa{rpar_#1_#2} & \pa{pdmd_#1_#2}}
\newcommand{\pahead}{%
    & & \pah{DMD$^{\dagger}$}{(no projection)}
    & \pah{Random}{direction $\perp$}
    & \pah{Critic}{score $\perp$}
    & \pah{Teacher endpoint}{residual $\perp$}
    & \pah{Critic endpoint}{residual $\parallel$}
    & \pah{\textbf{Critic endpoint}}{\textbf{residual $\perp$ (\methodname{})}}}

\begin{figure}[p]
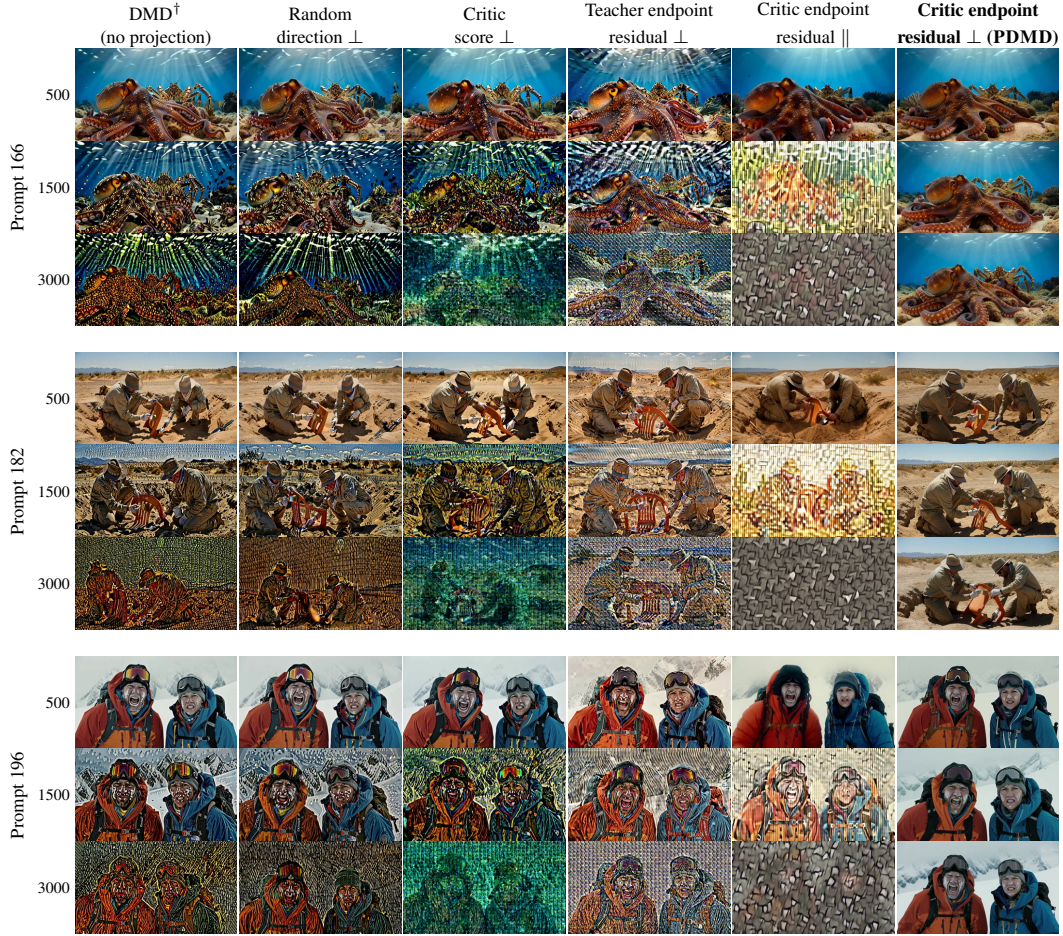

    \centering
    \pasetup
    \begin{tabular}{@{}r@{\hspace{2pt}}r@{\hspace{2pt}}c@{\hspace{1pt}}c@{\hspace{1pt}}c@{\hspace{1pt}}c@{\hspace{1pt}}c@{\hspace{1pt}}c@{}}
        \pahead \\[2pt]
        & \pai{500}  & \parow{166}{500} \\
        \pap{166} & \pai{1500} & \parow{166}{1500} \\
        & \pai{3000} & \parow{166}{3000} \\[10pt]
        & \pai{500}  & \parow{182}{500} \\
        \pap{182} & \pai{1500} & \parow{182}{1500} \\
        & \pai{3000} & \parow{182}{3000} \\[10pt]
        & \pai{500}  & \parow{196}{500} \\
        \pap{196} & \pai{1500} & \parow{196}{1500} \\
        & \pai{3000} & \parow{196}{3000} \\
    \end{tabular}
    \caption{\textbf{Effect of projection direction during 4-NFE MiniMax-H3 distillation.} Six variants that differ only in which direction is removed from the \dmd update $\bm d$, one per column, in the order of \cref{fig:projection-ablation-main}. Each row is one of three SoRA prompts, given by its prompt id, and every row is a tight vertical triple: the same run at 500, 1{,}500 and 3{,}000 iterations, top to bottom, so reading downwards follows one training trajectory. Every frame is taken at 30\% of the generated $1344{\times}768$ clip; prompt and seed are the same across variants. In these examples, all five control variants degrade in different ways. Retaining the student-critic endpoint-parallel component instead of removing it is the fastest failure: by 3{,}000 iterations it produces the same prompt-independent pebble texture for every prompt. \methodname{} continues to follow its prompt at 3{,}000 iterations. The five control variants show severe degradation on all three illustrated prompts.}
    \label{fig:projection-ablation}
\end{figure}

\subsection{Retained Update Norm}
\label{sec:h3-norm-ratio}

The four projection variants shown in \cref{fig:supp-h3-norm-ratio}, trained on MiniMax-H3 in the 4-NFE setup of \cref{sec:ablations}, each retain one component of the \dmd update $\bm d$ and discard its complement. The error-removal ratio $\gamma_e$ of \cref{tab:supp-2d-gamma} is not available here: it is defined against the population-optimal critic endpoint $\bm c^\star(q)$, which we estimate from a large student-sample bank in the 2D setting of \cref{sec:supp-2d-exp1}. In large-scale video generation we therefore measure what is observable, how much of the update each variant retains, and read the error question from three proxies: this norm ratio, the decoded removed component of \cref{sec:h3-proj-endpoints}, and the partial-projection control of \cref{sec:h3-partial-retention}. We log, at every student step, the ratio $\|\bm d_{\mathrm{kept}}\|/\|\bm d\|$ between the retained component and the full update, separately for the video and the audio latents. \cref{fig:supp-h3-norm-ratio} shows the ratio over training. The \dmd variant applies no projection, so its ratio is $1$ by definition and is not drawn.

Three observations follow. First, the random-direction variant keeps the full norm: a random direction in a high-dimensional latent space is nearly orthogonal to $\bm d$, so removing that component changes almost nothing, and the variant behaves like \dmd. Second, \methodname{} removes between $8\%$ and $17\%$ of the video update norm for over $90\%$ of training (median $12\%$). The critic-score variant removes a median of $6\%$ and still degrades; it removes more only late in training, after its scores have already collapsed (\cref{fig:h3-score-curves}). These comparisons suggest that a smaller update alone does not explain the stability of \methodname{}. Third, the residual-parallel variant keeps only $5$--$30\%$ of the video update (median $11\%$). This variant degenerates fastest, consistent with the parallel component containing a destabilizing part of the update. Taken together, \cref{fig:supp-h3-norm-ratio} shows that stability does not follow the amount removed: the random direction removes almost none of the update and degrades, the residual-parallel variant retains a median $11\%$ and degrades fastest, whereas \methodname{} retains $88\%$ and remains stable over the evaluated interval. What separates them is which component is removed, which is the ordering $\gamma_e$ measures directly in 2D.

\begin{figure}[t]
\centering
\includegraphics[width=0.92\linewidth]{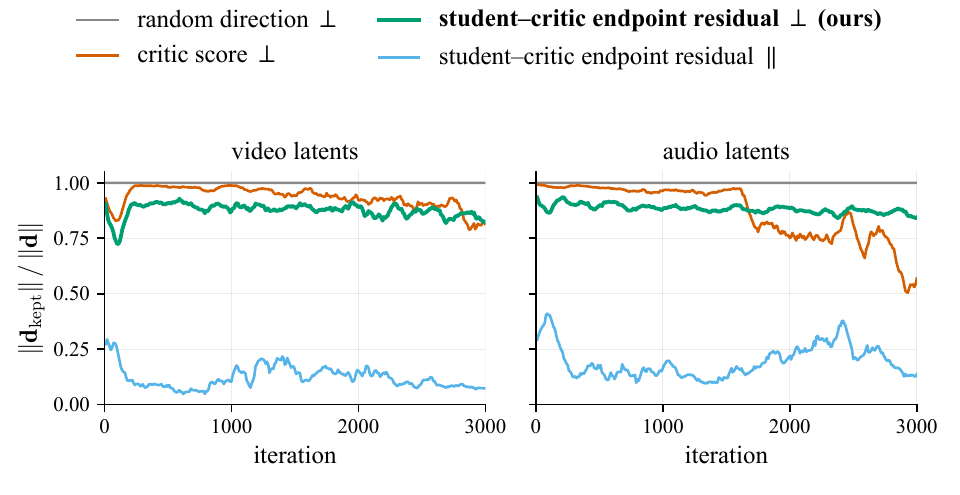}
\caption{\textbf{Retained fraction of the \dmd update norm on MiniMax-H3 at 4 NFE.}
$\bm d_{\mathrm{kept}}$ is the component the variant keeps, as a 100-iteration running mean.
Left: video latents.
Right: audio latents.}
\label{fig:supp-h3-norm-ratio}
\end{figure}

\subsection{Restoring Half of the Removed Component}
\label{sec:h3-partial-retention}

\methodname{} removes the component of $\bm d$ along the student--critic endpoint residual $\bm r$ entirely. To test whether a partial projection is enough, we interpolate between \dmd and \methodname{} with
\begin{equation}
    \bm d_\beta
    =
    \bm d_\perp + \beta\,\operatorname{proj}_{\bm r}(\bm d),
    \label{eq:partial-retention}
\end{equation}
where $\beta=0$ recovers \methodname{} and $\beta=1$ recovers plain \dmd. We train the midpoint $\beta=0.5$ for 5{,}000 iterations with the learning rates of \methodname{}. At $\beta=0.5$ half of the removed component is put back. The residual-parallel component scales linearly with $\beta$. If it drives degradation, restoring it should reduce the stability benefit; the analysis does not predict a quantitative degradation rate.

\cref{fig:supp-h3-extra-step-retain} is consistent with this prediction. The $\beta=0.5$ run stays at the level of the stable \methodname{} run for the first 2{,}000 iterations, with a total score between 82.64 and 82.98. The run crosses the 4-step base-model score at twice the iteration count of \dmd. \dmd drops below the 4-step base model (79.48 total) at 2{,}000 iterations and reaches 70.07 at 2{,}500; the $\beta=0.5$ run drops below the base model at 4{,}000 iterations and reaches 74.90 at 5{,}000. In this setting, restoring half of the residual-parallel component is enough to reproduce the degradation, while removing the full component remains stable over the evaluated interval.

\subsection{A Larger Student Step}
\label{sec:h3-larger-step}

\cref{sec:h3-norm-ratio} shows that \methodname{} removes a median $12\%$ of the video update norm. To test whether the projection helps only because its update is smaller than the \dmd update, we train \methodname{} with the student learning rate raised from $5\times10^{-5}$ to $6.25\times10^{-5}$, a factor of $1.25$, and keep the critic learning rate at $1\times10^{-5}$. The projected update keeps between $83\%$ and $92\%$ of the \dmd norm for most of training (\cref{sec:h3-norm-ratio}), motivating this scale factor: $1.25\times0.83>1$. This is a learning-rate stress test, not an exact match of parameter-update norms, because the generator Jacobian and AdamW also transform the update.

\cref{fig:supp-h3-extra-step-retain} plots the run next to \dmd at the base learning rate. The $1.25\times$ run does not exhibit a sustained decline over this interval. The total score of every one of the ten checkpoints lies between 82.82 and 83.48, above the 50-step teacher (82.41), and the last checkpoint at 5{,}000 iterations scores 83.18.
\dmd at the base learning rate, in the same figure, has fallen to 70.07 by 2{,}500 iterations. A larger step does not reproduce the degradation, suggesting that update direction, rather than magnitude alone, drives the stability gain. \cref{fig:supp-h3-larger-step-samples} shows frames from four checkpoints of the same run, and the frames remain visually consistent over this interval.

\begin{figure}[t]
    \centering
    \includegraphics[width=0.92\linewidth]{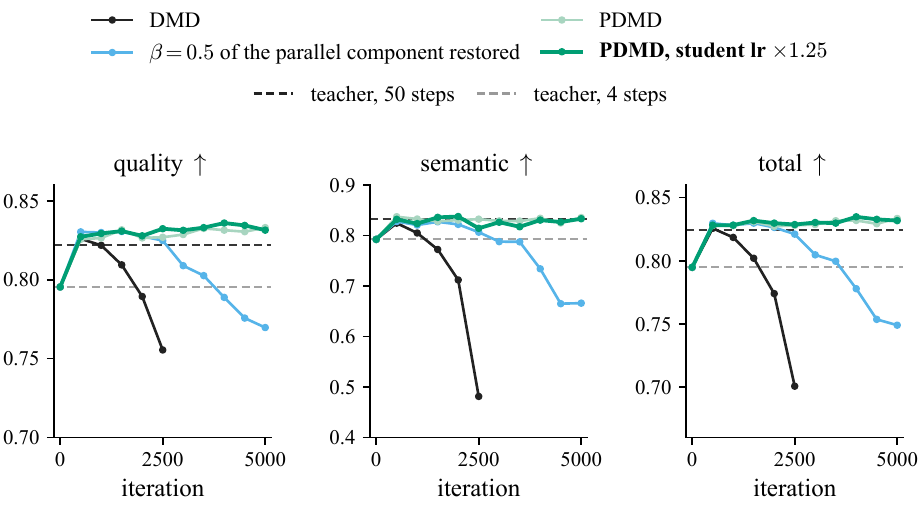}
    \caption{\textbf{A larger student step and a partial projection on MiniMax-H3 at 4 NFE.} Quality, semantic, and total scores of every 500-iteration checkpoint up to 5{,}000 iterations on the 387 VideoGen-Eval prompts. \dmd is the corresponding variant in \cref{fig:h3-score-curves}. The light green curve shows \methodname{} at the base learning rate, as far as it was scored, and the dark green curve shows \methodname{} trained with a $1.25\times$ larger student learning rate (\cref{sec:h3-larger-step}). The blue curve shows the run that restores half of the residual-parallel component removed by \methodname{}, using \cref{eq:partial-retention} with $\beta=0.5$ (\cref{sec:h3-partial-retention}). Iteration 0 is the undistilled base model at 4 NFE. Horizontal lines show the teacher scores at 50 and 4 sampling steps.}
    \label{fig:supp-h3-extra-step-retain}
\end{figure}

\begin{figure}[t]
    \centering
    \setlength{\tabcolsep}{0pt}
    \renewcommand{\arraystretch}{0}
    \scriptsize
    \begin{tabular}{@{}r@{\hspace{3pt}}c@{}c@{}c@{}c@{\hspace{5pt}}c@{}}
        \rule{0pt}{7pt}& 500 iterations & 1{,}000 iterations & 3{,}000 iterations
        & 5{,}000 iterations & best \dmd \\[2pt]
        \rule{0pt}{0.1063\linewidth}\raisebox{\dimexpr0.0532\linewidth-0.5\height\relax}[0pt][0pt]{\rotatebox{90}{Prompt 746}} &
          \includegraphics[width=0.1880\linewidth]{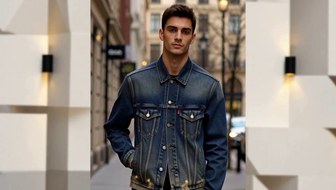} &
          \includegraphics[width=0.1880\linewidth]{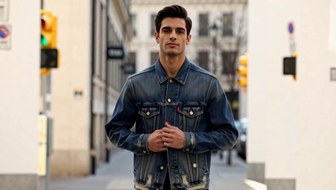} &
          \includegraphics[width=0.1880\linewidth]{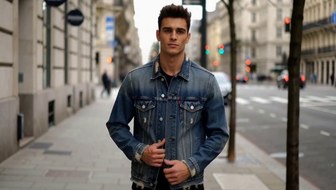} &
          \includegraphics[width=0.1880\linewidth]{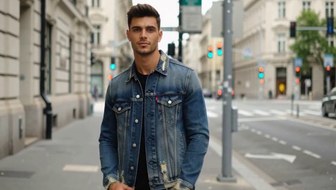} &
          \includegraphics[width=0.1880\linewidth]{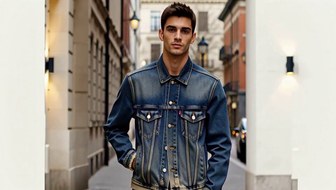} \\[6pt]
        \rule{0pt}{0.1063\linewidth}\raisebox{\dimexpr0.0532\linewidth-0.5\height\relax}[0pt][0pt]{\rotatebox{90}{Prompt 837}} &
          \includegraphics[width=0.1880\linewidth]{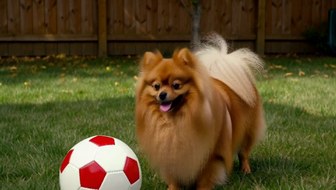} &
          \includegraphics[width=0.1880\linewidth]{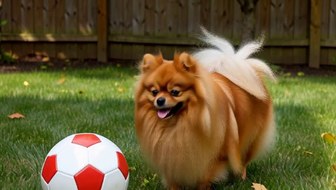} &
          \includegraphics[width=0.1880\linewidth]{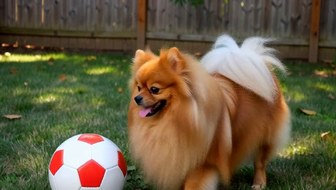} &
          \includegraphics[width=0.1880\linewidth]{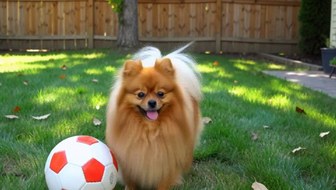} &
          \includegraphics[width=0.1880\linewidth]{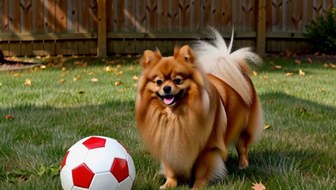} \\[6pt]
        \rule{0pt}{0.1444\linewidth}\raisebox{\dimexpr0.0722\linewidth-0.5\height\relax}[0pt][0pt]{\rotatebox{90}{Prompt 952}} &
          \includegraphics[width=0.1880\linewidth]{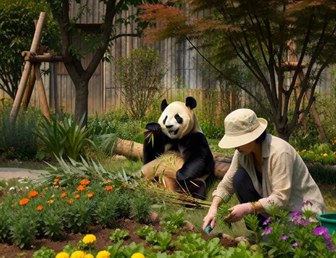} &
          \includegraphics[width=0.1880\linewidth]{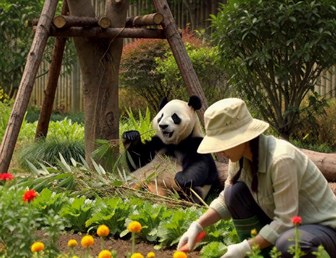} &
          \includegraphics[width=0.1880\linewidth]{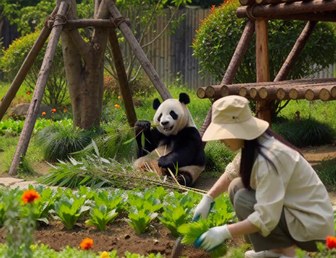} &
          \includegraphics[width=0.1880\linewidth]{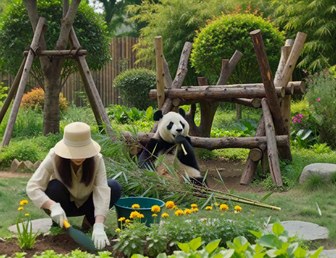} &
          \includegraphics[width=0.1880\linewidth]{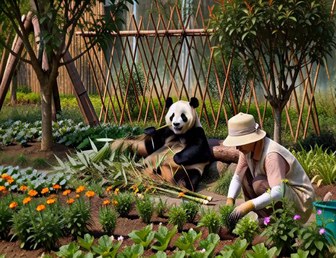} \\[6pt]
        \rule{0pt}{0.0806\linewidth}\raisebox{\dimexpr0.0403\linewidth-0.5\height\relax}[0pt][0pt]{\rotatebox{90}{Prompt 961}} &
          \includegraphics[width=0.1880\linewidth]{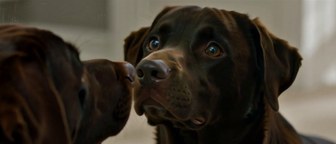} &
          \includegraphics[width=0.1880\linewidth]{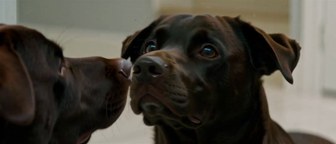} &
          \includegraphics[width=0.1880\linewidth]{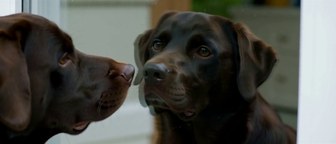} &
          \includegraphics[width=0.1880\linewidth]{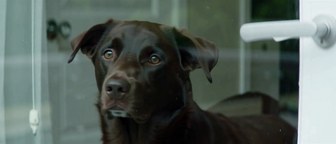} &
          \includegraphics[width=0.1880\linewidth]{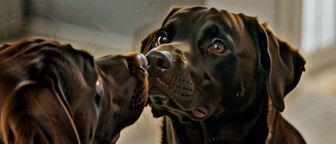} \\[6pt]
        \rule{0pt}{0.0806\linewidth}\raisebox{\dimexpr0.0403\linewidth-0.5\height\relax}[0pt][0pt]{\rotatebox{90}{Prompt 968}} &
          \includegraphics[width=0.1880\linewidth]{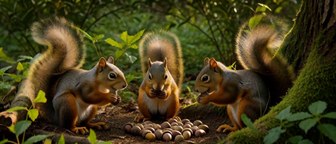} &
          \includegraphics[width=0.1880\linewidth]{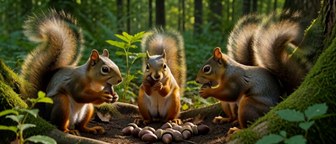} &
          \includegraphics[width=0.1880\linewidth]{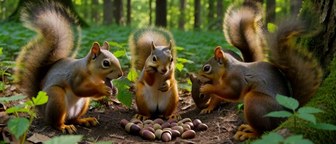} &
          \includegraphics[width=0.1880\linewidth]{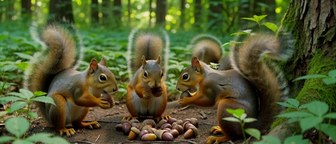} &
          \includegraphics[width=0.1880\linewidth]{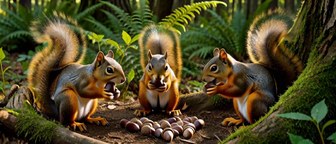} \\[6pt]
        \rule{0pt}{0.1880\linewidth}\raisebox{\dimexpr0.0940\linewidth-0.5\height\relax}[0pt][0pt]{\rotatebox{90}{Prompt 1095}} &
          \includegraphics[width=0.1445\linewidth]{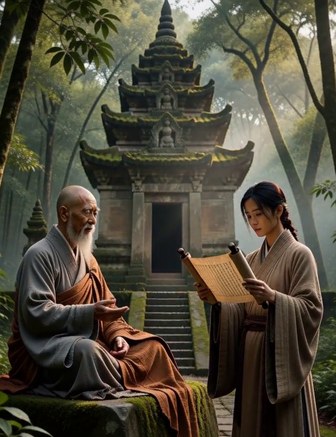} &
          \includegraphics[width=0.1445\linewidth]{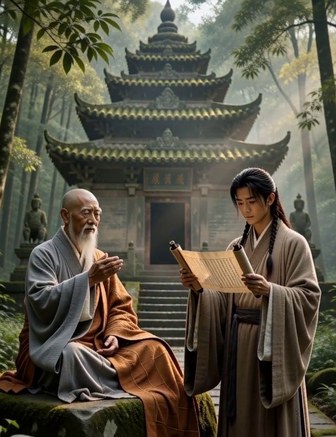} &
          \includegraphics[width=0.1445\linewidth]{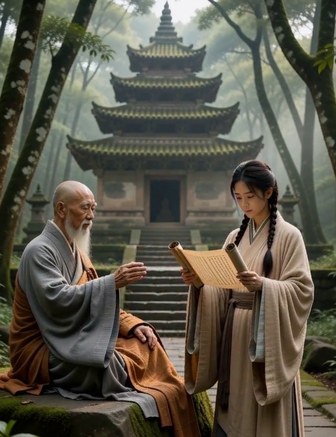} &
          \includegraphics[width=0.1445\linewidth]{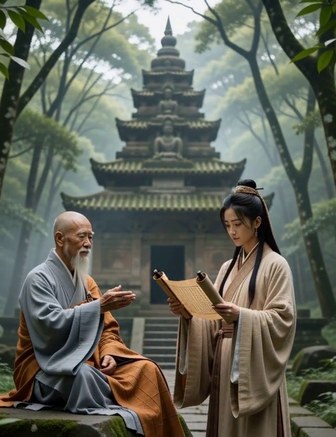} &
          \includegraphics[width=0.1445\linewidth]{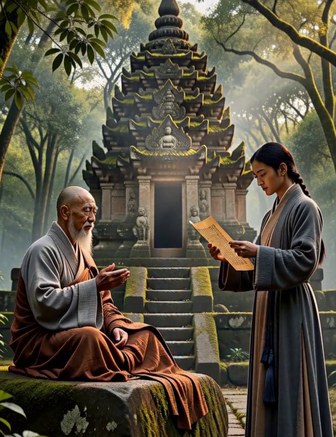} \\
    \end{tabular}
    \caption{\textbf{\methodname{} with a $1.25\times$ larger student step also does not degrade over 5{,}000 iterations.} The first four columns are checkpoints of the single \methodname{} run of \cref{sec:h3-larger-step}. The last column, set slightly apart, is the best \dmd checkpoint over a learning-rate search, the \dmd$^{\dagger}$ row of \cref{tab:h3}. Rows are labelled with the VideoGen-Eval prompt id; every panel is the frame at 80\% of the clip, with the same prompt, seed and sampler throughout.}
    \label{fig:supp-h3-larger-step-samples}
\end{figure}

\clearpage
\subsection{One Critic Update per Student Update on MiniMax-H3}
\label{sec:h3-critic-ratio}

The DMD-family MiniMax-H3 rows of \cref{tab:h3} take five critic updates per student update, the two-timescale update rule (TTUR) of \dmdtwo described in \cref{sec:h3-eval-protocol}. On Wan2.1, our \dmd and \methodname{} runs take one critic update per student update, while \dmdtwodag takes five. We repeat \methodname{} on MiniMax-H3 with one critic update per student update and every other setting of \cref{sec:h3-eval-protocol} unchanged. \Cref{tab:supp-h3-critic-ratio} compares the 1{,}000-iteration checkpoint of this run with the \dmd$^{\dagger}$ and \methodname{} rows of \cref{tab:h3}, which keep the five critic updates, and with \dmd trained with one critic update per student update, reported at 1{,}500 iterations, on the video and audio metrics. With one critic update per student update, \methodname{} scores 82.90 total, above \dmd$^{\dagger}$ with five and above every 4-NFE baseline of \cref{tab:h3}. Our projection does not depend on the extra critic updates. It is compatible with them.

\begin{table}[!htb]
    \centering
    \caption{\textbf{\methodname{} with one critic update per student update on MiniMax-H3.} VideoGen-Eval scores on the 387 prompts (percentages) and audio metrics of the same clips, as in \cref{tab:supp-h3-nfe2}. Student\,:\,critic is the ratio of student to critic updates; the best value per column is bold.}
    \label{tab:supp-h3-critic-ratio}
    \scriptsize
    \setlength{\tabcolsep}{2.5pt}
    \begin{tabular}{@{}lccccccccccc@{}}
        \toprule
        & & & \multicolumn{3}{c}{Video} & \multicolumn{4}{c}{Audio quality} & \multicolumn{2}{c}{Audio--video} \\
        \cmidrule(lr){4-6} \cmidrule(lr){7-10} \cmidrule(lr){11-12}
        Method & NFE & Student\,:\,critic & \textbf{Total} & Quality & Semantic & PQ & CE & CU & IS & IB & DeSync $\downarrow$ \\
        \midrule
        \dmd$^{\dagger}$ & 4 & 1\,:\,1 & 82.37 & 82.24 & 82.89 & 6.321 & 3.674 & 5.900 & 4.32 & 0.148 & 0.891 \\
        \dmd$^{\dagger}$ & 4 & 1\,:\,5 & 82.76 & 82.68 & \textbf{83.05} & 6.381 & 3.801 & 5.931 & 4.79 & 0.176 & 0.905 \\
        \methodname{} & 4 & 1\,:\,1 & 82.90 & 83.14 & 81.92 & \textbf{6.538} & \textbf{4.123} & \textbf{6.209} & \textbf{5.08} & 0.191 & 0.808 \\
        \methodname{} (\cref{tab:h3}) & 4 & 1\,:\,5 & \textbf{83.17} & \textbf{83.25} & 82.86 & 6.530 & 4.062 & 6.180 & 4.98 & \textbf{0.195} & \textbf{0.802} \\
        \bottomrule
    \end{tabular}
\end{table}

\subsection{Two-Step Generation}
\label{sec:h3-two-step}

We repeat the comparison at 2 NFE. The two runs share the setup of \cref{sec:h3-eval-protocol} with the student sampled in two steps; the projection is the only difference between them. Over the 6{,}000 iterations, every checkpoint of \methodname{} scores 80 total or above on the 387 VideoGen-Eval prompts, while the run without the projection falls below 60. \Cref{tab:supp-h3-nfe2} reports the best \dmd checkpoint at 1{,}000 iterations and the \methodname{} checkpoint at 2{,}000, with the audio metrics of the same two checkpoints: \methodname{} leads on all six. \Cref{fig:h3_nfe2} shows frames of these two checkpoints next to the four-step models. \Cref{fig:h3_nfe2_samples,fig:h3_nfe2_samples-2,fig:h3_nfe2_samples-3,fig:h3_nfe2_samples-4} provide eight additional prompts comparing \dmd at 2 NFE with \methodname at 2 and 4 NFE, with three uniformly sampled frames per clip.

\begin{table}[!ht]
    \centering
    \caption{\textbf{Best checkpoint of \dmd and \methodname{} at 2 NFE on MiniMax-H3.} VideoGen-Eval scores on the 387 prompts (percentages) and audio metrics of the same clips. PQ, CE, CU and IS are the audio-quality metrics of \cref{tab:h3}; IB and DeSync are the audio--video agreement metrics of \cref{tab:h3}. Arrows give the preferred direction; the best value per column is bold.}
    \label{tab:supp-h3-nfe2}
    \scriptsize
    \setlength{\tabcolsep}{3pt}
    \begin{tabular}{@{}lcccccccccc@{}}
        \toprule
        & & \multicolumn{3}{c}{Video} & \multicolumn{4}{c}{Audio quality} & \multicolumn{2}{c}{Audio--video} \\
        \cmidrule(lr){3-5} \cmidrule(lr){6-9} \cmidrule(lr){10-11}
        Method & NFE & \textbf{Total} & Quality & Semantic & PQ & CE & CU & IS & IB & DeSync $\downarrow$ \\
        \midrule
        \dmd$^{\dagger}$ & 2 & 81.58 & 81.87 & 80.39 & 4.529 & 2.380 & 3.781 & 1.88 & 0.085 & 1.022 \\
        \textbf{\methodname{} (ours)} & 2 & \textbf{82.88} & \textbf{82.83} & \textbf{83.05} & \textbf{6.061} & \textbf{3.719} & \textbf{5.597} & \textbf{4.16} & \textbf{0.177} & \textbf{0.838} \\
        \bottomrule
    \end{tabular}
\end{table}

\begingroup
\providecommand{\hmn}[4]{\rotatebox{90}{\parbox[c]{#3}{\centering #4 #1\\[0pt]#4 #2}}}
\providecommand{\hmf}[3]{\includegraphics[height=#3]{figures/src/motion/h3_#1_#2_strip.jpg}}
\providecommand{\hmi}[2]{\makebox[#2][c]{\tiny #1}}
\providecommand{\hmheadA}[1]{\hmi{0}{#1}\hmi{11}{#1}\hmi{22}{#1}\hmi{34}{#1}\hmi{45}{#1}\hmi{56}{#1}\hmi{67}{#1}\hmi{78}{#1}\hmi{89}{#1}\hmi{101}{#1}\hmi{112}{#1}\hmi{123}{#1}}
\providecommand{\hmpl}[2]{\raisebox{\dimexpr2\dimexpr#2\relax+3pt-0.5\height\relax}[0pt][0pt]{\rotatebox{90}{\scriptsize Prompt #1}}}
\providecommand{\hmnferow}[6]{%
    #6 & \hmn{#4}{#5}{#3}{\scriptsize} & \hmf{#1}{#2}{#3} \\[2pt]}
\providecommand{\hmnfeblock}[2]{%
    \hmnferow{dmd}{#1}{#2}{DMD$^{\dagger}$}{4 NFE}{}
    \hmnferow{dmd2nfe}{#1}{#2}{DMD$^{\dagger}$}{2 NFE}{}
    \hmnferow{pdmd}{#1}{#2}{\textbf{\methodname{}}}{\textbf{4 NFE}}{}
    \hmnferow{pdmd2nfe}{#1}{#2}{\textbf{\methodname{}}}{\textbf{2 NFE}}{\hmpl{#1}{#2}}}

\begin{figure}[p]
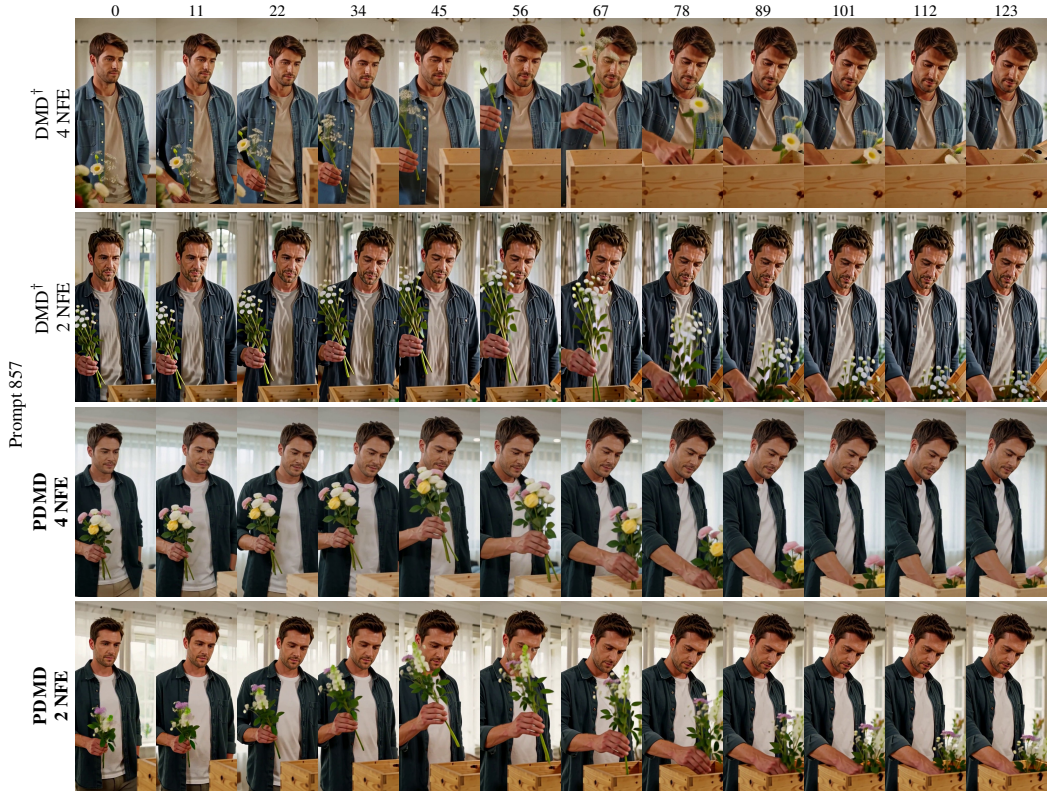

    \centering
    \setlength{\tabcolsep}{0pt}
    \renewcommand{\arraystretch}{0}
    \begin{tabular}{@{}c@{\hspace{2pt}}c@{\hspace{2pt}}l@{}}
        & & \hmheadA{1.071cm} \\[1pt]
        \hmnfeblock{857}{2.5cm}
    \end{tabular}

    \caption{\textbf{Motion comparisons of \dmd and \methodname{} at 2 NFE and 4 NFE.} DMD$^{\dagger}$ at 4 and 2 NFE and \methodname{} at 4 and 2 NFE, using the checkpoints reported in \cref{tab:h3,tab:supp-h3-nfe2}. VideoGen-Eval prompt 857 on MiniMax-H3, one row per model; twelve uniformly spaced frames of the 124-frame clip, left to right, labelled by frame index, with the same prompt and seed across rows. Against \methodname{} at 4 NFE, \methodname{} at 2 NFE shows slightly lower motion quality and less realistic object texture, but remains clearly better than the unprojected DMD$^{\dagger}$ at 2 NFE.}
    \label{fig:h3_nfe2}
\end{figure}
\endgroup

\clearpage
\begingroup
\newcommand{\nfesample}[1]{\raisebox{-0.5\height}{\includegraphics[width=0.285\linewidth,height=77pt,keepaspectratio]{figures/src/h3_nfe2_samples/#1.jpg}}}
\newcommand{\nferow}[3]{\parbox[c]{0.12\linewidth}{\centering\scriptsize #1} &
\nfesample{#2_#3_f0} & \nfesample{#2_#3_f62} & \nfesample{#2_#3_f123}\\[2pt]}
\newcommand{\nfeblock}[1]{%
\multicolumn{4}{c}{\small\textbf{Prompt #1}}\\[3pt]
& \scriptsize Frame 0 & \scriptsize Frame 62 & \scriptsize Frame 123\\[2pt]
\nferow{DMD$^{\dagger}$\\2 NFE}{dmd2}{#1}
\nferow{\methodname\\2 NFE}{pdmd2}{#1}
\nferow{\methodname\\4 NFE}{pdmd4}{#1}}

\begin{figure}[p]
\centering
\setlength{\tabcolsep}{1pt}
\renewcommand{\arraystretch}{0}
\begin{tabular}{@{}cccc@{}}
\nfeblock{736}\\[8pt]
\nfeblock{735}
\end{tabular}
\caption{\textbf{Additional two-step generation results (part 1 of 4).} DMD$^{\dagger}$ at 2 NFE, \methodname at 2 NFE, and \methodname at 4 NFE on MiniMax-H3 for VideoGen-Eval prompts 736 and 735. Each row shows three uniformly sampled frames (0, 62, 123; zero-based) from a 124-frame clip, with the original aspect ratio preserved. $^{\dagger}$ denotes our reimplementation.}
\label{fig:h3_nfe2_samples}
\end{figure}

\begin{figure}[p]
\centering
\setlength{\tabcolsep}{1pt}
\renewcommand{\arraystretch}{0}
\begin{tabular}{@{}cccc@{}}
\nfeblock{733}\\[8pt]
\nfeblock{743}
\end{tabular}
\caption{\textbf{Additional two-step generation results (part 2 of 4).} DMD$^{\dagger}$ at 2 NFE, \methodname at 2 NFE, and \methodname at 4 NFE on MiniMax-H3 for VideoGen-Eval prompts 733 and 743. Each row shows three uniformly sampled frames (0, 62, 123; zero-based) from a 124-frame clip, with the original aspect ratio preserved. $^{\dagger}$ denotes our reimplementation.}
\label{fig:h3_nfe2_samples-2}
\end{figure}

\begin{figure}[p]
\centering
\setlength{\tabcolsep}{1pt}
\renewcommand{\arraystretch}{0}
\begin{tabular}{@{}cccc@{}}
\nfeblock{744}\\[8pt]
\nfeblock{753}
\end{tabular}
\caption{\textbf{Additional two-step generation results (part 3 of 4).} DMD$^{\dagger}$ at 2 NFE, \methodname at 2 NFE, and \methodname at 4 NFE on MiniMax-H3 for VideoGen-Eval prompts 744 and 753. Each row shows three uniformly sampled frames (0, 62, 123; zero-based) from a 124-frame clip, with the original aspect ratio preserved. $^{\dagger}$ denotes our reimplementation.}
\label{fig:h3_nfe2_samples-3}
\end{figure}

\begin{figure}[p]
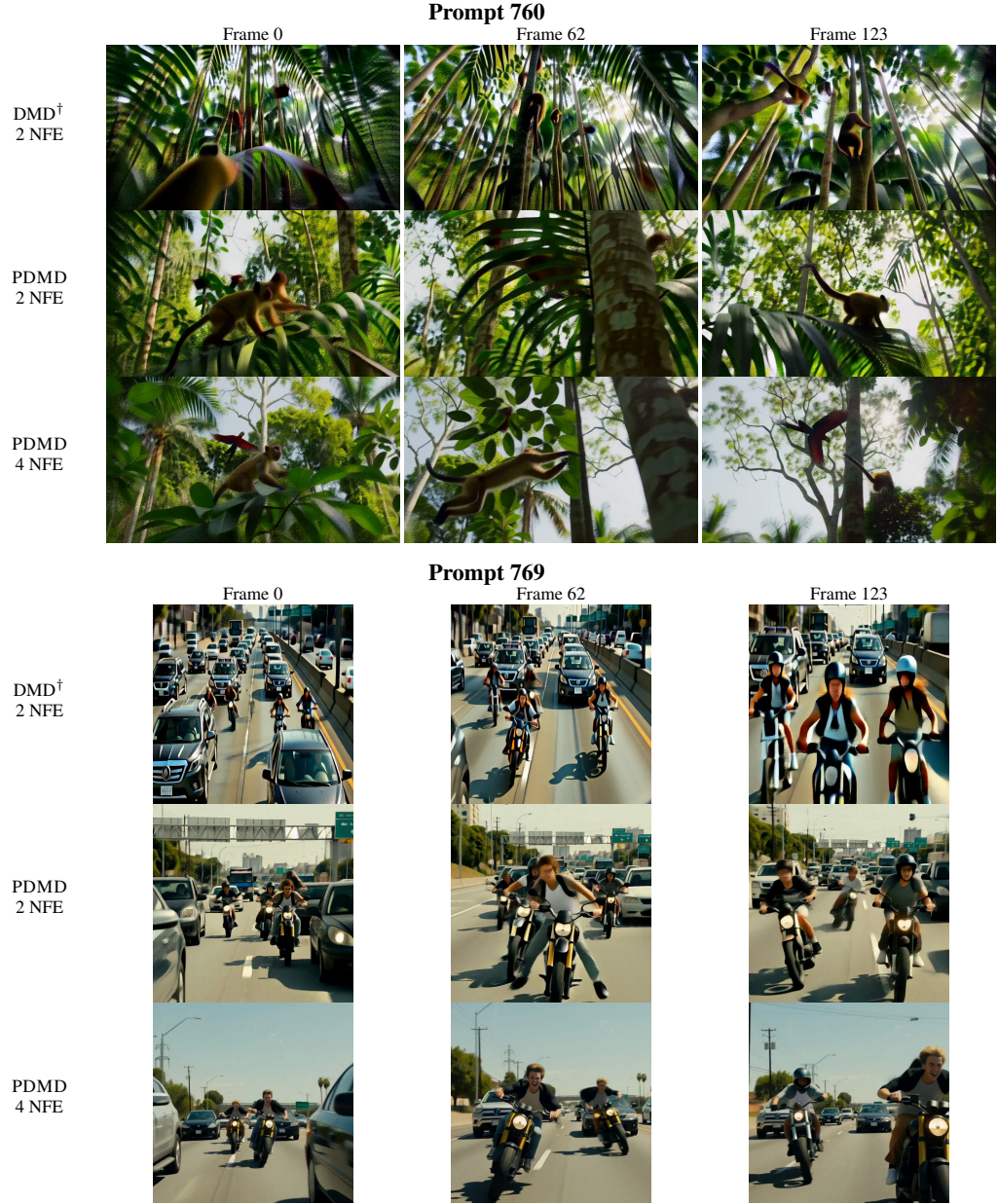

\centering
\setlength{\tabcolsep}{1pt}
\renewcommand{\arraystretch}{0}
\begin{tabular}{@{}cccc@{}}
\nfeblock{760}\\[8pt]
\nfeblock{769}
\end{tabular}
\caption{\textbf{Additional two-step generation results (part 4 of 4).} DMD$^{\dagger}$ at 2 NFE, \methodname at 2 NFE, and \methodname at 4 NFE on MiniMax-H3 for VideoGen-Eval prompts 760 and 769. Each row shows three uniformly sampled frames (0, 62, 123; zero-based) from a 124-frame clip, with the original aspect ratio preserved. $^{\dagger}$ denotes our reimplementation.}
\label{fig:h3_nfe2_samples-4}
\end{figure}

\endgroup

\clearpage

\section{Additional Visual Comparisons}
\label{sec:supp-visual-comparisons}

\subsection{The Removed Component in Video Space}
\label{sec:h3-proj-endpoints}

The norm ratio measures how much of the \dmd update the projection removes. \cref{fig:supp-proj-endpoints} shows what the removed part looks like after decoding. The images come from the \methodname{} run of \cref{sec:ablations} at iteration 500. At each checkpoint, every rank decodes the sample from its last student step three ways: the student endpoint $\bm x_s$, the \dmd target $\bm x_s - \bm d/Z$ built from the raw numerator, and the \methodname{} target $\bm x_s - \bm d^{\perp}/Z$ built from the projected one. The two targets share the normalizer $Z$ and the same student step; the projection is their only difference. Under the 125k-token packing each rank holds one sample, and \cref{fig:supp-proj-endpoints} shows twelve of the sixteen ranks of that step in rank order.

The \dmd target carries high-frequency speckle that the student endpoint does not have. The \methodname{} target keeps the same subject, colors, and frame layout, and carries visibly less of that speckle. The projection removes between $5\%$ and $30\%$ of the video update norm across the 16 samples at this step (median $16\%$); the visual comparison associates the removed component with speckle artifacts.

\begin{figure}[p]
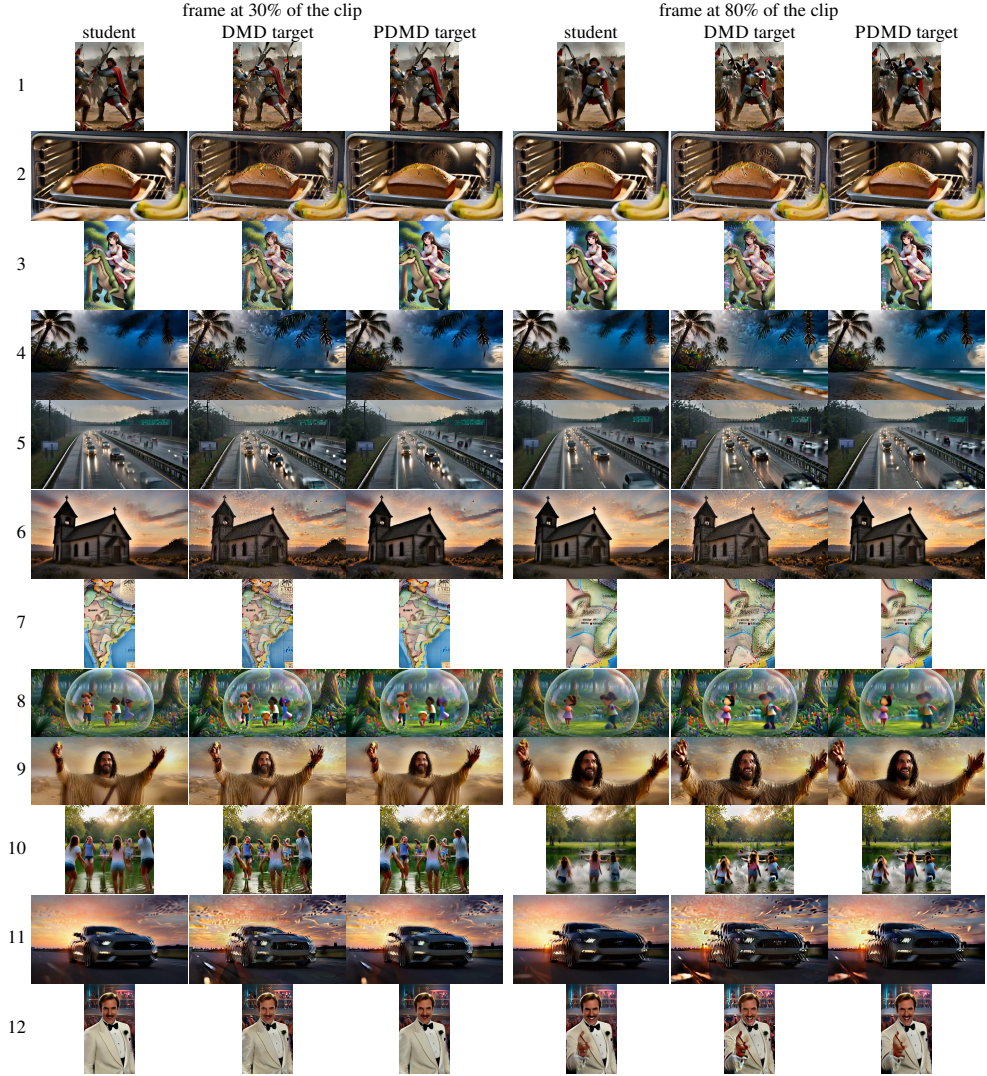

    \centering
    \setlength{\tabcolsep}{0pt}
    \renewcommand{\arraystretch}{0}
    \scriptsize
    \newcommand{\pef}[4]{\includegraphics[height=#1\linewidth]{figures/src/proj_endpoints/s500/rank#2_#3_#4.jpg}}
    \newcommand{\pel}[2]{\raisebox{\dimexpr#1\linewidth-0.5ex}{#2}}
    \scalebox{0.995}{%
    \begin{tabular}{@{}r@{\hspace{2pt}}c@{\hspace{1pt}}c@{\hspace{0.5pt}}c@{\hspace{4pt}}c@{\hspace{1pt}}c@{\hspace{0.5pt}}c@{}}
        \rule{0pt}{7pt} & \multicolumn{3}{c}{frame at 30\% of the clip} &
          \multicolumn{3}{c}{frame at 80\% of the clip} \\
        \rule{0pt}{7pt} & student & \dmd target & \methodname{} target &
          student & \dmd target & \methodname{} target \\[1pt]
        \pel{0.0420}{1} &
          \pef{0.0840}{2}{student}{30} &
          \pef{0.0840}{2}{target_before}{30} &
          \pef{0.0840}{2}{target_after}{30} &
          \pef{0.0840}{2}{student}{80} &
          \pef{0.0840}{2}{target_before}{80} &
          \pef{0.0840}{2}{target_after}{80} \\[0.5pt]
        \pel{0.0420}{2} &
          \pef{0.0840}{3}{student}{30} &
          \pef{0.0840}{3}{target_before}{30} &
          \pef{0.0840}{3}{target_after}{30} &
          \pef{0.0840}{3}{student}{80} &
          \pef{0.0840}{3}{target_before}{80} &
          \pef{0.0840}{3}{target_after}{80} \\[0.5pt]
        \pel{0.0420}{3} &
          \pef{0.0840}{4}{student}{30} &
          \pef{0.0840}{4}{target_before}{30} &
          \pef{0.0840}{4}{target_after}{30} &
          \pef{0.0840}{4}{student}{80} &
          \pef{0.0840}{4}{target_before}{80} &
          \pef{0.0840}{4}{target_after}{80} \\[0.5pt]
        \pel{0.0420}{4} &
          \pef{0.0840}{5}{student}{30} &
          \pef{0.0840}{5}{target_before}{30} &
          \pef{0.0840}{5}{target_after}{30} &
          \pef{0.0840}{5}{student}{80} &
          \pef{0.0840}{5}{target_before}{80} &
          \pef{0.0840}{5}{target_after}{80} \\[0.5pt]
        \pel{0.0420}{5} &
          \pef{0.0840}{6}{student}{30} &
          \pef{0.0840}{6}{target_before}{30} &
          \pef{0.0840}{6}{target_after}{30} &
          \pef{0.0840}{6}{student}{80} &
          \pef{0.0840}{6}{target_before}{80} &
          \pef{0.0840}{6}{target_after}{80} \\[0.5pt]
        \pel{0.0420}{6} &
          \pef{0.0840}{7}{student}{30} &
          \pef{0.0840}{7}{target_before}{30} &
          \pef{0.0840}{7}{target_after}{30} &
          \pef{0.0840}{7}{student}{80} &
          \pef{0.0840}{7}{target_before}{80} &
          \pef{0.0840}{7}{target_after}{80} \\[0.5pt]
        \pel{0.0420}{7} &
          \pef{0.0840}{8}{student}{30} &
          \pef{0.0840}{8}{target_before}{30} &
          \pef{0.0840}{8}{target_after}{30} &
          \pef{0.0840}{8}{student}{80} &
          \pef{0.0840}{8}{target_before}{80} &
          \pef{0.0840}{8}{target_after}{80} \\[0.5pt]
        \pel{0.0318}{8} &
          \pef{0.0636}{9}{student}{30} &
          \pef{0.0636}{9}{target_before}{30} &
          \pef{0.0636}{9}{target_after}{30} &
          \pef{0.0636}{9}{student}{80} &
          \pef{0.0636}{9}{target_before}{80} &
          \pef{0.0636}{9}{target_after}{80} \\[0.5pt]
        \pel{0.0318}{9} &
          \pef{0.0636}{10}{student}{30} &
          \pef{0.0636}{10}{target_before}{30} &
          \pef{0.0636}{10}{target_after}{30} &
          \pef{0.0636}{10}{student}{80} &
          \pef{0.0636}{10}{target_before}{80} &
          \pef{0.0636}{10}{target_after}{80} \\[0.5pt]
        \pel{0.0420}{10} &
          \pef{0.0840}{12}{student}{30} &
          \pef{0.0840}{12}{target_before}{30} &
          \pef{0.0840}{12}{target_after}{30} &
          \pef{0.0840}{12}{student}{80} &
          \pef{0.0840}{12}{target_before}{80} &
          \pef{0.0840}{12}{target_after}{80} \\[0.5pt]
        \pel{0.0420}{11} &
          \pef{0.0840}{13}{student}{30} &
          \pef{0.0840}{13}{target_before}{30} &
          \pef{0.0840}{13}{target_after}{30} &
          \pef{0.0840}{13}{student}{80} &
          \pef{0.0840}{13}{target_before}{80} &
          \pef{0.0840}{13}{target_after}{80} \\[0.5pt]
        \pel{0.0420}{12} &
          \pef{0.0840}{14}{student}{30} &
          \pef{0.0840}{14}{target_before}{30} &
          \pef{0.0840}{14}{target_after}{30} &
          \pef{0.0840}{14}{student}{80} &
          \pef{0.0840}{14}{target_before}{80} &
          \pef{0.0840}{14}{target_after}{80} \\
    \end{tabular}}
    \caption{\textbf{One student step, before and after the projection.} Results are from the training iteration-500 of the 4-NFE MiniMax-H3 \methodname{} run in \cref{sec:ablations}. Each row is one of the data-parallel ranks of the same student step, decoded three ways: the student endpoint $\bm x_s$, the \dmd target $\bm x_s - \bm d/Z$ built from the raw numerator, and the \methodname{} target $\bm x_s - \bm d^{\perp}/Z$ built from the projected one. Under the 125k-token packing each rank holds one sample; the rows are twelve of the sixteen ranks of that step, in rank order. Left: the frame at 30\% of the clip.  Right: the frame at 80\%. The \dmd target carries high-frequency speckle, but the \methodname{} target retains the same subject, colors, and layout while exhibiting less speckle. This comparison associates the removed component with visible artifacts.}
    \label{fig:supp-proj-endpoints}
\end{figure}

\subsection{Saturation Metrics}
\label{sec:supp-saturation}

\Cref{tab:supp-sat} lists two color statistics of the models in \cref{tab:wan13b,tab:h3}. Chroma is the mean CIELAB $C^*=\sqrt{a^{*2}+b^{*2}}$ over pixels. Saturation is the mean $C^*$ divided by the mean lightness $L^*$. Both are averaged over frames and clips.

\methodname{} has lower mean saturation than \dmd on both backbones. On Wan2.1, DMD$^{\dagger}$ sits at 0.709 saturation against 0.587 for the 50-step teacher, and \methodname{} at 0.476; on MiniMax-H3, DMD$^{\dagger}$ and \dmdtwodag sit at 0.460 and 0.520 against 0.405, and \methodname{} at 0.357. \rcm and \anyflow sit above the teacher on both color statistics for both backbones. These aggregate statistics complement the visual comparisons and user studies; lower saturation alone does not establish better perceptual quality.

\begin{table}[htbp]
    \centering
    \caption{\textbf{Color statistics of the models in \cref{tab:wan13b,tab:h3}.} Wan2.1 values are means over the 4{,}720 clips of the VBench protocol in \cref{sec:supp-wan-protocol}; MiniMax-H3 values are means over the 387 VideoGen-Eval clips of \cref{sec:vgeneval387}.}
    \label{tab:supp-sat}
    \small
    \begin{tabular}{@{}lcc@{}}
        \toprule
        Method & Saturation ($C^*/L^*$) & Chroma $C^*$ \\
        \midrule
        \multicolumn{3}{@{}l}{\textit{Wan2.1-T2V-1.3B}} \\
        \deemph{Wan2.1-T2V-1.3B, 50 steps} & \deemph{0.587} & \deemph{21.5} \\
        Wan2.1-T2V-1.3B, 4 steps & 1.233 & 27.6 \\
        \rcm & 0.633 & 23.6 \\
        \anyflow & 0.607 & 25.9 \\
        DMD$^{\dagger}$ & 0.709 & 23.5 \\
        \dmdtwodag & 0.437 & 17.9 \\
        \textbf{\methodname{} (ours)} & 0.476 & 17.9 \\
        \midrule
        \multicolumn{3}{@{}l}{\textit{MiniMax-H3-33B}} \\
        \deemph{MiniMax-H3-33B, 50 steps} & \deemph{0.405} & \deemph{15.0} \\
        MiniMax-H3-33B, 4 steps & 0.669 & 17.7 \\
        H3 Turbo LoRA & 0.548 & 17.2 \\
        rCM$^{\dagger}$ & 0.440 & 15.2 \\
        AnyFlow$^{\dagger}$ & 0.437 & 15.3 \\
        DMD$^{\dagger}$ & 0.460 & 16.6 \\
        \dmdtwodag & 0.520 & 18.3 \\
        \textbf{\methodname{} (ours)} & 0.357 & 14.0 \\
        \bottomrule
    \end{tabular}
\end{table}

\subsection{Motion Comparisons}
\label{sec:supp-motion}

Beyond the dynamic-degree scores in \cref{tab:wan13b,tab:h3}, this section shows more examples of the motion in our distilled results. \Cref{fig:supp-motion-wan119} shows Wan2.1 prompt 119 with eighteen uniformly spaced frames per method, one column per method. \Cref{fig:supp-motion-h3-857,fig:supp-motion-h3-895} show MiniMax-H3 prompts 857 and 895 with twelve and ten uniformly spaced frames per method, one row per method. \Cref{fig:h3_nfe2} adds the two-step runs of \cref{sec:h3-two-step} on prompt 857, below the four-step rows of \dmd$^{\dagger}$ and \methodname{}.

The three scenes show three recurring failures of the baselines. Motion is lost: in 895 the distilled baselines except AnyFlow$^{\dagger}$ are almost static while \methodname{} follows the teacher. Motion is misread: in 119 every baseline except ADV turns the rider's upper body around. Motion is blurred: Turbo LoRA, rCM$^{\dagger}$ and AnyFlow$^{\dagger}$ blur the slowly moving bouquet in 857 into translucent ghosting. Several baselines also carry a color cast (119) or a smeared, shifted tone (DMD$^{\dagger}$ and \dmdtwodag in 857). \methodname{} avoids these visible failures in the three scenes.

\newcommand{\mww}{\dimexpr(\textwidth-0.16in-9pt)/8\relax}
\newcommand{\mwfh}{\dimexpr0.577\mww\relax}%
\newcommand{\mwl}[1]{\rule{0pt}{\mwfh}\raisebox{\dimexpr0.5\mwfh-0.5\height\relax}[0pt][0pt]{\rotatebox{90}{\tiny #1}}\\}
\newcommand{\mwlabels}{\begin{tabular}[b]{@{}c@{}}%
    \mwl{0}\mwl{5}\mwl{9}\mwl{14}\mwl{19}\mwl{24}\mwl{28}\mwl{33}\mwl{38}\mwl{42}\mwl{47}\mwl{52}\mwl{56}\mwl{61}\mwl{66}\mwl{71}\mwl{75}\mwl{80}\end{tabular}}
\newcommand{\mwh}[3]{\parbox[b]{\mww}{\centering\scriptsize #1\\[1pt]%
    \tiny #2\\[1pt]\scriptsize #3}}
\newcommand{\mw}[2]{\includegraphics[width=\mww]{figures/src/motion/wan_#1_#2_strip.jpg}}
\newcommand{\mwhead}{%
    & \mwh{Wan2.1}{\resizebox{0.92\mww}{!}{\citep{wan2025wan}}}{$50{\times}2$ NFE}
    & \mwh{Wan2.1}{\resizebox{0.92\mww}{!}{\citep{wan2025wan}}}{$4{\times}2$ NFE}
    & \mwh{rCM}{\citep{zheng2026rcm}}{4 NFE}
    & \mwh{ADV}{\citep{you2026adaptive}}{4 NFE}
    & \mwh{AnyFlow}{\citep{gu2026anyflow}}{4 NFE}
    & \mwh{DMD$^{\dagger}$}{\citep{yin2024onestep}}{4 NFE}
    & \mwh{\dmdtwodag}{\citep{yin2024improved}}{4 NFE}
    & \mwh{\textbf{\methodname{}}}{\textbf{(Ours)}}{\textbf{4 NFE}} \\[2pt]}
\newcommand{\wanmotionfig}[3]{%
\begin{figure}[p]
    \centering
    \setlength{\tabcolsep}{0pt}
    \renewcommand{\arraystretch}{0}
    \begin{tabular}{@{}r@{\hspace{2pt}}c@{\hspace{1pt}}c@{\hspace{1pt}}c@{\hspace{1pt}}c@{\hspace{1pt}}c@{\hspace{1pt}}c@{\hspace{1pt}}c@{\hspace{1pt}}c@{}}
        \mwhead
        \makebox[0.16in][c]{\mwlabels} & \mw{t50}{#1} & \mw{t4}{#1} & \mw{rcm}{#1} & \mw{adv}{#1} & \mw{af}{#1} & \mw{dmd}{#1} & \mw{dmd2}{#1} & \mw{pdmd}{#1} \\
    \end{tabular}
    \caption{\textbf{Motion on four-step Wan2.1-T2V-1.3B, VBench prompt #1.} Each column is one method from \cref{tab:wan13b}; the eighteen rows are uniformly spaced frames of the 81-frame clip, top to bottom, labelled by frame index, with the same prompt and seed across methods. #3}
    \label{#2}
\end{figure}}

\wanmotionfig{119}{fig:supp-motion-wan119}{Besides the visible color cast
    of the other methods, every baseline except ADV misreads the rider's
    motion and turns his upper body by 180 degrees.}

\newcommand{\hmn}[4]{\rotatebox{90}{\parbox[c]{#3}{\centering #4 #1\\[0pt]#4 #2}}}
\newcommand{\hmf}[3]{\includegraphics[height=#3]{figures/src/motion/h3_#1_#2_strip.jpg}}
\newcommand{\hmi}[2]{\makebox[#2][c]{\tiny #1}}
\newcommand{\hmrowA}[6]{%
    \hmn{#4}{#5}{#3}{#6} & \hmf{#1}{#2}{#3} \\[2pt]}
\newcommand{\hmrowB}[6]{%
    \hmn{#4}{#5}{#3}{#6} & \hmf{#1}{#2}{#3} \\[2pt]}
\newcommand{\hmheadA}[1]{\hmi{0}{#1}\hmi{11}{#1}\hmi{22}{#1}\hmi{34}{#1}\hmi{45}{#1}\hmi{56}{#1}\hmi{67}{#1}\hmi{78}{#1}\hmi{89}{#1}\hmi{101}{#1}\hmi{112}{#1}\hmi{123}{#1}}
\newcommand{\hmheadB}[1]{\hmi{0}{#1}\hmi{14}{#1}\hmi{27}{#1}\hmi{41}{#1}\hmi{55}{#1}\hmi{68}{#1}\hmi{82}{#1}\hmi{96}{#1}\hmi{109}{#1}\hmi{123}{#1}}
\newcommand{\hmotionbody}[9]{%
\begin{figure}[p]
    \centering
    \setlength{\tabcolsep}{0pt}
    \renewcommand{\arraystretch}{0}
    \begin{tabular}{@{}c@{\hspace{2pt}}l@{}}
        & #2{#6} \\[1pt]
        #1{t50}{#3}{#5}{MiniMax-H3}{50 NFE}{#8}
        #1{t4}{#3}{#5}{MiniMax-H3}{4 NFE}{#8}
        #1{lora}{#3}{#5}{Turbo LoRA}{4 NFE}{#8}
        #1{rcm}{#3}{#5}{rCM$^{\dagger}$}{4 NFE}{#8}
        #1{af}{#3}{#5}{AnyFlow$^{\dagger}$}{4 NFE}{#8}
        #1{dmd}{#3}{#5}{DMD$^{\dagger}$}{4 NFE}{#8}
        #1{dmd2}{#3}{#5}{\dmdtwodag}{4 NFE}{#8}
        #1{pdmd}{#3}{#5}{\textbf{\methodname{}}}{\textbf{(ours), 4 NFE}}{#8}
    \end{tabular}
    \caption{\textbf{Motion on four-step MiniMax-H3, VideoGen-Eval prompt #3.} Each row is one method from \cref{tab:h3}; the #7 frames are uniformly spaced over the 124-frame clip, left to right, labelled by frame index, with the same prompt and seed across methods. #9}
    \label{#4}
\end{figure}}
\newcommand{\hmotionfigA}[6]{\hmotionbody{\hmrowA}{\hmheadA}{#1}{#2}{#3}{#4}{twelve}{#5}{#6}}
\newcommand{\hmotionfigB}[6]{\hmotionbody{\hmrowB}{\hmheadB}{#1}{#2}{#3}{#4}{ten}{#5}{#6}}

\hmotionfigA{857}{fig:supp-motion-h3-857}{2.47cm}{1.059cm}{\scriptsize}{%
    Turbo LoRA, rCM$^{\dagger}$ and AnyFlow$^{\dagger}$ blur the slowly moving bouquet into translucent ghosting.
    DMD$^{\dagger}$ and \dmdtwodag exhibit shifted colors and smeared textures.}
\hmotionfigB{895}{fig:supp-motion-h3-895}{1.32cm}{1.32cm}{\tiny}{Except for
    AnyFlow$^{\dagger}$, the distilled baselines render an almost completely static
    video. \methodname{} shows the baseball bat entering the frame, similar to
    the teacher.}

\newcommand{\hmpl}[2]{\raisebox{\dimexpr2\dimexpr#2\relax+3pt-0.5\height\relax}[0pt][0pt]{\rotatebox{90}{\scriptsize Prompt #1}}}
\newcommand{\hmnferow}[6]{%
    #6 & \hmn{#4}{#5}{#3}{\scriptsize} & \hmf{#1}{#2}{#3} \\[2pt]}
\newcommand{\hmnfeblock}[2]{%
    \hmnferow{dmd}{#1}{#2}{DMD$^{\dagger}$}{4 NFE}{}
    \hmnferow{dmd2nfe}{#1}{#2}{DMD$^{\dagger}$}{2 NFE}{}
    \hmnferow{pdmd}{#1}{#2}{\textbf{\methodname{}}}{\textbf{4 NFE}}{}
    \hmnferow{pdmd2nfe}{#1}{#2}{\textbf{\methodname{}}}{\textbf{2 NFE}}{\hmpl{#1}{#2}}}

\clearpage

\subsection{Additional Samples}
\label{sec:supp-h3-qualitative}

\Cref{fig:wan_grid_supp} shows more prompts for the methods of \cref{tab:wan13b} on Wan2.1-T2V-1.3B. \Cref{fig:h3_grid} compares the eight models in \cref{tab:h3} on eight of the 387 VideoGen-Eval prompts, with two frames from every clip over two pages; \cref{fig:degradation} instead compares the methods over training.

\newcommand{\swgl}[1]{\makebox[0.16in][c]{\raisebox{-0.5\height}[0pt][0pt]{\rotatebox{90}{\scriptsize Prompt #1}}}}
\newcommand{\swgw}{\dimexpr(\textwidth-0.16in-9pt)/8\relax}
\newcommand{\swgh}[3]{\parbox[b]{\swgw}{\centering\scriptsize #1\\[1pt]%
    \tiny #2\\[1pt]\scriptsize #3}}
\newcommand{\swg}[1]{\includegraphics[width=\swgw]{figures/src/wan_grid/#1.jpg}}
\newcommand{\swghead}{%
    & \swgh{Wan2.1}{\citep{wan2025wan}}{$50{\times}2$ NFE}
    & \swgh{Wan2.1}{\citep{wan2025wan}}{$4{\times}2$ NFE}
    & \swgh{rCM}{\citep{zheng2026rcm}}{4 NFE}
        & \swgh{ADV}{\citep{you2026adaptive}}{4 NFE}
    & \swgh{AnyFlow}{\citep{gu2026anyflow}}{4 NFE}
    & \swgh{DMD$^{\dagger}$}{\citep{yin2024onestep}}{4 NFE}
    & \swgh{\dmdtwodag}{\citep{yin2024improved}}{4 NFE}
    & \swgh{\textbf{\methodname{}}}{\textbf{(Ours)}}{\textbf{4 NFE}}}
\newcommand{\swgsetup}{%
    \setlength{\tabcolsep}{0pt}%
    \renewcommand{\arraystretch}{0}}
\newcommand{\swgblock}[2]{%
    \swgl{#1} & \swg{t50_#1_#2} & \swg{t4_#1_#2} & \swg{rcm_#1_#2} & \swg{adv_#1_#2} & \swg{af_#1_#2} & \swg{dmd_#1_#2} & \swg{dmd2_#1_#2} & \swg{pdmd_#1_#2} \\
    & \swg{t50_#1_80} & \swg{t4_#1_80} & \swg{rcm_#1_80} & \swg{adv_#1_80} & \swg{af_#1_80} & \swg{dmd_#1_80} & \swg{dmd2_#1_80} & \swg{pdmd_#1_80} \\[2pt]}

\begin{figure}[!t]
    \centering
    \swgsetup
    \begin{tabular}{@{}r@{\hspace{2pt}}c@{\hspace{1pt}}c@{\hspace{1pt}}c@{\hspace{1pt}}c@{\hspace{1pt}}c@{\hspace{1pt}}c@{\hspace{1pt}}c@{\hspace{1pt}}c@{}}
        \swghead \\[2pt]
        \swgblock{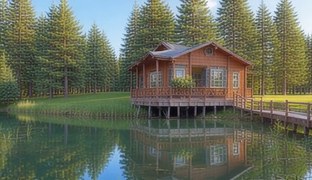}{30}
        \swgblock{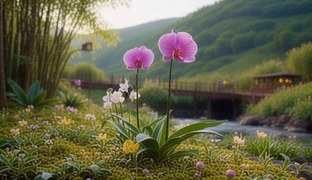}{30}
        \swgblock{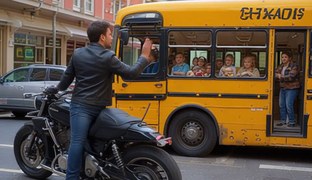}{30}
        \swgblock{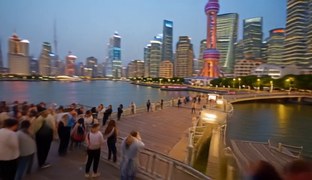}{30}
        \swgblock{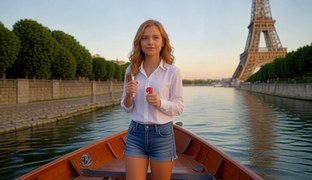}{f7}
        \swgblock{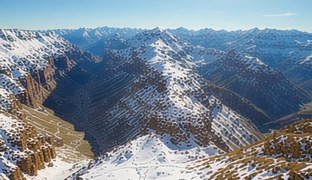}{30}
    \end{tabular}
    \caption{\textbf{Additional qualitative results for four-step Wan2.1-T2V-1.3B.} Each block shows an early and a later frame from one video under the same layout, methods, protocol, and seed as \cref{fig:wan_grid_main}; $^{\dagger}$ marks our implementations (\cref{sec:supp-wan-protocol}).}
    \label{fig:wan_grid_supp}
\end{figure}

\begin{figure}[p]
    \ContinuedFloat
    \centering
    \swgsetup
    \begin{tabular}{@{}r@{\hspace{2pt}}c@{\hspace{1pt}}c@{\hspace{1pt}}c@{\hspace{1pt}}c@{\hspace{1pt}}c@{\hspace{1pt}}c@{\hspace{1pt}}c@{\hspace{1pt}}c@{}}
        \swghead \\[2pt]
        \swgblock{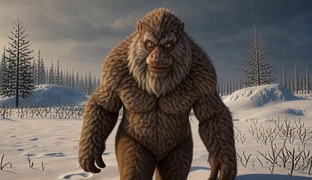}{30}
        \swgblock{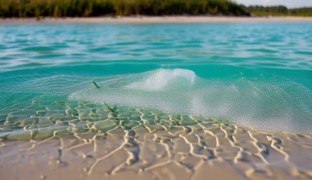}{30}
        \swgblock{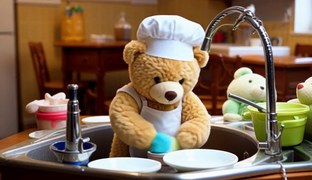}{30}
        \swgblock{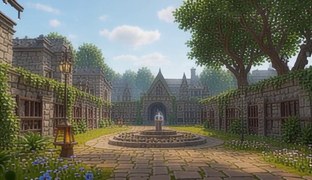}{30}
        \swgblock{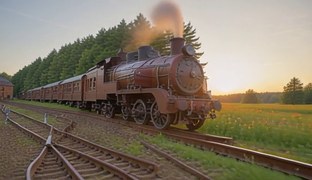}{30}
        \swgblock{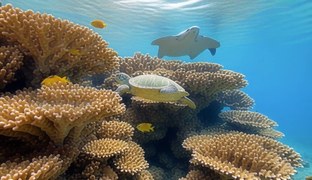}{30}
    \end{tabular}
    \caption{\textbf{Additional qualitative results for four-step Wan2.1-T2V-1.3B (continued).} The remaining prompts use the same methods, frame positions, protocol, and seed as the preceding page.}
\end{figure}

\newcommand{\hgl}[1]{\makebox[0.16in][c]{\raisebox{-0.5\height}[0pt][0pt]{\rotatebox{90}{\tiny Prompt #1}}}}
\newcommand{\hgw}{\dimexpr(\textwidth-0.16in-9pt)/8\relax}
\newcommand{\hgh}[3]{\parbox[b]{\hgw}{\centering\scriptsize #1\\[1pt]%
    \tiny #2\\[1pt]\scriptsize #3}}
\newcommand{\hg}[1]{\includegraphics[width=\hgw]{figures/src/h3_grid/#1.jpg}}
\newcommand{\hghead}{%
    & \hgh{MiniMax-H3}{\citep{minimaxh3}}{50 NFE}
    & \hgh{MiniMax-H3}{\citep{minimaxh3}}{4 NFE}
    & \hgh{H3 Turbo LoRA}{\citep{larry2026h3turbo}}{4 NFE}
    & \hgh{\rcm$^{\dagger}$}{\citep{zheng2026rcm}}{4 NFE}
    & \hgh{\anyflow$^{\dagger}$}{\citep{gu2026anyflow}}{4 NFE}
    & \hgh{DMD$^{\dagger}$}{\citep{yin2024onestep}}{4 NFE}
    & \hgh{\dmdtwodag}{\citep{yin2024improved}}{4 NFE}
    & \hgh{\textbf{PDMD}}{\textbf{(Ours)}}{\textbf{4 NFE}}}
\newcommand{\hgsetup}{%
    \setlength{\tabcolsep}{0pt}%
    \renewcommand{\arraystretch}{0}}

\begin{figure}[p]
    \centering
    \hgsetup
    \begin{tabular}{@{}r@{\hspace{2pt}}c@{\hspace{1pt}}c@{\hspace{1pt}}c@{\hspace{1pt}}c@{\hspace{1pt}}c@{\hspace{1pt}}c@{\hspace{1pt}}c@{\hspace{1pt}}c@{}}
        \hghead \\[2pt]
        \hgl{837} & \hg{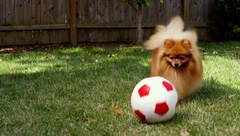} & \hg{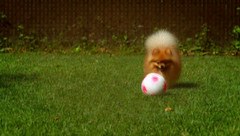} & \hg{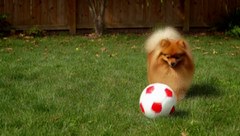} & \hg{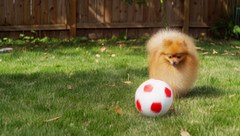} & \hg{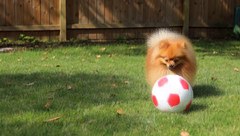} & \hg{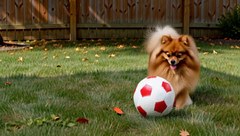} & \hg{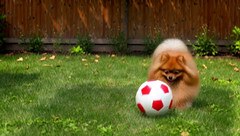} & \hg{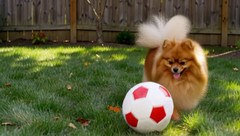} \\
        & \hg{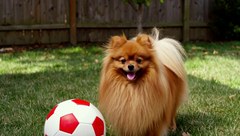} & \hg{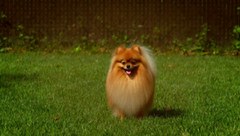} & \hg{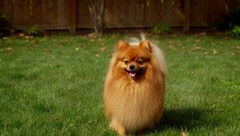} & \hg{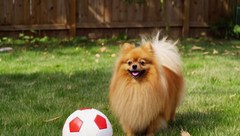} & \hg{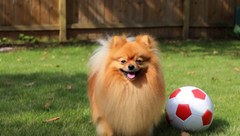} & \hg{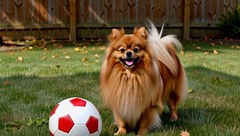} & \hg{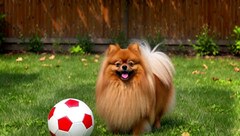} & \hg{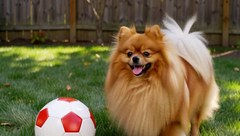} \\[2pt]
        \hgl{857} & \hg{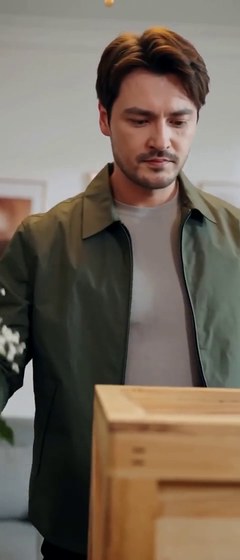} & \hg{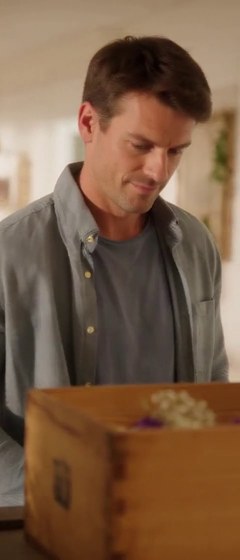} & \hg{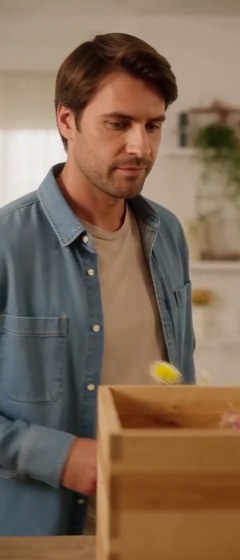} & \hg{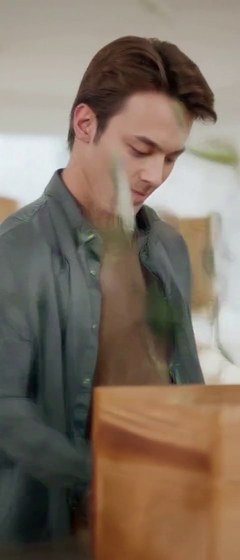} & \hg{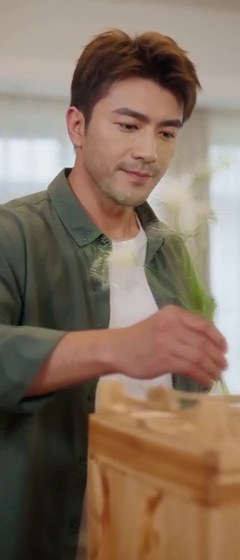} & \hg{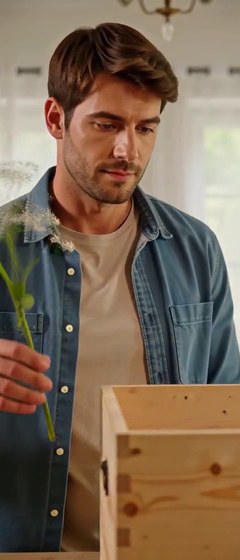} & \hg{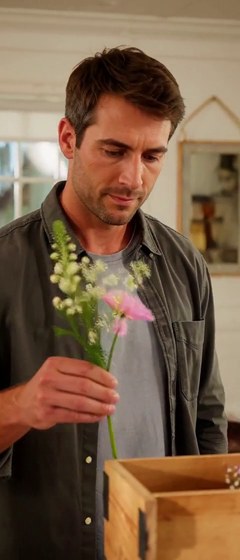} & \hg{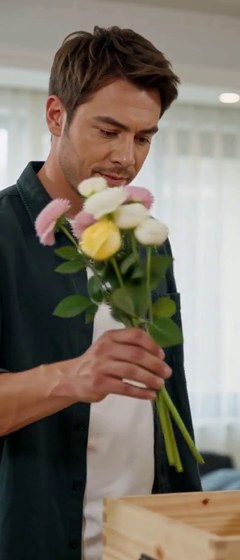} \\
        & \hg{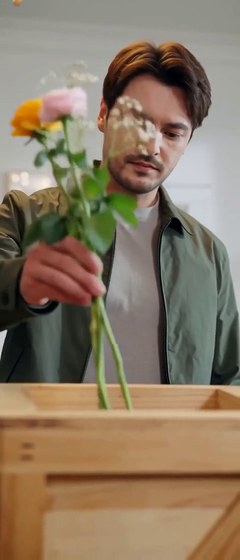} & \hg{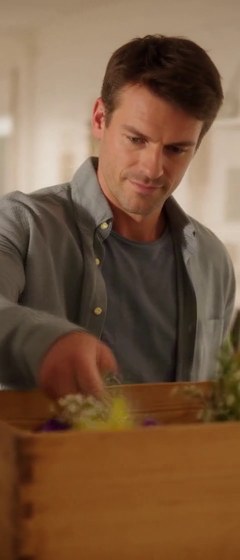} & \hg{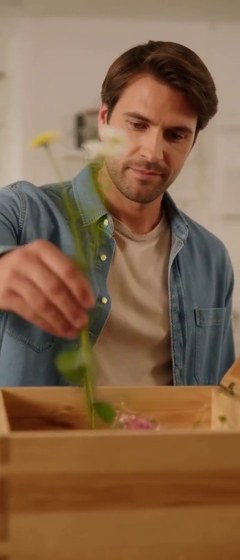} & \hg{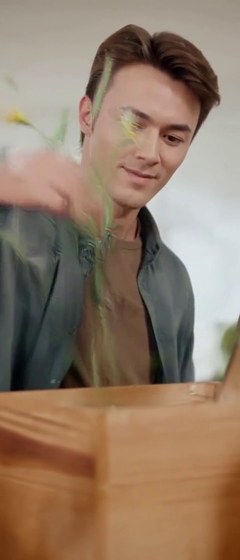} & \hg{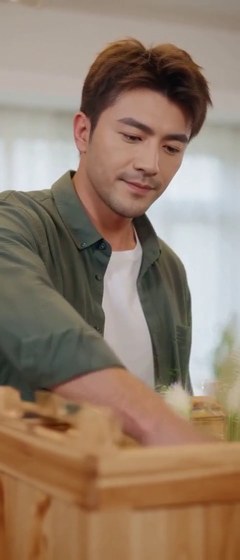} & \hg{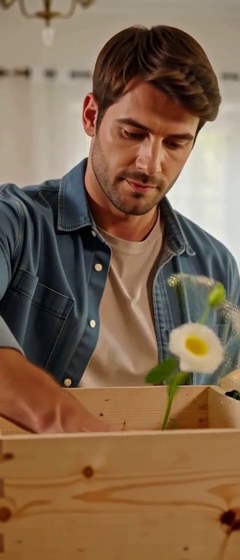} & \hg{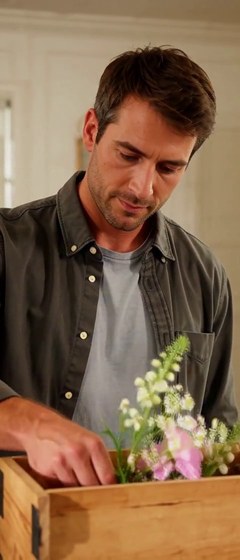} & \hg{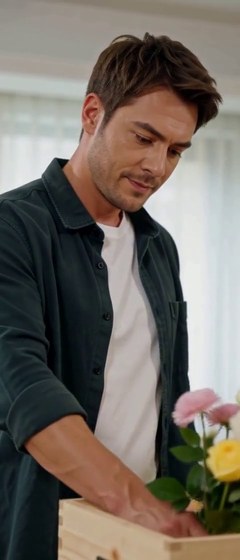} \\[2pt]
        \hgl{860} & \hg{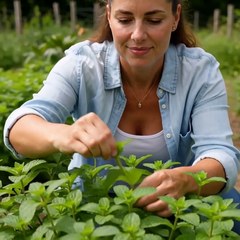} & \hg{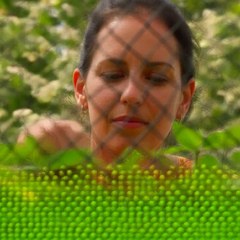} & \hg{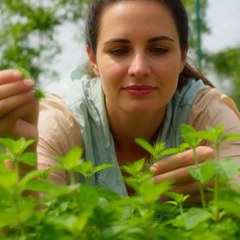} & \hg{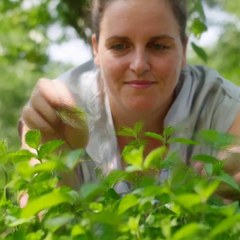} & \hg{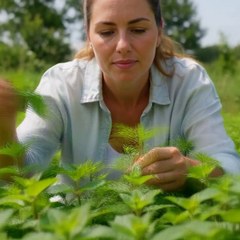} & \hg{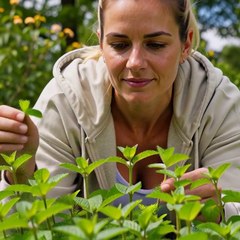} & \hg{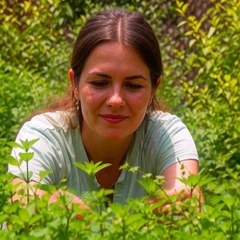} & \hg{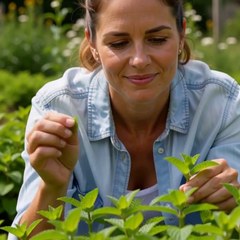} \\
        & \hg{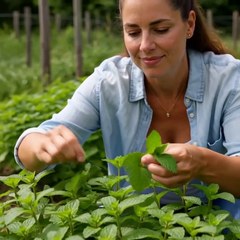} & \hg{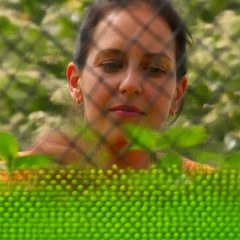} & \hg{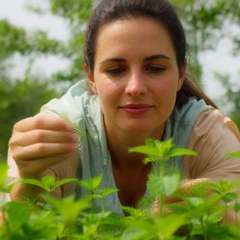} & \hg{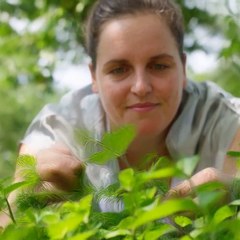} & \hg{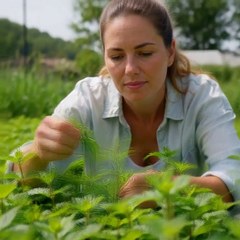} & \hg{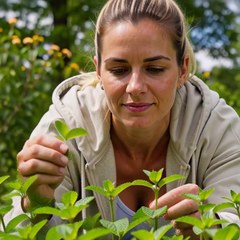} & \hg{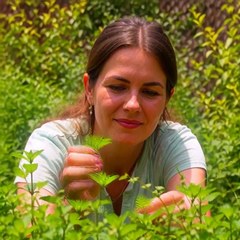} & \hg{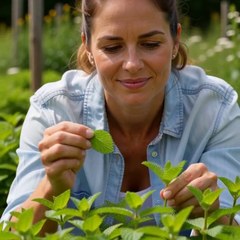} \\[2pt]
        \hgl{911} & \hg{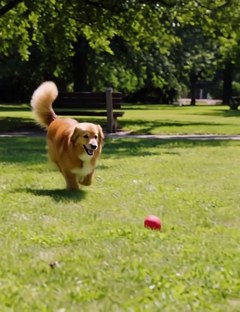} & \hg{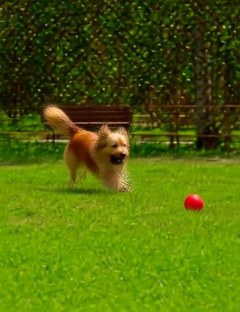} & \hg{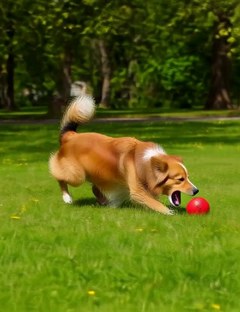} & \hg{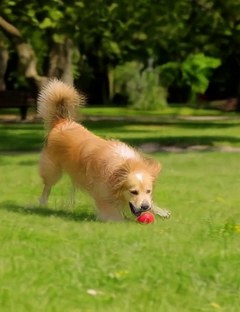} & \hg{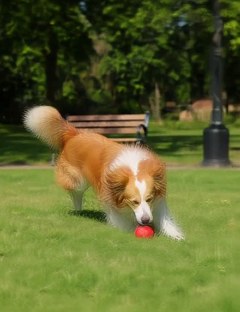} & \hg{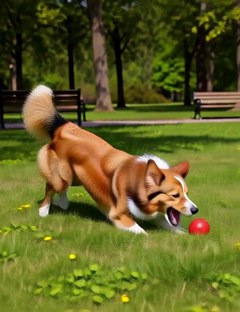} & \hg{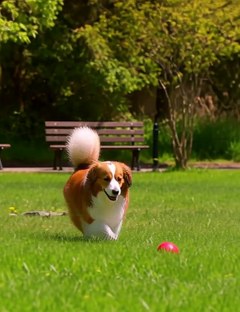} & \hg{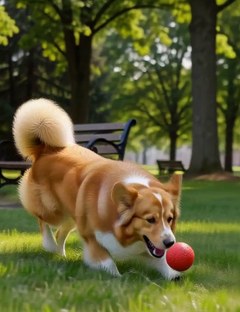} \\
        & \hg{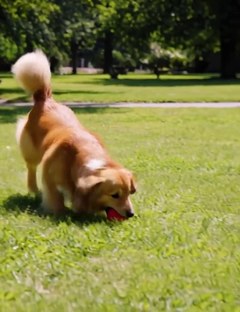} & \hg{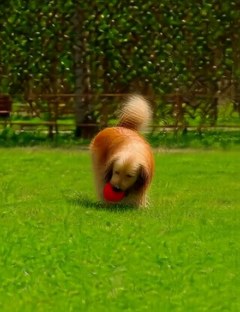} & \hg{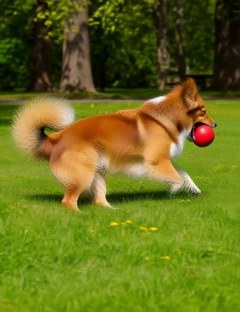} & \hg{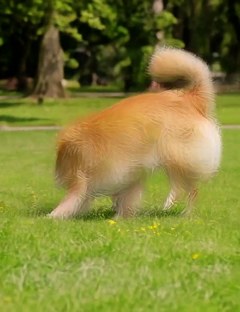} & \hg{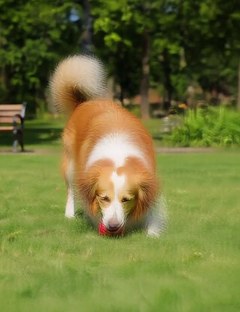} & \hg{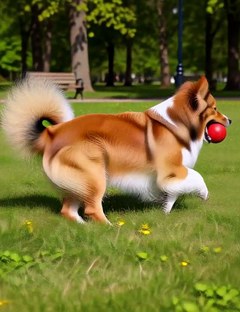} & \hg{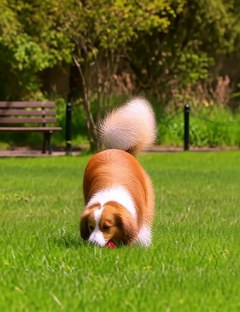} & \hg{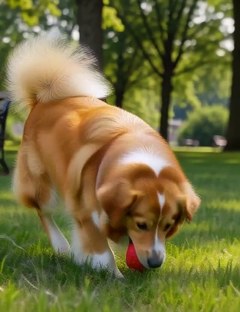} \\[2pt]
    \end{tabular}
    \caption{\textbf{Qualitative results on four-step MiniMax-H3 (part 1 of 2).} The eight models of \cref{tab:h3} on the first four of the 8 prompts, two frames of every clip per prompt with the earlier frame above. Every clip is 124 frames, and rows keep the benchmark's own aspect ratio, so their heights differ. $^{\dagger}$ marks our reimplementations.}
    \label{fig:h3_grid}
\end{figure}

\begin{figure}[p]
    \centering
    \hgsetup
    \begin{tabular}{@{}r@{\hspace{2pt}}c@{\hspace{1pt}}c@{\hspace{1pt}}c@{\hspace{1pt}}c@{\hspace{1pt}}c@{\hspace{1pt}}c@{\hspace{1pt}}c@{\hspace{1pt}}c@{}}
        \hghead \\[2pt]
        \hgl{912} & \hg{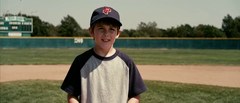} & \hg{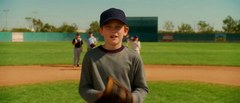} & \hg{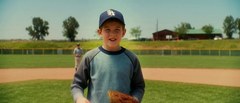} & \hg{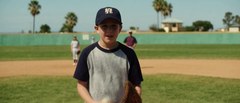} & \hg{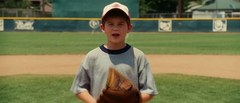} & \hg{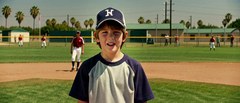} & \hg{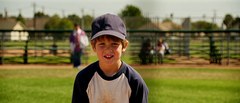} & \hg{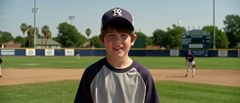} \\
        & \hg{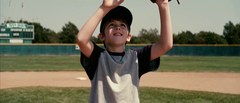} & \hg{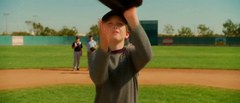} & \hg{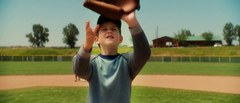} & \hg{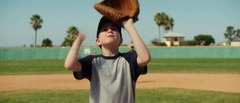} & \hg{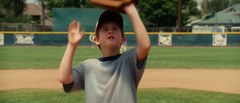} & \hg{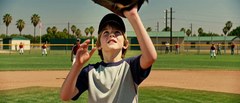} & \hg{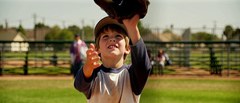} & \hg{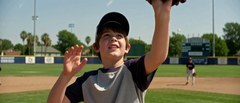} \\[2pt]
        \hgl{916} & \hg{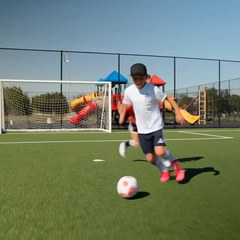} & \hg{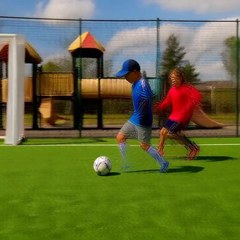} & \hg{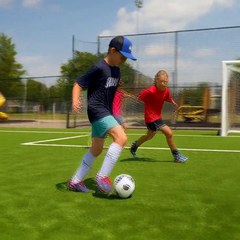} & \hg{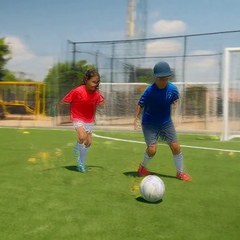} & \hg{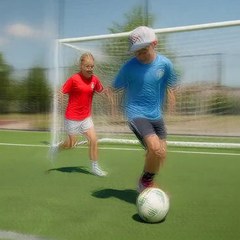} & \hg{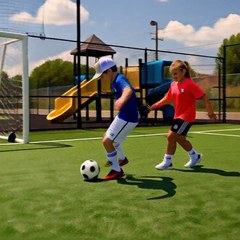} & \hg{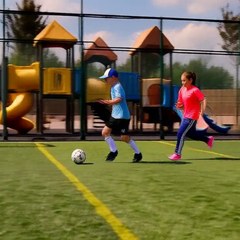} & \hg{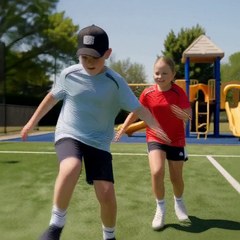} \\
        & \hg{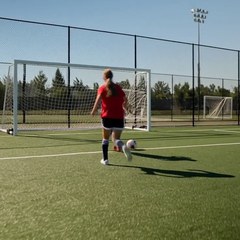} & \hg{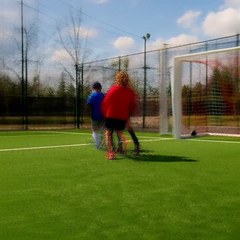} & \hg{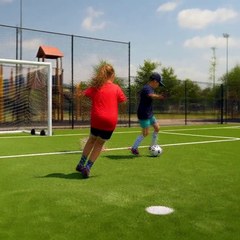} & \hg{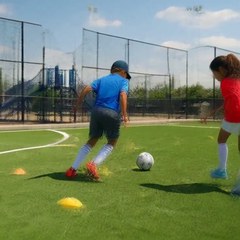} & \hg{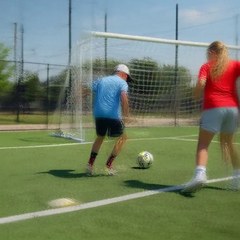} & \hg{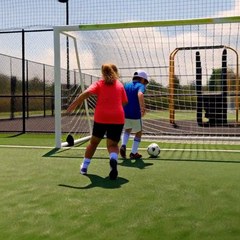} & \hg{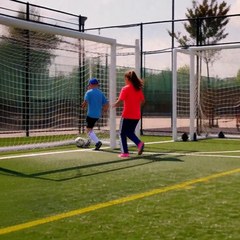} & \hg{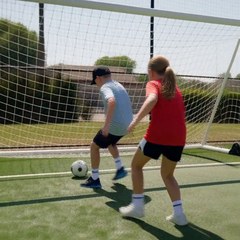} \\[2pt]
        \hgl{968} & \hg{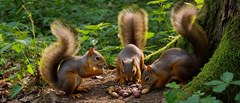} & \hg{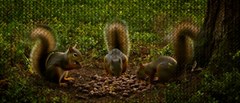} & \hg{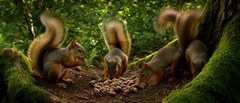} & \hg{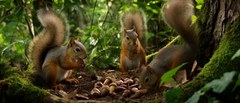} & \hg{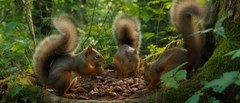} & \hg{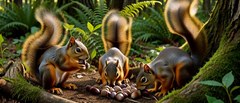} & \hg{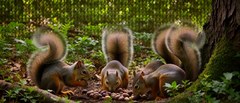} & \hg{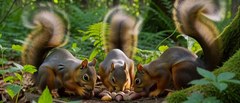} \\
        & \hg{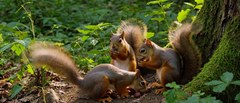} & \hg{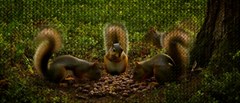} & \hg{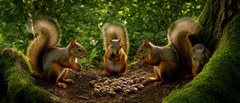} & \hg{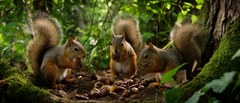} & \hg{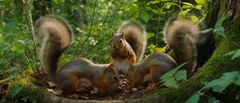} & \hg{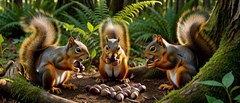} & \hg{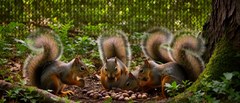} & \hg{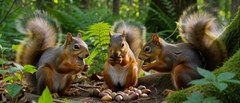} \\[2pt]
        \hgl{1010} & \hg{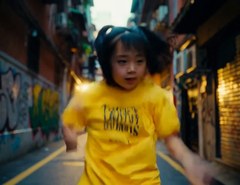} & \hg{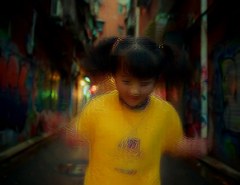} & \hg{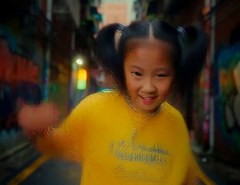} & \hg{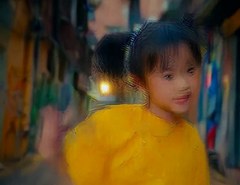} & \hg{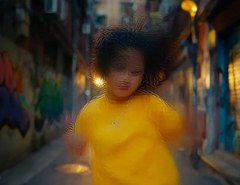} & \hg{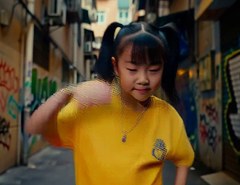} & \hg{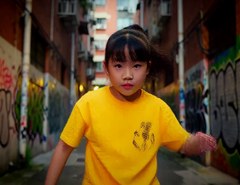} & \hg{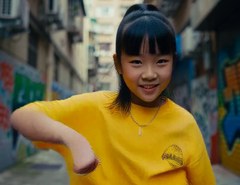} \\
        & \hg{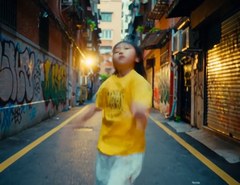} & \hg{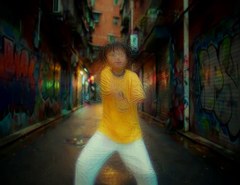} & \hg{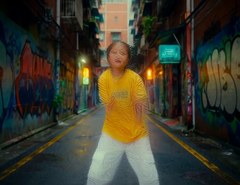} & \hg{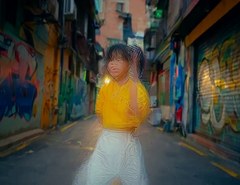} & \hg{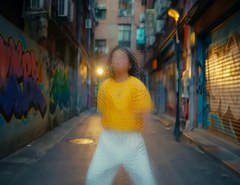} & \hg{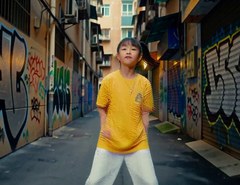} & \hg{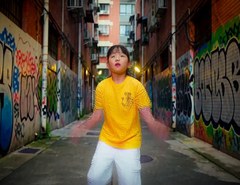} & \hg{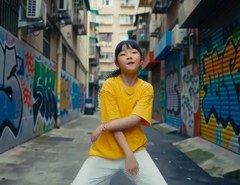} \\[2pt]
    \end{tabular}
    \caption{\textbf{Qualitative results on four-step MiniMax-H3 (part 2 of 2).} Four more prompts, same models, same layout and the same two-frame rule as \cref{fig:h3_grid}.}
    \label{fig:h3_grid-b}
\end{figure}

\clearpage

\section{Limitations and Future Work}
\label{sec:failure-cases}

\label{sec:failure-one-step}

\methodname{} does not yet yield satisfactory single-step video generation in our experiments. At 1 NFE, samples remain smeared and fall below the quality of the four-step results (\cref{fig:limitation}). Improving single-step quality may require an additional objective, such as a discriminator loss; we leave this investigation to future work. As in graphics, where a biased estimator with smaller variance often replaces an unbiased one, using a biased critic error estimate with smaller deviation is another possible improvement (\cref{par:supp-biased-critic-error}). Whether the residual direction remains the right one to be removed when a single evaluation must cover the whole video generation is left to future work.

The theory guarantees conditional error removal, not convergence of the coupled student--critic optimization. A nonvanishing error-removal fraction requires the critic error to remain appreciable relative to conditional endpoint variance, and signal retention requires weak alignment with the residual.
If the useful signal aligns with the residual, the projection can remove it as well; when the critic is exact, it can still discard useful signal (\cref{eq:supp-projected-correction-error}).
Although these assumptions cannot be verified directly on the video models, the decoded \methodname{} target carries visibly less high-frequency speckle than the \dmd target (\cref{fig:supp-proj-endpoints}).

Our video experiments cover two backbones, primarily at 4 NFE, and do not establish preservation of diversity across prompts and seeds. On MiniMax-H3, the 50-step teacher remains preferred in the user study despite \methodname{}'s higher aggregate video score (\cref{tab:h3}), indicating a remaining perceptual-quality gap and a limitation of the automatic metrics.

\end{document}